\documentclass[11pt]{article}
\usepackage[T1]{fontenc}
\usepackage[utf8]{inputenc}
\usepackage{amsmath,amssymb,amsthm}
\usepackage[margin=1.15in]{geometry}
\usepackage{graphicx}
\usepackage{hyperref}
\hypersetup{colorlinks=true,linkcolor=black,citecolor=blue,urlcolor=blue,
  pdftitle={What Can Language Models Inherit from Solomonoff Induction? Limits, Attainable Properties, and What Served Conditionals Reveal},
  pdfauthor={L. Ozelim, A. Uthamacumaran, H. Zenil}}

\newtheorem{theorem}{Theorem}[section]
\newtheorem{lemma}[theorem]{Lemma}
\newtheorem{proposition}[theorem]{Proposition}
\newtheorem{corollary}[theorem]{Corollary}
\newtheorem{fact}[theorem]{Fact}
\theoremstyle{definition}
\newtheorem{definition}[theorem]{Definition}
\newtheorem{assumption}[theorem]{Assumption}

\theoremstyle{remark}
\newtheorem{remark}[theorem]{Remark}
\theoremstyle{plain}

\newcommand{\Ac}{\mathcal{A}}
\newcommand{\Cc}{\mathcal{C}}
\newcommand{\Mc}{\mathfrak{M}}
\newcommand{\E}{\mathbb{E}}
\newcommand{\Nb}{\mathbb{N}}

\newcommand{\eps}{\varepsilon}
\newcommand{\dl}[2]{D\!\left(#1 \,\|\, #2\right)}
\newcommand{\TV}{d_{\mathrm{TV}}}
\providecommand{\argmin}{\operatorname*{arg\,min}}
\title{What Can Language Models Inherit from Solomonoff Induction?\\
Limits, Attainable Properties, and What Served Conditionals Reveal}
\author{Luan Ozelim$^{1}$, Abicumaran Uthamacumaran$^{1}$ and Hector Zenil$^{1,2}$\\[2mm]
\small $^{1}$ Algocyte | Oxford Immune Algorithmics, Oxford University Innovation \\ \& 
\small London Institute for Healthcare Engineering, UK\\
\small $^{2}$ Algorithmic Dynamics Lab, Department of Biomedical Computing,\\
\small School of Biomedical Engineering and Imaging Sciences, King's Institute for AI,\\
\small King's College London, UK\\
}
\date{}
\providecommand{\CVST}{\mathcal{C}_V}
\begin{document}
\maketitle

\begin{abstract}
A familiar line of reasoning runs as follows: language models are flexible
approximators trained under a strictly proper loss, so scale brings them
closer to Solomonoff induction, and their predictive success is evidence of
that approach. We apply known results from algorithmic information theory,
universal prediction and Bayesian consistency. In its uniform form the
conclusion holds for no computable model: each has cheap environments it never
learns, and every uniformly computable family of total predictors shares one on
which each member fails beyond a horizon growing with its index, so ``closer''
can hold only pointwise, one environment at a time. Nor could the offered routes
certify even that: finite
evaluation cannot certify agreement on rarely reached contexts, perfect
corpus fit selects the corpus law as prior, and asymptotic criteria cannot
see the constants that carry Solomonoff's guarantee. Our lens is the Kraft
inequality read for predictors. Every model has at most the same budget of
predictive success, so Solomonoff's predictor, which leaves part of it
unspent, differs from a language model not in having more of it but in where
it spends it; on mutually singular families
this is exact, and only full-mass Bayes mixtures reach the frontier. We
then discuss what is inheritable, including in-context learning as
amortised Bayesian inference; what survives within a context window; which
Solomonoff-like properties a deployed system should have or avoid; and what
served conditionals determine. A closing table states what common practices
certify and what to report. Tests are developed in a companion paper.
\end{abstract}

\section{Introduction}
\label{sec:intro}

Language models are trained to minimise the log-loss, a strictly proper
scoring rule~\cite{gneiting2007}, and transformers are universal
approximators of continuous sequence-to-sequence functions of fixed length
on compact domains~\cite{yun2020}. A familiar line of reasoning joins these
facts: with enough data and compute, such a model approaches the
predictor that learns every computable source, Solomonoff's mixture over
all computable hypotheses under a simplicity prior~\cite{solomonoff1964,solomonoff1978},
and its measured predictive success
is evidence of that approach. We call this claim~(C).

Claim~(C) is not a position anyone defends at full strength, and we do
not attribute it to anyone. It is the common strengthening of a family
of positions that \emph{are} defended, stated at full strength so that
the obstacles can be located: each obstacle below falls at a
different point, and a claim held at half strength would hide which. The
positions themselves are the target, and Figure~\ref{fig:positions}
records for each one what survives and which result applies. The closest
published statement of (C) itself, the hypothesis that transformers
approximate Solomonoff induction better than any other extant sequence
predictor~\cite{youngwitbrock2024transformers}, is comparative and
carries no guarantee. Deployed models, whose compute-efficient training
stops well short of convergence~\cite{kaplan2020scaling,hoffmann2022chinchilla},
give (C) little direct
purchase. The weaker forms identify language modelling with
compression~\cite{deletang2024compression}, whose efficiency is reported to track
benchmark performance almost linearly~\cite{huang2024compression}; credit
networks with a simplicity bias matched to real
data~\cite{goldblum2024nfl,mingard2025occam}; report that neural
predictors meta-trained on samples from a universal prior show falling regret
with scale on such samples, with convergence to Solomonoff's predictor
established in the limit under realisability rather than
demonstrated~\cite{grau2024universal}; or describe language
models as computable approximations to Solomonoff
induction~\cite{wan2025solomonoff}. Each is a statement about \emph{where} a model predicts well, and each
survives this paper in part. What does not survive is the step they can be
read as licensing: that the part which survives \emph{certifies} the rest. None of
the results below says that such success is no evidence at all about allocation
elsewhere; they say that no finite amount of it can guarantee that allocation. Compression efficiency tracking benchmark performance is a
statement about a corpus and bears on no other allocation
(\S\ref{sec:closure}); a simplicity bias matched to real data is a
statement about one prior and not about a universal one
(\S\ref{sec:attainable}); falling regret on samples from a universal
prior is convergence \emph{under realisability}, which is the hypothesis
at issue (\S\ref{sec:scale}); and a computable approximation to
Solomonoff induction is what Section~\ref{sec:scale} shows
cannot be uniform. Figure~\ref{fig:positions} states this position by
position.

\paragraph{The four positions, and what survives each.}
\emph{Compression.} Language modelling is compression, and compression
efficiency is reported to track capability~\cite{deletang2024compression,huang2024compression};
its most quoted informal statement in machine learning casts unsupervised
learning as compression~\cite{sutskever2023observation}. The position long predates that evidence. Compression on a fixed
corpus was proposed as an objective replacement for the Turing test
\cite{mahoney1999}, is the basis of a standing prize for compressing a
snapshot of Wikipedia~\cite{hutterprize}, and was
folded into formal intelligence measures
by~\cite{hernandezorallo2000beyond,hernandezorallo2010}; the older form is the more specific
one, because it names the corpus and so concedes what the recent form
leaves open.
A rate measures allocation \emph{on the data compressed}, where
universality costs at most $K(q)+O(1)$ bits. What it does not do is
certify a constant on a declared class: a rate below PNG or FLAC says
nothing about one, and whether these models beat an adaptive low-order
context model on the same chunks is not
reported~\cite{deletang2024compression,heurteldepeiges2024compression}.
Compression by the shortest generated program, the reading closest to
$K$, remains far from strong baselines on real
data~\cite{kolmogorovtest2025}. Non-vacuous generalisation bounds do
hold, on the training law and with prediction
smoothing~\cite{lotfi2024nonvacuous,lotfi2024tokenlevel}. Survives as a
statement about one corpus (\S\ref{sec:closure},
Theorem~\ref{thm:si-underdet}).

\emph{Simplicity bias.} Networks prefer low-complexity functions and
sequences, as real data do~\cite{goldblum2024nfl,mingard2025occam}. This
is a direction of allocation, attainable and consistent with the
no-free-lunch argument, and it is measurable as truth-loss levels on a
declared class. What it is not is a universal prior, and asymptotic
signatures cannot identify it (\S\ref{sec:icl}, \S\ref{sec:si-tail},
Proposition~\ref{prop:slope}).

\emph{Meta-learned universal predictors.} Neural predictors meta-trained
on samples from a universal prior show falling regret with
scale~\cite{grau2024universal}, with convergence established in the limit
under realisability rather than demonstrated. This is amortised Bayesian
inference on the sampled class, certified in context by prior mass and
exponentially weak in the task's complexity. Outside the sampler's bounds
on program length and steps the prior is uncertified, and the network has
its own diagonal environment (Proposition~\ref{prop:pretrain-certify},
Theorems~\ref{thm:si-underdet} and~\ref{thm:si-diag}).

\emph{Computable approximation.} Training optimises program length
through loss minimisation, and next-token prediction approximates
Solomonoff induction better than other extant
predictors~\cite{wan2025solomonoff,youngwitbrock2024transformers}. The
first reading is the $K(q)$ premium of \S\ref{sec:closure}: bounded and
pathwise, but informative only where $q$ itself predicts well. The second survives as $\eps$-optimality on a
context distribution (Assumption~\ref{ass:si-eps}), so that ranking
predictors by closeness to Solomonoff's means comparing constants on a
declared class; convergence along a scaling family is pointwise,
not uniform (Remark~\ref{rem:price-of-universality}, \S\ref{sec:scale}
with \S\ref{sec:scale-exchange}, \S\ref{sec:attainable}).

\begin{figure}[!tp]
\centering
\includegraphics[width=\textwidth]{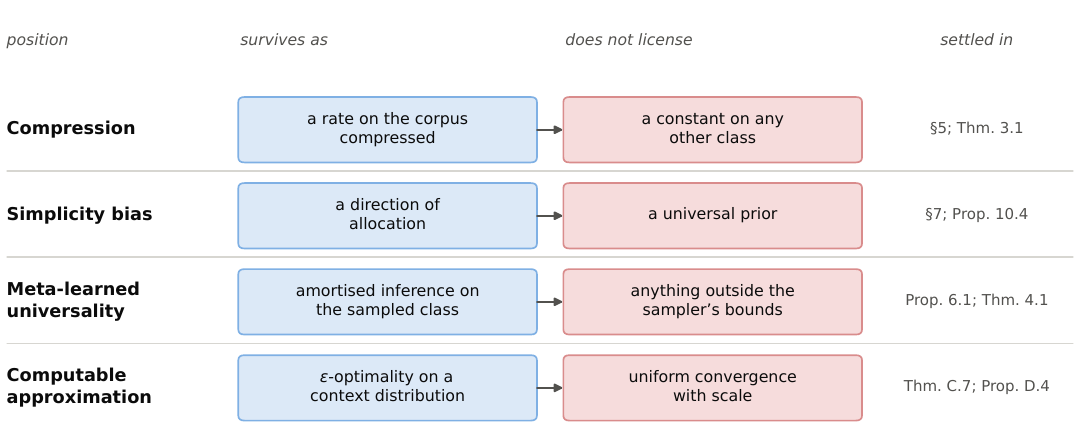}
\caption{The four positions (C) strengthens, and the line each one stops at. The arrow is the boundary this paper draws: to its left, what the position establishes and this paper grants; to its right, the step it is commonly read as licensing and does not. The adjudication is in the text above, and the right-hand column names where each line is drawn.}
\label{fig:positions}
\end{figure}

\paragraph{The Kraft lens.}
For any predictor, computable or not, at most $2^{b}$ strings of a given
length receive cumulative log-loss below $b$ bits. This is the Kraft
inequality~\cite{kraft1949,mcmillan1956} read for predictors, a lens
rather than a result: competence is a fixed budget. Solomonoff's predictor
has no more of it than a bigram model; it spends it elsewhere: within a factor
$2^{K(q)+O(1)}$ of what any computable $q$ spends on every sequence, which is what
universality means. Where a predictor spends the budget is what its prior is. Each obstacle to~(C)
below is a statement about allocation; where the lens is exact is stated
in Section~\ref{sec:frontier} (Figure~\ref{fig:frontier}).

\paragraph{Contributions.}
This is a discussion paper: most individual results are known, the dozen or so
new pieces are mostly elementary, and Appendix~\ref{sec:provenance} gives the role of
each. The terrain it crosses has a standard survey
in~\cite{rathmanner2011}, which treats universal induction as a theory of
inference and asks whether it is the right formalisation of Occam's
razor. The scope here is narrower in two ways. It is about
\emph{prediction} and not about agency: the universal-agent programme
that extends Solomonoff induction with a utility and an action channel
is out of scope, and where a result from it is used (Pareto
optimality among predictors in \S\ref{sec:frontier}), it is used for
its prediction half only. And the question is empirical rather than
foundational: not whether Solomonoff induction is the right ideal, but
which of its properties a trained and served language model can have,
which it cannot, and which of the routes commonly offered for the claim
that it approaches one
could establish it. Item~(ii) shows that the conclusion of~(C) in its uniform form
holds for no computable model, and that approach can hold only pointwise, and
items (i), (iii) and (v) that no route offered for it could certify it.
\begin{enumerate}
\item[(i)] \emph{Evaluation} (Section~\ref{sec:evaluation}). Near-optimal
expected log-loss bounds divergence on a set of contexts only in inverse
proportion to its probability, so $N$ queries, adaptive or not, cannot
separate models that differ only where the queries rarely go. Benchmark
and red-team evidence can falsify, but never certify, agreement with a
reference, and scaling curves track the $\eps$ of the test distribution
only.
\item[(ii)] \emph{Scale} (Section~\ref{sec:scale}). Putnam's
diagonal~\cite{putnam1963}, in Legg's form~\cite{legg2006}, gives every computable
model, closed pipelines with chain of thought included, an environment no
more complex than itself on which it loses $\log V$ bits per step
($\log(2V/3)$ for real-valued conditionals) while
Solomonoff's predictor pays a constant (for tools and retrieval see
Section~\ref{sec:si-prelim}). A uniformly computable family of total predictors
shares one such
environment; this forces nothing on a given model at an observable
horizon, but shows that apparent convergence along a scaling family rests
on replacing ``one model that is good at every horizon'' by ``at every
horizon, some model that is good''.
\item[(iii)] \emph{Data} (Section~\ref{sec:closure}). A predictor with zero
regret on its corpus coincides with the corpus law, so has finite regret
only on sources absolutely continuous with respect to it, and at most
$\log c$ on those it dominates by $c$: fitting a corpus chooses a prior.
Solomonoff's predictor pays at most $K(q)+O(1)$ bits more than it on the
corpus, below one percent of the corpus code length at about $2$ bits per
token, $16$-bit weights and at least $10^{3}$ tokens per parameter. Data
curation and mixture weights are prior design, and post-training
reallocates under the same budget.
\item[(iv)] \emph{Inheritance} (Section~\ref{sec:attainable}).
Solomonoff-form guarantees are attainable on declared classes by Bayes
mixtures, the formal content of in-context learning as amortised Bayesian
inference. We state what a pretraining loss certifies about the in-context
allocation, and why demonstrations can steer it.
\item[(v)] \emph{Asymptotics} (Section~\ref{sec:si-tail}). Consistency,
merging and vanishing per-step excess loss survive finite modification and
reweighting of the
prior~\cite{blackwelldubins1962,wood2013nonequivalence,leike2015bad}, while
Solomonoff's guarantee lives in the constants these operations change.
Such criteria are blind to the prior.
\item[(vi)] \emph{Deployment} (Section~\ref{sec:deployed}). Within a
context window, conservation, underdetermination, the single diagonal and
prior reading keep their force. On stationary sources a bounded-window
model repays the cost of identifying its environment at most about once
per window, and exactly that on a copy source; only temperature one keeps
Solomonoff's bound; truncation is dogmatism by design; and the
character-level law is a marginal over tokenisations.
\item[(vii)] \emph{Properties} (Section~\ref{sec:properties}). We record
which ``Solomonoff-like'' properties are attainable and, as a judgement,
rank them; several attainable ones are attack surfaces, and we argue,
without proof, that a trustworthy system should have Solomonoff-like
beliefs and non-Solomonoff-like constraints.
\item[(viii)] \emph{Measurement} (Section~\ref{sec:tests}). We define what
served conditionals determine without a prior: a coherence residual and
the arbitrage it prices, invariance nulls, and the implied prior over a
declared class with its identifiability. The companion paper (Paper~B)~\cite{paperB}
turns them into tests and applies them to served models.
\end{enumerate}
Section~\ref{sec:discussion} draws the consequences for research practice.

\paragraph{How to read this paper.}
Section~\ref{sec:background} fixes notation and states the lens. Proofs
are in the appendices, which open with a table of notation. Classical
results are Facts, adapted or sharpened ones carry an ``after'' citation,
and Appendix~\ref{sec:provenance} records the status of each.
Paragraphs headed ``In words'' restate key results. A shorter route for
machine learning readers: Figure~\ref{fig:positions}, the Kraft
lens (Section~\ref{sec:conservation} and Figure~\ref{fig:frontier}),
Sections~\ref{sec:scale}, \ref{sec:closure} and~\ref{sec:deployed}, and
Table~\ref{tab:discussion-practice}. Resource-bounded forms of the limits are in Appendix~\ref{sec:rb}.

\paragraph{What is claimed and what is not.}
Nothing below shows that language models generalise poorly, that scale is
unproductive, or that any particular system is deficient: no computable
system has the universal property, so impossibility is not deficiency.
Nor is attainability universality (a system matching a declared posterior
exactly has its own diagonal environment), or uncertifiability
unmeasurability (coherence, invariance and a prior over a declared class
remain measurable). The claims concern what a class of
arguments can establish and a class of measurements can certify, for
idealised predictors at stated horizons, and Section~\ref{sec:deployed}
says which survive deployment.
The paper measures no language model of its own; the served-model readings it
relays in Section~\ref{sec:discussion} are Paper~B's and are labelled as such.
Its own computations are the worked numerical illustrations and the frontier
figure, which are produced by the released scripts (\texttt{demo\_paperA\_numbers.py}
and \texttt{fig\_frontier.py}).

\providecommand{\qbar}{\bar q}
\providecommand{\Kraft}{\mathsf{Kr}}
\providecommand{\Fam}{\mathcal{F}}

\section{Background and the Kraft lens}
\label{sec:background}

\subsection{Predictors, regret and Solomonoff's predictor}
\label{sec:si-prelim}

Fix a finite alphabet $\Ac$ with $|\Ac|=V\ge 2$. Finite strings are
$x\in\Ac^{*}$, with length $|x|$, empty string $\epsilon$ and prefix order
$\sqsubseteq$; $x_{1:n}$ is a length-$n$ prefix and $x_{<t}:=x_{1:t-1}$.
Logarithms are base~$2$ and information is in bits.

\begin{definition}[Predictors, log-loss, regret and excess loss]
\label{def:si-predictor}
A \emph{predictor} $q$ assigns to every context $x\in\Ac^{*}$ a vector
$q(\cdot\mid x)\ge 0$ with $\sum_a q(a\mid x)\le 1$, extended to strings by
$q(x_{1:n}):=\prod_{t\le n}q(x_t\mid x_{<t})$. It is a \emph{semimeasure},
and a probability measure on $\Ac^{\infty}$ if every conditional sums to
one~\cite{zvonkinlevin1970,li2019kolmogorov}. The \emph{log-loss} is
$L_n(q;x_{1:n}):=-\log q(x_{1:n})$, and $L_{|x|}(q;x):=-\log q(x)$ for a
string of any length. For a probability measure $\mu$ (a \emph{source})
the \emph{regret} of $q$ at horizon $n$ is
\begin{equation}
\label{eq:si-regret}
R_n(q;\mu):=\E_{\mu}\Bigl[\log\frac{\mu(X_{1:n})}{q(X_{1:n})}\Bigr]
=\sum_{t=1}^{n}\E_{\mu}\,\dl{\mu(\cdot\mid X_{<t})}{q(\cdot\mid X_{<t})},
\end{equation}
with $\dl{p}{s}:=\sum_ap(a)\log\frac{p(a)}{s(a)}\ge 0$ for subprobability
$s$; it is nondecreasing in $n$, and $R_\infty(q;\mu):=\lim_nR_n(q;\mu)$.
For a deterministic environment $z\in\Ac^{\infty}$,
$R_n(q;z):=L_n(q;z_{1:n})$. For a \emph{reference} predictor $p$ (a source
$\mu$, Solomonoff's $M$, or another model), $L_n(q;x)-L_n(p;x)$ is the
\emph{excess loss of $q$ over $p$}, and
$-\log q(x_t\mid x_{<t})+\log p(x_t\mid x_{<t})$ its per-step excess loss;
over the source, the expected excess loss is the regret.
\end{definition}

A softmax language model whose conditionals sum to one defines a measure on
infinite \emph{token} sequences absent truncation; its law on finite,
end-of-sequence-terminated strings is a probability distribution only when the
end-of-sequence mass is tight~\cite{du2023tight}; on
characters its law is a marginal over tokenisations
(Section~\ref{sec:deployed}).

\begin{definition}[Computable predictors]
\label{def:si-computable}
A predictor is \emph{rational-computable} if a Turing machine outputs the
exact rational $q(a\mid x)$ on input $(x,a)$, and \emph{computable} if
$q(a\mid x)$ is a computable real uniformly in $(x,a)$. $K$ denotes prefix
Kolmogorov complexity~\cite{li2019kolmogorov}, and $K(q)$ that of a
canonical index of such a machine.
\end{definition}

A deployed model taken as a \emph{closed pipeline} (finitely many
finite-precision parameters, a bounded window, and decoding under an
enforced computation clock) is rational-computable (rounded outputs need not
sum to one, which the diagonal tolerates; Remark~\ref{rem:rb-real}), as is any closed
pipeline around it, such as clocked chain of thought. This is not in
tension with the Turing-completeness of
attention~\cite{perez2021turing}: that construction needs unbounded
precision in the residual stream and an unbounded number of decoding
steps, and these are the two resources that a deployment does not
have. Remove either (fix the precision, with the bounded window, or enforce the clock) and
the pipeline is a finite object computing rationals, which is all the
results below require. A pipeline that queries the world (retrieval,
tool loops, retraining on outside data) is not such a closed object: it is
computable only relative to what it queries, and the
diagonal constructions of Section~\ref{sec:scale} then yield environments
computable relative to it. A served model is moreover a
random measure (serving can be nondeterministic across
batches~\cite{he2025nondeterminism}; Paper~B); randomised predictors are treated in
Section~\ref{sec:scale-decoding}.

\begin{definition}[Solomonoff's predictor]
\label{def:si-M}
Fix a universal monotone machine $U$~\cite[Ch.~4]{li2019kolmogorov}
(unrelated to the reference machine of Appendix~\ref{sec:rb}). The
universal semimeasure is $M(x):=\sum_{p:\,x\sqsubseteq U(p)}2^{-|p|}$, over
minimal programs whose output extends $x$; it is lower semicomputable and
positive on every string. $K(\mu)$ is the complexity of a program computing
a computable measure $\mu$ to any requested precision, and $K(\nu)$ that of
a program enumerating a lower-semicomputable semimeasure $\nu$ from below.
\end{definition}

\begin{fact}[Dominance and Solomonoff's bound; \cite{zvonkinlevin1970,solomonoff1978,li2019kolmogorov,hutter2005uai}]
\label{fact:si-dominance}\label{fact:si-solomonoff}\label{fact:si-program}
For every lower-semicomputable semimeasure $\nu$,
$M(x)\ge 2^{-K(\nu)-O(1)}\,\nu(x)$ for all $x\in\Ac^{*}$. Hence
$R_\infty(M;\mu)\le K(\mu)+O(1)$ bits for every computable measure $\mu$,
so the per-step expected excess loss of $M$ over $\mu$ vanishes; and if
$U$ on program $p$ emits $z\in\Ac^{\infty}$, then $R_n(M;z)\le|p|$ for all
$n$, since $p$ contributes $2^{-|p|}$ to every $M(z_{1:n})$.
\end{fact}

\noindent This is what it would mean to ``achieve Solomonoff induction'':
one predictor whose total regret on \emph{every} computable environment is
bounded by that environment's description length, uniformly in the
horizon.

\paragraph{Bayes mixtures.}
For a countable family $\Fam=(\mu_k)_{k\in J}$ of probability measures and
weights $w_k\ge 0$ with $\sum_kw_k\le 1$, the \emph{Bayes mixture}
$\xi_w$ is the predictor with
\begin{equation}
\label{eq:bg-mixture}
\xi_w(x):=\sum_{k\in J}w_k\,\mu_k(x)\qquad(x\ne\epsilon).
\end{equation}
Its conditionals weight the components by their posterior, the mass
$1-\sum_kw_k$ is lost at the first step, and $\xi_w\ge w_k\mu_k$ gives
$R_\infty(\xi_w;\mu_k)\le\log(1/w_k)$. We call $\xi_w$ \emph{full-mass} if
$\sum_kw_k=1$.

\paragraph{The premise granted to (C)\@.}
The approximation argument delivers at most the following.

\begin{assumption}[$\eps$-optimality]
\label{ass:si-eps}
There are a \emph{context distribution} $\rho$ on $\Ac^{*}$ (the prefix
law of the training or test source, or a benchmark's prompt distribution),
target conditionals $\mu(\cdot\mid c)$, and $\eps\ge 0$ with
\begin{equation}
\label{eq:si-eps}
\E_{c\sim\rho}\,\dl{\mu(\cdot\mid c)}{q(\cdot\mid c)}\le\eps .
\end{equation}
The asymptotic form is a family $(q_i)_{i\ge 1}$ in which, for every
$\eps>0$, some member satisfies~\eqref{eq:si-eps}.
\end{assumption}

Strict propriety and universal approximation (Section~\ref{sec:intro})
are the usual reasons for granting a network within $\eps$ of the true
conditionals. We grant this without the latter's restrictions (fixed
length, compact domain, continuous target) and concede estimation and
optimisation. Two features of~\eqref{eq:si-eps} matter below: it concerns
only $(\mu,\rho)$, and when $\rho$ is the
\emph{position-averaged prefix law} at horizon $n$ (draw $t$ uniformly from
$\{1,\dots,n\}$, then $c=X_{<t}$ with $X\sim\mu$), it yields the linear
envelope $R_n(q;\mu)\le\eps n$, whereas Fact~\ref{fact:si-solomonoff} gives
a constant envelope; for fixed $\eps>0$ the first never implies the second.

\subsection{Conservation as a lens}
\label{sec:conservation}

\begin{fact}[Kraft inequality and conservation of competence; \cite{kraft1949,mcmillan1956}, {\cite[Ch.~2]{li2019kolmogorov}}]
\label{thm:conservation}\label{lem:antichain}
Let $q$ be a predictor and $S\subseteq\Ac^{*}$ an antichain (no element is
a prefix of another). Then $\sum_{x\in S}q(x)\le 1$. Hence, for every
$b\ge 0$, $\bigl|\{x\in S:L_{|x|}(q;x)\le b\}\bigr|\le 2^{b}$; in
particular at most $2^{b}$ strings of each length $n$ have
$L_n(q;x)\le b$.
\end{fact}

\paragraph{In words.} Every predictor has the same fixed amount of
predictive success to distribute over the strings of each length;
predictors differ in where they put it and how much they leave unused,
never in how much they have. On language models, accuracy on
deterministic tasks tracks the probability of the output and of the task
(rot-13 against rot-2), which is where the budget is
spent~\cite{mccoy2023embers}.

The counting is that behind incompressibility; the correspondence between
codes and priors it expresses underlies minimum description length and
universal prediction~\cite{barron1998mdl,grunwald2007mdl,cesabianchi2006},
and no-free-lunch results state the same trade for problems weighted
uniformly~\cite{schaffer1994,wolpert1996}, which a simplicity prior
escapes~\cite{lattimore2013nfl}. Whether that escape leaves the
no-free-lunch theorems with any force against a learner that is doing
well in practice is contested,
and~\cite{sterkenburggrunwald2021} argue the negative: the theorems
presuppose learning algorithms that are purely data-driven, whereas standard
algorithms take a model as input and can be given a model-relative
justification, so the theorems constrain less than they are usually taken
to. The counting used here does not depend on that
dispute, and the difference is worth stating because the lens's force
rests on it. A no-free-lunch theorem averages performance over
a weighting of problems and concludes that the average is fixed; change
the weighting and the conclusion moves. The Kraft bound assumes no
weighting at all. It says that on any \emph{antichain} of strings (of
one length, or any prefix-free set) at most $2^{b}$ members can be
given loss below $b$, because their probabilities must sum to at most
one. That is a counting fact about the code, true of every predictor
including $M$, and it survives any reweighting of which problems
matter. The simplicity prior escapes the \emph{average} but not the
ledger. Appendix~\ref{app:bg-conservation} derives the
exchange rate at which post-training and data mixtures reallocate: when $q$ is a measure, gaining $b$ bits on a set $A$ of
strings of length $n$ removes at least $(2^{b}-1)\,q(A)$ of mass from the
rest (Corollary~\ref{cor:exchange}).

$M$, a semimeasure that leaks (its measure part, Lemma~\ref{lem:measurepart},
has mass below one, for otherwise $M$ would be a lower-semicomputable
measure, hence computable, and could not dominate the point mass on its own
diagonal sequence), spends less of the budget than a softmax
language model, which spends all of it, and lies off the frontier of every
pairwise singular family (Section~\ref{sec:frontier}); what
separates them is where each spends (the last clause of
Fact~\ref{fact:si-dominance}) and the diagonal constructions of
Section~\ref{sec:scale}. At a finite horizon, and hence
within any context window, the budget binds exactly on environments: if
$z^{(1)},\dots,z^{(m)}$ have pairwise distinct length-$n$ prefixes, then
$\sum_k2^{-R_n(q;z^{(k)})}\le 1$, so at least half of them cost $q$ more
than $\log(m/2)$ bits (the finite-horizon Kraft law,
Corollary~\ref{thm:profile}). The same law caps the loss ceilings
$\log(1/w_h)$ that a Bayes mixture over deterministic hypotheses can promise
on many tasks at once (the Kraft floor, Corollary~\ref{cor:kraftfloor}).

\subsection{The frontier of regret profiles}
\label{sec:frontier}

Which regret profiles can one predictor achieve on a family, and which
are best? For mutually singular environments, Kraft's inequality answers.
The theorem combines the Lebesgue decomposition and data processing with
the distinguishability argument of~\cite{merhav1995strong}, the
measure part of a
semimeasure~\cite{levinvyugin1977,bienvenu2017semimeasures}, and Pareto optimality of
Bayes mixtures~\cite{hutter2003optimality}. Hutter proves that $\xi$ is
Pareto-optimal, remarks that ``for prediction only Bayes-mixes are
Pareto-optimal'', and asks for the precise statement: ``it would be
interesting to show under which conditions, the class of all Bayes-mixtures
(i.e.\ with all possible values for the weights) is complete in the sense that
every Pareto-optimal strategy can be based on a
Bayes-mixture''~\cite[\S5.2]{hutter2003optimality}. For finite families at a
finite horizon, Wald's complete-class theorem~\cite{wald1950} already gives it
(Proposition~\ref{prop:frontier-overlap}(i)). Clause~(iii) gives it at
$R_\infty$ for countable pairwise mutually singular families, where the
full-mass Bayes mixtures are exactly the predictors whose Kraft sum is one;
once the carriers are disjoint that equality is close to immediate, and what is
added is the countable case, the ledger and its converse. The closest prior statement of
the equality case we have found is an aside of
Koolen's~\cite[\S4.1]{koolen2013pareto}, whose subject is absolute loss:
``For certain log loss games, each Pareto regret trade-off is witnessed
uniquely by the Bayesian mixture of expert predictions w.r.t.\ a certain
non-uniform prior and vice versa (not shown).'' That is an unproved remark
about expert games at finite horizon rather than about $R_\infty$ on mutually
singular measures, and we have not found the ledger or the converse of
Proposition~\ref{prop:frontier-ledger} stated anywhere.

\paragraph{Setting.}
The \emph{regret profile} of $q$ on $\Fam=(\mu_k)_{k\in J}$ is
$(R_\infty(q;\mu_k))_{k\in J}$, and the \emph{Kraft set} is
$\Kraft_J:=\{r\in[0,\infty]^J:\sum_k2^{-r_k}\le 1\}$ with
$2^{-\infty}:=0$. Measures are \emph{mutually singular}, $\mu\perp\nu$, if
a Borel set carries $\mu$ and is $\nu$-null. Distinct deterministic
environments, and distinct stationary ergodic measures, are mutually
singular; two laws whose first-symbol marginals have full support and
which agree thereafter are not. Every
predictor $q$ has a \emph{measure part} $\qbar$, the largest measure
dominated by $q$ on cylinders, and $R_\infty(q;\mu)=\dl{\mu}{\qbar}$ for
every source $\mu$ (Lemma~\ref{lem:measurepart}).

\begin{figure}[t]
\centering
\includegraphics[width=0.9\textwidth]{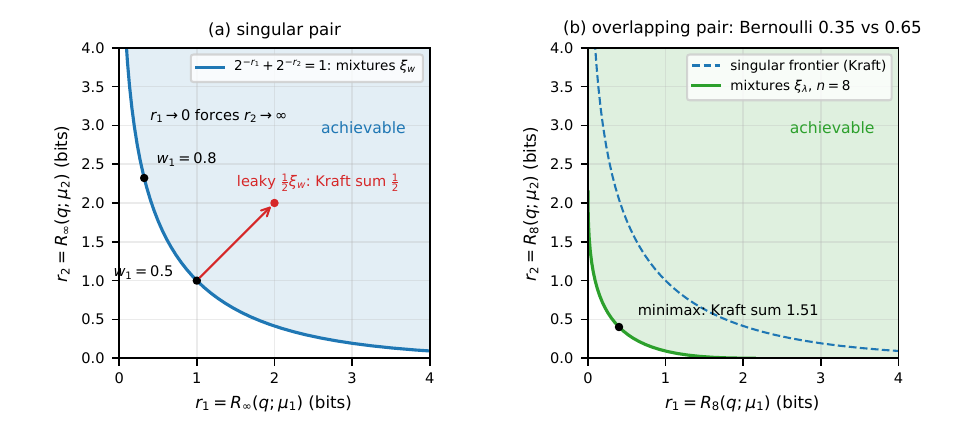}
\caption{Regret profiles on two sources. (a) $R_\infty$ on a singular
pair: the achievable region (shaded) is the Kraft set, bounded by
$2^{-r_1}+2^{-r_2}=1$ and traced by the mixtures $\xi_w$, so a prior is a
point on this curve; the leaky $\frac12\xi_w$ lies one bit inside
(Proposition~\ref{prop:frontier-ledger}). (b) $R_8$ on i.i.d.\
Bernoulli$(0.35)$ and $(0.65)$ sources, overlapping at this horizon: the boundary traced by the
mixtures $\xi_\lambda$ (Proposition~\ref{prop:frontier-overlap}) lies
below the Kraft curve, with Kraft sum about $1.5$ at the minimax point.}
\label{fig:frontier}
\end{figure}

\begin{theorem}[Frontier theorem; after \cite{merhav1995strong,hutter2003optimality}]
\label{thm:frontier}
Let $q$ be any predictor. In \textup{(i)--(iii)} the family $\Fam$ is
pairwise mutually singular.
\begin{enumerate}
\item[\textup{(i)}] \textup{(Kraft law.)}
$\sum_k2^{-R_\infty(q;\mu_k)}\le1$, and the deficit splits exactly into
leakage, mass off the family and misfit within it
(Proposition~\ref{prop:frontier-ledger}).
\item[\textup{(ii)}] \textup{(Exact attainment.)} For every
$r\in\Kraft_J$, the mixture $\xi_w$ with $w_k=2^{-r_k}$ has
$R_\infty(\xi_w;\mu_k)=r_k$ for every $k$.
\item[\textup{(iii)}] \textup{(The frontier is the full-mass Bayes
mixtures.)} $\sum_k2^{-R_\infty(q;\mu_k)}=1$ iff $q=\sum_kw_k\mu_k$ on every
cylinder for weights with $\sum_kw_k=1$, and then
$w_k=2^{-R_\infty(q;\mu_k)}$. A profile in $\Kraft_J$ is Pareto-optimal iff
its Kraft sum is~$1$.
\end{enumerate}
Hence, for a nonempty family, the best Kraft constant
$C_\infty(\Fam):=\sup_q\sum_k2^{-R_\infty(q;\mu_k)}$ equals $1$ iff $\Fam$
is pairwise mutually singular, the converse being
Proposition~\ref{prop:frontier-ledger}(ii). (Proof:
Appendix~\ref{app:proof-thm:frontier}.)
\end{theorem}

\paragraph{In words.} Singular members have disjoint carriers $\Omega_k$,
and a predictor pays $\dl{\mu_k}{\qbar}\ge\log(1/\qbar(\Omega_k))$ on
$\mu_k$, where the carrier masses $\qbar(\Omega_k)$ sum to at most one
(Figure~\ref{fig:frontier}).

Part (ii) turns the mixture bound after~\eqref{eq:bg-mixture} into an
equality, and asymptotic criteria are constant along this frontier
(Corollary~\ref{cor:blind}). A predictor within $r$ bits of a source $\mu$
pays at least $\log\frac{1}{1-2^{-r}}\approx\log\frac{1}{r\ln 2}$ bits (for small $r$) on every
environment singular to $\mu$ (the price of specialisation,
Corollary~\ref{cor:specialise}). On finite overlapping families the
region is still exact, but the Kraft sum is maximal only at mixtures over
maximising priors (Proposition~\ref{prop:frontier-overlap}), so ``a prior
is an allocation'' is exact only on singular families; for countable
non-singular ones only $C_\infty>1$ is proved.

\section{What evaluation can certify}
\label{sec:evaluation}

Language models are evaluated on finite evidence. Benchmarks draw prompts
from a fixed distribution, red teams choose prompts adaptively, and
scaling studies fit held-out loss curves. Two elementary tools, Markov's
inequality and a coupling of evaluation runs, show what this evidence
can establish. It constrains a model's conditionals only on contexts
that the evaluation reaches with probability comparable to the inverse
of its number of queries. It can therefore falsify agreement with a
reference such as Solomonoff's predictor, but it cannot certify that
agreement where agreement would matter. The intuition is old, from
identification in the limit~\cite{gold1967} and the no-free-lunch
theorems~\cite{wolpert1996}; what follows makes it quantitative.
Throughout, $q$
satisfies Assumption~\ref{ass:si-eps} with respect to $(\mu,\rho)$, and
$D(c) := \dl{\mu(\cdot\mid c)}{q(\cdot\mid c)}$.

\subsection{Underdetermination}
\label{sec:si-underdet}

\begin{theorem}[Underdetermination by evaluation]
\label{thm:si-underdet}
\begin{enumerate}
\item[\textup{(i)}] \textup{(Mass--divergence trade-off.)} For every
$B \subseteq \Ac^{*}$ with $\rho(B) = \delta > 0$,
$\E_\rho[D(c) \mid c \in B] \le \eps/\delta$. For every $b > 0$,
$\rho\{c : D(c) \ge b\} \le \eps/b$.
\item[\textup{(ii)}] \textup{(Modification.)} Let $\rho(B) = \delta$ and
$\gamma \in (0,1/V]$. Let $r(\cdot\mid c)$, $c \in B$, be any
conditionals with $\min_a r(a\mid c) \ge \gamma$. The predictor $q'$
that equals $r$ on $B$ and $q$ elsewhere satisfies
Assumption~\ref{ass:si-eps} with $\eps + \delta\log(1/\gamma)$ in place
of $\eps$.
\item[\textup{(iii)}] \textup{(Sampled evaluations cannot see the
difference.)} An \emph{$N$-query $\rho$-evaluation} draws contexts
$c_1,\dots,c_N$ i.i.d.\ from $\rho$ and continuations
$y_k \sim \mu(\cdot\mid c_k)$. It observes $(c_k, y_k, q(\cdot\mid c_k))_{k\le N}$
together with independent internal randomness and outputs a verdict.
If $q$ and $q'$ agree outside a set $B$ with $\rho(B) = \delta$, the
verdict laws under $q$ and $q'$ differ in total variation by at most
$1-(1-\delta)^{N} \le N\delta$.
\end{enumerate}
\end{theorem}

\noindent(Proof: Appendix~\ref{app:proof-thm:si-underdet}.)

\paragraph{In words.} A good average score on a test distribution says
almost nothing about the contexts that distribution rarely produces. On a
set of small probability, a model's predictions can be replaced at will
at negligible cost to its average score. No test that samples contexts
from the same distribution can tell the replaced model from the original.

Clause (iii) covers any verdict computed from the observations at the queried
contexts, including accuracy on single-symbol answers sampled from $q$; an answer of
several symbols reads $q$ at contexts beyond the sampled one, which the clause covers
only when those contexts are counted among the $N$ queries and $B$ is closed under
the extensions read. Choosing
conditionals on $B$ that put half their mass on a least likely symbol of
a reference $P^{*}$ yields Corollary~\ref{cor:si-uncert}. Suppose $q$ satisfies
Assumption~\ref{ass:si-eps} and passes a given $N$-query evaluation with
probability at least $1-\alpha$. Then some predictor $q'$ satisfies the
assumption with $\eps + \delta\log(2V)$, passes with probability at
least $1-\alpha-N\delta$, and is at total-variation distance at least
$\tfrac12 - \tfrac1{2V}$ from $P^{*}$ at \emph{every} context of $B$.
Take $P^{*}$ to be the conditionals of $M$. Empirical success under
Assumption~\ref{ass:si-eps} then does not bound disagreement with
Solomonoff's predictor on any set of contexts whose $\rho$-probability is
small compared with $1/N$. The regime is not exotic: contexts of length
below $n$ number $(V^{n}-1)/(V-1)$, while a corpus or benchmark touches
a vanishing fraction of them.

\subsection{Adaptive evaluation}
\label{sec:eval-adaptive}

An evaluator who \emph{chooses} its $N$ contexts, adaptively and with
full sight of the served conditionals, is not covered by
Theorem~\ref{thm:si-underdet}(iii), and the transcript-coupling argument familiar
from query-complexity lower bounds says how much it can gain. Let
$\mathcal{R}_\delta := \{c : \rho(c) \le \delta\}$ be the individually
rare contexts and $\bar\rho_\delta := \rho(\mathcal{R}_\delta)$ their
total mass. Against any adaptive protocol, a predictor $q'$ that differs
from $q$ only on a set $B \subseteq \mathcal{R}_\delta$ of mass between
$\delta$ and $2\delta$, and is far from any reference $P^{*}$ on all of
$B$, changes the verdict law by at most the query mass the protocol places
on $B$. A protocol that separates $q$ from every such $q'$ with error at
most $\alpha$ on each side therefore needs
\[
  N \;\ge\; (1-2\alpha)\,\frac{\bar\rho_\delta}{2\delta}
\]
queries, and passive sampling already reaches the matching order: it
detects any fixed disagreement set of mass $\delta$ with
$N \ge \ln(1/\alpha)/\delta$ queries.

\paragraph{What adaptivity buys.}\label{rem:si-adaptive}
At a given resolution $\delta$, choosing where to look buys at most the factor
$2\ln(1/\alpha)/((1-2\alpha)\bar\rho_\delta)$ over passive sampling
(Theorem~\ref{thm:si-adaptive}). The comparison is at fixed $\delta$ and
presumes $\bar\rho_\delta \ge \delta$: the factor is a constant when the rare
mass is of order one at the resolution of interest, which holds for a corpus
prefix law, and is not when the rare mass is itself of order $\delta$, as it is
for a prompt law concentrated on a few items; since
$\bar\rho_\delta \downarrow 0$ as $\delta \downarrow 0$, no factor is constant
uniformly in $\delta$. The
comparison is not quite like for like, and the asymmetry favours the
passive side: the adaptive figure is a lower bound on separating $q$
from \emph{every} $q'$, while the passive one is an upper bound on
separating a \emph{known} $P^{*}$, which for $P^{*} = M$ is never in
hand. The factor is therefore a ceiling on what adaptivity can buy,
not a measured exchange rate. A
protocol published and fixed before the model is trained or tuned is
the extreme case: a disagreement placed among the rare contexts it does
not query is invisible to it. The restriction to individually rare
contexts cannot be dropped, since fewer than $1/\delta$ contexts have
probability above $\delta$ and a protocol can query all of them. Nor can
the bound be stated through $\rho$'s largest atom: if $\rho$ draws a
position $t \le n$ uniformly and returns the prefix $x_{<t}$, the empty
context alone has mass $1/n$. What adaptivity does buy is the ability to
\emph{find} failures rather than wait to sample them.
Section~\ref{sec:scale} shows that failures of a specific, cheaply
constructed kind always exist: an adaptive evaluator can falsify, but it
cannot certify.

\paragraph{Benchmarks, red teaming, scaling curves and contamination.} These are
instances of the same bound. Accuracy is a coarse function of the conditionals, so
passing an accuracy benchmark is compatible with near-maximal disagreement with a
reference on a set of small mass; red teaming falsifies but cannot certify;
scaling laws measure the excess loss on the evaluation law and say nothing off it;
and contamination changes which contexts the bound is about.
Appendix~\ref{sec:eval-practice} works each out.

\section{What scale cannot deliver}
\label{sec:scale}
\label{sec:impossibility}

Scaling suggests a retreat from (C): no fixed model is Solomonoff's
predictor, but ever larger models approach it. Putnam's classical
diagonal argument~\cite{putnam1963}, in the complexity-theoretic form
given by Legg~\cite{legg2006}, implies that no computable predictor has
Solomonoff's regret property. A dovetailed form of the same argument
applies to a whole uniformly computable family of total predictors at once. The family
result has no force on any particular model at any observable horizon; what it shows is that the
apparent convergence of scaling is a quantifier exchange.
Sterkenburg~\cite{sterkenburg2019,sterkenburg2018} analyses the scope of
the diagonal argument in detail.

\subsection{A single predictor}
\label{sec:si-diag}

\begin{theorem}[Diagonal environment; after Putnam and Legg~\cite{putnam1963,legg2006}]
\label{thm:si-diag}
Let $q$ be a total rational-computable predictor. Define
$x^{q} \in \Ac^{\infty}$ by
$x^{q}_{t} := \min\argmin_{a\in\Ac} q(a\mid x^{q}_{<t})$, where the
minimum is taken in a fixed order on $\Ac$, and let $\mu_q$ be the point
measure on $x^{q}$. Then:
\begin{enumerate}
\item[\textup{(i)}] $q(x^{q}_t\mid x^{q}_{<t}) \le 1/V$ for every $t$,
hence $R_n(q;\mu_q) = L_n(q;x^{q}_{1:n}) \ge n\log V$ for all $n$;
\item[\textup{(ii)}] $\mu_q$ is computable with
$K(\mu_q) \le K(q) + O(1)$, and $R_n(M;\mu_q) \le K(q) + O(1)$ for all
$n$.
\end{enumerate}
\end{theorem}

\noindent(Proof: Appendix~\ref{app:proof-thm:si-diag}.)

\paragraph{In words.} Every fixed computable predictor has a computable
sequence, no more complex than the predictor itself, on which it does no
better than uniform guessing at every step. On the same sequence,
Solomonoff's predictor pays a constant in total.

The construction extends to real-valued computable conditionals, with
per-step loss at least $\log(2V/3)$ in place of $\log V$. It therefore
yields Corollary~\ref{cor:si-nocomputable}: no computable predictor has finite
regret on every computable environment. This covers every language model
taken as a closed pipeline (Definition~\ref{def:si-computable}), so the
consequent of (C) holds for none, and empirical
evidence can bear only on graded, class-relative versions of it
(Section~\ref{sec:attainable}), where coverage itself has a price in
description length~\cite{legg2006} (Remark~\ref{rem:price}).

The diagonal environment is not one exotic sequence. Behind every prefix
$p$ there is such an environment, of complexity at most
$K(p)+K(q)+O(1)$, on which $q$ loses at least $\log V$ per step after
$p$. Behind the prefixes of length $\ell$ there are $V^{\ell}$ distinct
ones (Theorem~\ref{thm:breakers}). At the step where the environment
leaves $p$, the loss is at least $\log((V-1)/(1-q_{\max}))$, where
$q_{\max}$ is the model's largest conditional at $p$, so it grows with
sharpness. This is not a cost of scale: sharpness beyond calibration is
documented after post-training (Section~\ref{sec:closure}), while larger
pretrained models are better calibrated on multiple-choice
questions~\cite{kadavath2022know}. A user who can read the served
conditionals, and who always continues with a least likely symbol,
realises a finite-horizon instance of the construction
(Remark~\ref{rem:si-cheap}).

\paragraph{A model family.} A scaling family generated by one program from an index is
effectively enumerable, and a dovetailed form of the same construction gives a
single environment on which every member of a uniformly computable family fails,
member $i$ beyond a horizon growing with its index
(Theorem~\ref{thm:si-classdiag}, Appendix~\ref{sec:si-classdiag}). It forces
nothing on any particular model at an observable horizon; what it shows is that
the apparent convergence of scaling is a quantifier exchange. Computable families
converging \emph{pointwise} to $M$ exist; what fails is selection, since a rule
that selects, ensembles or switches between members by any computable function of
the data is one computable predictor with its own diagonal
(\S\ref{sec:scale-exchange}).

\subsection{Randomised decoding and compute budgets}
\label{sec:scale-decoding}

\paragraph{Randomised decoding.}\label{rem:randomised}
Served models sample, and serving can be nondeterministic across
batches~\cite{he2025nondeterminism}. Randomisation does not evade the
diagonal. If the scored object is the served conditional, as whenever
log-probabilities are read, temperature and truncation are deterministic
reshapings of it and Definition~\ref{def:si-computable} applies
unchanged. If instead the conditionals $q_\omega(\cdot\mid x)$ depend on
internal randomness $\omega$, apply the construction to the mean predictor
$q^{\mathrm{avg}}(\cdot\mid x) := \E_\omega q_\omega(\cdot\mid x)$ (not the measure
part $\qbar$ of Section~\ref{sec:frontier}), which is computable whenever the
seed law is. On the approximate diagonal of $q^{\mathrm{avg}}$, Jensen's
inequality at each step gives
$\E_\omega L_n(q_\omega;x^{q^{\mathrm{avg}}}_{1:n}) \ge n\log(2V/3)$, and that
diagonal has complexity $K(q^{\mathrm{avg}})+O(1)$. The budget of
Fact~\ref{thm:conservation} belongs to the averaged allocation.

\paragraph{The resource-bounded form.} Both diagonals have forms for
models of bounded description length and per-step compute, in which the
failing environment moves with the budget; Section~\ref{sec:deployed}
summarises them and Appendix~\ref{sec:rb} proves them.

\providecommand{\db}[2]{d\!\left(#1 \,\|\, #2\right)}

\section{Fitting a corpus chooses a prior}
\label{sec:closure}

Carried to its limit, Assumption~\ref{ass:si-eps} delivers a predictor
with zero excess loss over the corpus law. Minimising log-loss is
minimising code length~\cite[Ch.~5]{cover2006}, and strict propriety
makes the source itself the unique minimiser of the population
risk~\cite{gneiting2007}. Fitting a corpus selects the corpus law as the prior.
That prior has finite regret only on laws absolutely continuous with
respect to it, and at most $\log c$ on a law it dominates by $c$. The later
training stages (data mixing, post-training, refitting on generated
text) move the allocation around without enlarging the budget of
Fact~\ref{thm:conservation}.

\paragraph{Closure as domination.}
For probability measures on $\Ac^\infty$, we write
$\dl{\nu}{\mu}:=\lim_n\dl{\nu_n}{\mu_n}$, where a subscript $n$ denotes
the marginal on $\Ac^n$. The limit object of Assumption~\ref{ass:si-eps}
on $\mu$ is a predictor $q$ with $R_\infty(q;\mu)=0$.

\begin{theorem}[{Closure; (i)--(ii) cf.\ \cite[Ch.~7]{gray2011entropy}, (iii) via~\cite{tsybakov2009}}]
\label{thm:closure}
Let $\mu$ be a probability measure on $\Ac^\infty$ and $R_\infty(q;\mu)=0$.
\begin{enumerate}
\item[\textup{(i)}] $q=\mu$ on every cylinder, and
$R_\infty(q;\nu)=\dl{\nu}{\mu}$ for every probability measure $\nu$.
\item[\textup{(ii)}] $R_\infty(q;\nu)<\infty$ implies $\nu\ll\mu$. If
$\nu\le c\,\mu$ on cylinders, then $R_\infty(q;\nu)\le\log c$; in
particular $\mu\ge w_i\mu_i$ gives $R_\infty(q;\mu_i)\le\log(1/w_i)$.
\item[\textup{(iii)}] $R_n(q;\nu)\to\infty$ for every $\nu\perp\mu$, in
particular for every stationary ergodic $\nu\ne\mu$ when $\mu$ is
stationary ergodic. If $\mu$ is non-atomic,
$R_n(q;z)=-\log\mu(z_{1:n})\to\infty$ for every $z\in\Ac^\infty$.
\item[\textup{(iv)}] If $\mu$ is computable, some computable environment
gives $q$ infinite regret.
\end{enumerate}
\end{theorem}

\noindent(Proof: Appendix~\ref{app:proof-thm:closure}.)

In words: the limit object has zero regret on its own source and pays
$\dl{\nu}{\mu}$ on every other law $\nu$, infinite unless $\nu\ll\mu$.
Two readings of ``the corpus law'' must be kept apart. The empirical law of a finite corpus is atomic,
so the second half of clause~(iii) does not apply to it, though the first
does. Its exact fit has infinite loss on
every string of the corpus length that the corpus does not contain
(Corollary~\ref{cor:training-mixture}(iv)). Memorisation studies measure
how close trained models come to that
extreme~\cite{feldman2020memorization,carlini2023quantifying}. The theorem
applies as stated to the population law. In both readings the law is
over tokens, and the law over characters is a marginal over
tokenisations (Section~\ref{rem:tokenisation}).

Clause~(ii) is the constructive side. Clause~(iv) says that no
computable corpus carries every class's constants, so the simplicity
prior is not recovered by fitting a computable corpus and has to be
supplied
(Corollary~\ref{cor:closure}). It can be computed over a declared class,
as in enumerated mixtures over programs up to a declared budget
(companion work~\cite{ficl2026,cdm2026}). Or it can be amortised, by
pretraining on samples from such an
enumeration~\cite{grau2024universal}; this extends reach but gives up the
certificate (Theorem~\ref{thm:si-underdet}). At finite loss the closure
is approached at a rate (Corollary~\ref{cor:training-certify}(i)).

\paragraph{What the constant buys.}
By Fact~\ref{fact:si-dominance}, for every computable model $q$ and
every string $x$ of length $n$, $L_n(M;x)\le L_n(q;x)+K(q)+O(1)$. The
on-corpus premium of universality over a given model is therefore at
most $K(q)+O(1)$ bits, paid once, pathwise, at every horizon. $K(q)$ is
at most the length of a program that stores the weights and runs
inference, and prequential coding describes a trained network far more
briefly than its stored weights~\cite{blier2018description}, so the figure
below is conservative. For $N_{\mathrm{par}}$ parameters stored at $16$ bits,
trained on $N_{\mathrm{tok}}$ tokens coded at $b$ bits per token, the
premium is at most about $16N_{\mathrm{par}}/(b\,N_{\mathrm{tok}})$ of the
corpus code length. The 8-billion-parameter Llama~3 model saw about
$1.5\times10^{13}$ tokens, nearly $2\times10^{3}$ per
parameter~\cite{dubey2024llama3}. With $16$-bit weights, at least $10^3$ tokens per parameter and
$b\approx2$, the ratio is below one percent; for the 70-billion-parameter
model, on the same token count (about $200$ per parameter), it is about
four percent, and at a few tens of tokens per parameter it would be a
sizeable fraction. The premium is bounded and
the reverse trade (clause~(iii)) is not
(Remark~\ref{rem:price-of-universality}).

The compression view of language modelling~\cite{deletang2024compression}
reaches the same accounting from the other side: once model size is
charged, large models compress small datasets poorly, and that charge is
the $K(q)$ premium above. A compression rate measured on one corpus
measures allocation on that corpus (off-corpus rates:
Section~\ref{sec:discussion}).

\paragraph{Post-training, data mixing and refitting.} Post-training is a reallocation:
at a fixed prompt, the exact optimiser of a KL-regularised reward objective changes
the regret on each environment by an affine function of that environment's mean
reward, a price no KL budget bounds; doubling a source's weight buys one bit on it
and costs bits on the others; and maximum-likelihood refitting on generated data
drifts to a vertex of the allocation. Appendix~\ref{app:posttrain-main} states these
and relates them to the documented calibration and pass@$k$ costs of post-training.

\providecommand{\db}[2]{d\!\left(#1 \,\|\, #2\right)}
\providecommand{\tlo}{t^{-}}

\section{What is inheritable}
\label{sec:attainable}

This section collects what a
computable predictor can inherit from Solomonoff induction. Pretraining
sets an allocation, the context sets a posterior under it, and tools and
test-time compute change the class over which the allocation is spread.

\paragraph{Class mixtures.}
Hutter's mixture bound~\cite{hutter2003optimality,hutter2005uai} is the
answer on declared classes. Let $\Cc=(\mu_i)$ be a uniformly computable
family and $w$ a computable prior with $w_i>0$, $\sum_iw_i\le1$ and
computable tail bounds. Then the Bayes mixture $\xi_w$ of
\eqref{eq:bg-mixture} is computable, and
$R_\infty(\xi_w;\mu_i)\le\log(1/w_i)$ for every $i$ (the mixture
bound, Fact~\ref{thm:attainable}), with equality on a singular family
(Theorem~\ref{thm:frontier}(ii)). Computable predictors thus attain the
\emph{form} of Fact~\ref{fact:si-solomonoff}, with two substitutions: an enumerable
class replaces the universal one, and $\log(1/w_\mu)$ replaces $K(\mu)$.
The second substitution is forced: on any infinite uniformly computable
family of distinct deterministic environments, no computable predictor
has regret within a constant of $K$ (Proposition~\ref{prop:nokform}).
Only the universal class itself is excluded, by the diagonal of
Section~\ref{sec:scale}. The standard predictors are rungs on this
ladder, with their constants in Appendix~\ref{app:corpus}: the Krichevsky--Trofimov
estimator~\cite{krichevsky1981}, context-tree
weighting~\cite{willems1995}, exhaustively enumerated bounded-program mixtures
(companion work~\cite{ficl2026}), and Levin search and the speed
prior~\cite{levin1973,schmidhuber2002speed}.

\paragraph{The hierarchy is strict.}
Take any total rational-computable predictor $q$ and add its diagonal environment
(Theorem~\ref{thm:si-diag}) as a mixture of joint laws with a computable
weight $\delta\in(0,1)$. The result costs at most $\log\frac1{1-\delta}$ on every source
at every horizon. On the diagonal it has regret at most
$\log\frac1\delta$, where $q$ pays $n\log V$ (the non-collapse
theorem, Theorem~\ref{thm:nocollapse}). So no computable predictor is
maximal or complete, and coverage of a whole class is priced by
description length. A model's inductive position is
therefore a class together with constants on it, not a scalar degree of
universality.

\subsection{In-context learning as amortised Bayesian inference}
\label{sec:icl}

On a class mixture, the conditional law after a context $c$ is
$\sum_k\pi(k\mid c)\,\mu_k(\cdot\mid c)$. Pretraining fixes the weights;
the context moves the posterior. This is the formal content of the
Bayesian reading of in-context learning~\cite{xie2022bayesian}. It has
been tested where the prior is known, and its nulls probed on pretrained
models, where it is not.
\begin{itemize}
\item Transformers meta-trained on sampled tasks approximate the
posterior predictive~\cite{muller2022pfn,panwar2024prism}, and the
error of the Bayes-optimal such predictor decomposes into a
meta-learning and an intra-task term, both controlled by the mutual
information between the task and the observed
sequence~\cite{jeon2024ict}, the same quantity that carries the
memory term of Proposition~\ref{prop:window-tax}.
\item Meta-training on conditionally i.i.d.\ tasks predicts exchangeable
demonstrations~\cite{muller2022pfn}. Few-shot accuracy of pretrained
models depends on demonstration order~\cite{lu2022order} and on
majority-label and recency biases~\cite{zhao2021calibrate}, violations
of the order null that exchangeable Bayes predictors satisfy
(Proposition~\ref{prop:order}).
\item With few pretraining tasks, the model is Bayes-optimal for that
discrete prior over seen tasks. Past a task-diversity threshold it
departs from that Bayes predictor, even on the pretraining tasks, towards
ridge regression, Bayes for a Gaussian prior, and so generalises to
unseen tasks~\cite{raventos2023diversity}: a small in-distribution excess
buys a different allocation off the pretraining tasks, as
Proposition~\ref{prop:pretrain-certify} permits.
\item Falck et al.\ find violations of the martingale property, and
detecting them requires no knowledge of the
prior~\cite{falck2024martingale}; Chlon et al.\ trace such martingale and order
violations to positional encodings, which break exchangeability, and
argue that they are compatible with Bayesian prediction in expectation
over orderings~\cite{chlon2025martingale}.
\item Min et al.\ find that replacing demonstration labels with random
ones barely lowers accuracy on many classification tasks, while the label
space, the input distribution and the format
matter~\cite{min2022rethinking}. In the mixture reading, such demonstrations
carry evidence mainly about which task is in play: Pan et al.\ separate
this task recognition, which survives random labels, from task learning,
which needs correct labels and emerges with
scale~\cite{pan2023recognition}. Larger models override semantic priors when
in-context labels are flipped~\cite{wei2023larger}, a direct reading of
prior strength against in-context evidence.
\end{itemize}
A mechanistic line explains the same behaviour by the algorithm a forward
pass runs. Transformers trained on simple function classes learn them in
context~\cite{garg2022icl}, and on linear regression their predictions
track gradient descent or ridge regression, which a forward pass can
implement~\cite{vonoswald2023gd,akyurek2023icl}. The two readings agree
on how a posterior can be computed but can disagree on which prior is
held: when the implemented algorithm is Bayes for a prior other than the
pretraining law, as past Ravent\'os et al.'s diversity threshold, the
mechanism, not the pretraining mixture, fixes the allocation. Whether
in-context learning appears at all depends on properties of the data,
such as burstiness and Zipfian class frequencies~\cite{chan2022data},
and it can be transient, giving way to in-weights learning under further
training~\cite{singh2023transient}. Wurgaft et al.\ model both the
threshold and transience as Bayesian model selection between a
memorising and a generalising predictor under a loss--complexity
trade-off~\cite{wurgaft2025rational}: the allocation lens applied to
predictors, with an explicit simplicity penalty, consistent with the
$K(q)$ premium of Section~\ref{sec:closure}. As to which prior is held,
parameter--function maps of small networks at initialisation favour
simple functions~\cite{valleperez2019,mingard2025occam}, and SGD
approximately samples the Bayesian posterior under that
prior~\cite{mingard2021sgd}. Trained
transformers prefer low-complexity sequences and
hypotheses~\cite{goldblum2024nfl,deora2025occam}. Meta-learning yields Bayes-optimal
sequential strategies~\cite{ortega2019metalearning}, networks trained on
samples from a universal mixture amortise it~\cite{grau2024universal},
and the in-context objective is a prequential code
length~\cite{elmoznino2024occam}. Each measures allocation on a finite family
through a complexity proxy; none could certify constants on an unbounded
class (Corollary~\ref{cor:si-uncert}), and an amortised prior is
underdetermined off its task-sampling law. What does a pretraining loss
certify?

\paragraph{Setting.} Let $(\mu_k)_{k\in J}$ be probability measures on
$\Ac^\infty$ (tasks), $J$ countable, $\pi$ a prior with all $\pi_k>0$,
and $P:=\sum_k\pi_k\mu_k$ the pretraining law. Fix a horizon $n$ (the
context window). The excess loss of $q$ over $P$, cumulative over the
window, is
$\psi:=\dl{P_n}{q_n}=\E_P[L_n(q;X_{1:n})]-H(P_n)$. Write $\db wt$ for
binary relative entropy and
\[
  \tlo(\lambda,\psi):=\min\{t\in(0,\lambda]:\db{\lambda}{t}\le\psi\},\qquad
  s^*(\eta,\psi):=\max\{s\in[0,1]:\db{s}{s2^{-\eta}}\le\psi\}.
\]

\begin{proposition}[What pretraining loss certifies; after~\cite{zhang2023icl,wies2023learnability}]
\label{prop:pretrain-certify}
Let $q$ be any predictor with $\psi<\infty$.
\begin{enumerate}
\item[\textup{(i)}] \textup{(Bayes-relative inheritance, any family.)}
For every probability measure $\nu$ and $\lambda\in(0,1]$ with
$\nu_n\le P_n/\lambda$,
\[
  R_n(q;\nu)-R_n(P;\nu)\;\le\;\log\frac{\lambda}{\tlo(\lambda,\psi)},
  \quad
  \log\frac{\lambda}{\tlo(\lambda,\psi)}
  \;=\;\sqrt{\frac{2(1-\lambda)\,\psi}{\lambda\ln2}}
  +\frac{(1+\lambda)\psi}{3\lambda}+O(\psi^{3/2})
\]
as $\psi\to0$ at fixed $\lambda$.
The remainder is not uniform in $\lambda$: its constant grows without bound as
$\lambda \to 0$, so the leading term is a useful reading of the
certificate only while $\psi\ll\lambda$, and the exact left-hand side must be
used otherwise. If $P_n$ carries an event $E$ of mass exactly $\lambda$, then
for $\nu_n=P_n(\cdot\mid E)$ the bound is attained by some $q$ with excess
$\psi$. In particular the bound holds for each
task, with $\lambda=\pi_k$.
\item[\textup{(ii)}] \textup{(By prior mass.)} For every $\eta>0$,
$\pi\{k:R_n(q;\mu_k)-R_n(P;\mu_k)\ge\eta\}\le s^*(\eta,\psi)\le
\psi/c(\eta)$, where $c(\eta):=\eta-(1-2^{-\eta})\log e>0$.
\item[\textup{(iii)}] \textup{(Kraft floor off the support.)} If $\nu_n$
and $P_n$ have disjoint supports, then $2^{-\psi}+2^{-R_n(q;\nu)}\le1$.
Hence $R_{n'}(q;\nu)\ge\log\frac1{1-2^{-\psi}}$ for all $n'\ge n$, with
equality at $n'=n$ iff $q_n=2^{-\psi}P_n+(1-2^{-\psi})\nu_n$.
\end{enumerate}
\end{proposition}

\noindent(Proof: Appendix~\ref{app:proof-prop:pretrain-certify}.)

\paragraph{In words.} A pretraining loss certifies the \emph{Bayes}
constants of the pretraining prior, and it does so by prior mass, not
task by task. An excess of $\psi$ bits per window leaves room to degrade
all tasks in a set of prior mass about $\psi/c(\eta)$ by $\eta$ bits each.
A single task of weight $\pi_k$ can lose about
$\sqrt{2(1-\pi_k)\psi/(\pi_k\ln2)}$ bits while $\psi\ll\pi_k$, and about
$\psi/\pi_k$ beyond that; at $\pi_k=2^{-20}$ and $\psi=10^{-3}$ the first
expression reads $55$ bits against an exact certificate of $1050$. Pretraining loss barely
constrains competence off the support: regret $R$ there is compatible
with an excess of only $\log\frac1{1-2^{-R}}\approx2^{-R}/\ln2$ bits
(attained).

Clause~(i) is the log-loss form, with tight constants, of Zhang et al.'s
bound~\cite{zhang2023icl}. That bound controls in-context regret by the
Bayes-model-averaging regret plus a coverage constant
$\gamma_{\mathrm{cov}}$, here $1/\lambda$, times the pretraining error,
so the dependence on the weight of the task is intrinsic.
Wies et al.\ need a floor on component weights for the same
reason~\cite{wies2023learnability}. The simpler bound $\psi/c(\eta)$ is tight
only as the mass tends to zero. When $\inf_k\pi_k=0$, no rate holds uniformly
over tasks (Corollary~\ref{cor:training-certify}). Clause~(iii) is the
exact price of specialisation at a finite horizon, and it needs disjoint
supports. Generalisation beyond the pretraining
tasks~\cite{raventos2023diversity} relates to it only by analogy: those
regression tasks carry Gaussian noise and overlap the pretraining law.

Three limits apply. First, the task-mixture model of a corpus is an
idealisation~\cite{xie2022bayesian}, literal for prior-fitted networks,
synthetic-task pretraining and amortised enumerations. Second, $\psi$ is
cumulative over the window, not per token: a per-token excess
$\eps=10^{-3}$ bits (Assumption~\ref{ass:si-eps} with $\rho$ the
position-averaged prefix law of $P$) gives $\psi=4096\,\eps\approx4.1$ bits
over $4096$ tokens, enough to degrade every task
by four bits. At prior mass $\lambda=2^{-k}$, the certificate
$\log(\lambda/\tlo(\lambda,\psi))$ of
clause~(i) on in-context regret in excess of the Bayes mixture's has
leading term $\sqrt{\psi\,2^{k+1}/\ln2}$ bits while $\psi\ll2^{-k}$ and is nearly
$\psi\,2^{k}$ beyond. Computed exactly, for a best-case $\psi=10^{-3}$ bits
per window it is $2.14$ bits at $k=10$, where $\psi\,2^{k}\approx1$ lies
between the two regimes and the leading term reads $1.72$, and $1050$ bits
at $k=20$; at $\psi\approx4.1$ bits it is $4.2\times10^{3}$ bits already at $k=10$.
The certificate is exponentially weak in $k$, where Solomonoff's
guarantee is linear. Third, a split of $\psi$
into prior error and within-task misfit exists, but it depends on a
convention (Proposition~\ref{prop:training-ledger}).

\paragraph{Steering by demonstrations.}
The same posterior that makes in-context learning work makes a Bayes
mixture steerable, quantitatively (Proposition~\ref{prop:manipulable}):
if every demonstration multiplies each competitor's likelihood ratio by
at most $2^{-r}$, the served continuation law is within total variation
$\eta$ of the target hypothesis $h_y$ once the number of demonstrations
reaches $(\log(1/w_{h_y})+\log(1/\eta))/r$. A matching necessary length
holds for concentration of the posterior. Choosing the demonstrations on
which a model is least confident, as Wan and Mei
propose~\cite{wan2025solomonoff}, is what the proposition already favours, since
steering speed is set by the per-demonstration evidence rate $r$,
largest where hypotheses disagree; no Solomonoff approximation is needed
for it. That many-shot jailbreaking~\cite{anil2024} works by this
mechanism is a hypothesis, supported by the fit of Bayesian scaling laws
to many-shot curves~\cite{arora2024bayesian}.

Wolf et al.~\cite{wolf2024alignment} are the direct antecedent. In a
mixture model of behaviours, any behaviour of positive prior weight can
be elicited by a prompt whose length grows with the logarithm of the
inverse weight. The defence this suggests is dogmatism about designated
behaviours imposed on the predictive distribution (a zero prior weight alone
does not suffice, since other hypotheses can reproduce the excluded one;
Section~\ref{sec:undesirable}), and it is not what post-training is known to do:
safety alignment can be shallow, concentrated in the first few output
tokens~\cite{qi2025shallow}. Non-dogmatism about the world entails steerability
by whoever controls the context, whenever demonstrations carry evidence at a
positive rate against every competitor (Proposition~\ref{prop:manipulable};
Section~\ref{sec:undesirable}).

\paragraph{Tools, retrieval and test-time compute.} A tool or retriever combined with
the model as a joint mixture keeps each component's bound at the cost of its weight,
while interpolation of conditionals and selection between them need not; and extra
computation at test time changes what the model computes, not the prior it holds.
Appendix~\ref{app:tools-more} gives the statements.

\section{Asymptotic criteria cannot see allocation}
\label{sec:si-tail}

The conclusions that universal approximation and a strictly proper loss
can deliver are asymptotic: consistency, merging, vanishing per-step
excess loss over the source. The classical Bayesian versions of these
properties hold for every prior that gives the truth positive weight.
Doob's consistency theorem~\cite{doob1953}, merging of
opinions~\cite{blackwelldubins1962} and weak merging~\cite{kalailehrer1994} all
have this character. They are also weaker than they sound in the other
direction: Doob's theorem exempts a null set of parameters, and
\cite{diaconisfreedman1986} show the exemption can bite, exhibiting
priors with full support under which the posterior concentrates on the
wrong answer. Consistency is therefore cheap in both senses: it holds
for almost any prior that covers the truth, and where it fails it fails
without warning. This strengthens rather than weakens the reading
below: a criterion that almost every prior satisfies cannot distinguish
among priors, and failing it does not single out any prior either. Solomonoff's guarantee is a statement about
constants. The elementary results below (formal statements in
Appendix~\ref{app:limits-tail}) make precise why a property invariant
under finite modification and reweighting cannot entail a bound that both
operations destroy.

\paragraph{Tail invariance.} Let $q'$ agree with $q$ outside a finite set
$F$ of contexts, and let both have all conditionals on $F$ at least
$\gamma > 0$. Then, for every source $\mu$ and horizon $n$
(Fact~\ref{lem:si-tail}),
\begin{equation}
\label{eq:tail-bound}
|R_n(q;\mu) - R_n(q';\mu)| \le \Lambda(F)\log(1/\gamma),
\end{equation}
where $\Lambda(F)$ counts the \emph{distinct lengths} occurring in $F$.
For $F = \Ac^{<N}$ this constant is $N$ rather than $(V^{N}-1)/(V-1)$.
Finite total regret, $R_n/n \to 0$ and vanishing per-step excess loss
are therefore unchanged by any such modification; we call properties of
this kind \emph{tail-invariant}.

\paragraph{The witness $M_N$.} Let $M_N$ predict uniformly on contexts of
length below $N$ and like $M$ afterwards:
\begin{equation}
\label{eq:tail-MN}
M_N(\cdot\mid x) := \mathrm{Unif}(\Ac)\ \ (|x| < N),\qquad
M_N(\cdot\mid x) := M(\cdot\mid x)\ \ (|x| \ge N).
\end{equation}
By \eqref{eq:tail-bound}, every $M_N$ keeps all of $M$'s consistency
properties on computable sources. Yet on the constant sequence
$0^{\infty}$, of complexity $O(1)$, it pays $N\log V + O(1)$. For every
function $f$, the property ``$R_\infty(\cdot;\mu) \le f(K(\mu))$ for all
computable $\mu$'' therefore fails for all large $N$
(Theorem~\ref{thm:si-tail}).

\paragraph{Prior-blindness.} Tail invariance concerns finite
modifications. Reweighting a prior is a second operation that asymptotic
criteria cannot see. On the binary class of sequences
$z^{(j)} := 0^{j}1^{\infty}$, the Bayes mixture \eqref{eq:bg-mixture}
has regret \emph{exactly} $\log(1/w_j)$ on $z^{(j)}$ at every horizon
$n \ge j+1$. Its per-step excess loss is exactly zero after step $j+1$.
Its constants can be set to $f(j)+j+1$ for any prescribed $f$
(Proposition~\ref{prop:si-prior}; cf.\
\cite{wood2013nonequivalence,leike2015bad}). At $n = \infty$ this is the
attainment clause of Theorem~\ref{thm:frontier}. For an arbitrary
pairwise singular family the same holds along the whole frontier
(Corollary~\ref{cor:blind}). All full-mass Bayes mixtures with positive
weights, and their leaky variants, share every consistency property,
while their regret profiles differ by unbounded amounts. Frontier
membership is itself not tail-invariant: changing a frontier mixture's
conditional at a single context where some member charges every
one-symbol continuation moves it strictly inside the Kraft set.

\paragraph{The implication in (C) is invalid.} Combining these results
gives Corollary~\ref{cor:si-invalid}. On any uniformly computable countable
class of sources with a deterministic member, a consistency property satisfied by
a computable Bayes mixture does not entail a complexity bound on that
class for any $f$. The finite-$\eps$ premise of
Assumption~\ref{ass:si-eps} is not tail-invariant, but by
Theorem~\ref{thm:si-underdet}(ii) a modification on a set $B$ changes it
by at most $\rho(B)\log(1/\gamma)$. Placing the uniform modification on
the contexts of a simple deterministic environment, once that
environment has left the region the corpus produces, gives the same
separation at an $\eps$-cost as small as $\log V$ times the $\rho$-mass of
those contexts. ``Language models are, or approach, $\eps$-optimal under the
log-loss'' does not entail ``language models approach Solomonoff
induction'', and the separating witnesses are computable and explicit.

\paragraph{In words.} Consistency, merging and vanishing per-step
divergence cannot distinguish Solomonoff's predictor from one that
learned nothing for its first $N$ steps. Nor can they distinguish two
Bayesian predictors whose priors differ arbitrarily. What they cannot
see is the allocation of the conserved budget.

\subsection{What ``language models are Bayesian'' can mean}
\label{sec:tail-bayesian}

Bayesian descriptions of in-context learning (Section~\ref{sec:icl}) have
also been tested on coin flips~\cite{gupta2025coinflips} and on flight and
hotel recommendation, where Bayesian behaviour appears only after training
on a Bayesian model's predictions~\cite{qiu2025bayesian}. The results
above constrain what such descriptions can carry, in three ways.

First, ``Bayesian'' without a declared class and prior does not restrict
a coherent model: every probability
measure $q$ on $\Ac^{\infty}$ is the integral mixture, with prior $q$
itself, of the point masses on sequences, and its conditionals are the
corresponding posterior predictives. The content of the description lies
in the declared class and in coherence, which Section~\ref{sec:tests}
treats.

Second, once a class is declared, every asymptotic signature of Bayesian
behaviour is shared by all priors that give the truth positive weight
(prior-blindness, above), so observing one identifies no prior, nor
whether it came from data, architecture or post-training.

Third, even the universal class does not fix a prior. Universal priors built
from different reference machines agree only up to a multiplicative constant,
and that constant is unbounded over machines, so at any finite horizon the
choice of machine is unconstrained by the universality requirement; the class
of universally dominant priors is in any case strictly larger than the
Solomonoff--Levin constructions~\cite{wood2013nonequivalence}. For agents, Leike and
Hutter~\cite{leike2015bad} show that a suitable choice of universal prior makes
Legg--Hutter intelligence~\cite{legghutter2007}, and with it balanced Pareto optimality,
entirely subjective, and that every policy is Pareto optimal in the class
of all computable environments. The elementary prediction analogue we
use is not claimed there. For any computable measure $\nu$ and
$\delta \in (0,1)$, the predictor $(1-\delta)\nu + \delta M$ dominates
every lower-semicomputable semimeasure, so it is universal in the
dominance sense of Fact~\ref{fact:si-dominance}. Its regret exceeds
$\nu$'s by at most $\log\frac1{1-\delta}$ on every source, and exceeds
Solomonoff's bound by at most $\log\frac1\delta$ (the argument of
Theorem~\ref{thm:nocollapse}). Every computable predictor is thus within
$\log\frac1{1-\delta}$ bits of some universal predictor, which is weaker
than being the posterior predictive of a universal prior but enough to
make ``universal'' uninformative about behaviour.

Taken together, these points mean that a claim that a model ``is
Bayesian'', or ``approximately Solomonoff'', has empirical content only as
a statement about constants over a declared class at finite horizon.
Reading such constants from served conditionals is the subject of
Section~\ref{sec:readprior}.

\providecommand{\dMem}{\mathrm{Mem}}
\providecommand{\dExc}{\mathcal{E}}
\providecommand{\dwin}{\mathrm{win}}

\section{The deployed regime}
\label{sec:deployed}

Sections~\ref{sec:evaluation}--\ref{sec:si-tail} treat a model as one
computable predictor on unbounded streams
(Definition~\ref{def:si-computable}). A deployed model reads through a
context window of $W$ tokens and is sampled at a temperature, often with
truncation. It predicts tokens rather than characters, and it increasingly
answers after a reasoning trace. This section says what each of these does
to the preceding results.

\paragraph{What keeps force within a window.}
At horizons $n\le W+1$ a window-$W$ model sees its whole history, so every
finite-horizon statement of the paper applies verbatim: conservation and
the Kraft floors (Fact~\ref{thm:conservation} and its corollaries),
underdetermination (Theorem~\ref{thm:si-underdet}), Hutter's mixture bound
(since $R_n\le R_\infty$), the identities of
Section~\ref{sec:readprior}, and the diagonals, where the single diagonal
forces loss $n\log V$ and the class diagonal only
$(\tfrac12 nd_i-2)\log V$ on member $i$
(Theorems~\ref{thm:si-diag} and~\ref{thm:si-classdiag}). Statements about
limits lose their content: consistency and merging, tail invariance and its
corollaries (Section~\ref{sec:si-tail}), and the infinite-horizon frontier (Theorem~\ref{thm:frontier}). A model
whose memory is its window reaches $n\to\infty$ only through that window,
which the next result prices.

\paragraph{Context windows.}
Write $\dwin_W(x)$ for the last $W$ symbols of $x$ (all of $x$ if
$|x|<W$). A predictor is \emph{window-$W$} if $q(\cdot\mid x)$ depends on
$x$ only through $\dwin_W(x)$, and \emph{time-aware window-$W$} if it
depends on $x$ only through $(|x|,\dwin_W(x))$. Sliding-window attention
with relative positions gives the first kind, and absolute positions give
the second. Retrieval, recurrent state and compressed caches are covered by
neither (Appendix~\ref{app:deployed}).

For a source $\mu$, the \emph{memory term}
$\dMem^W_n(\mu) := \sum_{t\le n} I_\mu(X_t;X_{<t}\mid t,\dwin_W(X_{<t}))$
is the information about the next symbol that the past carries and the
window does not. For stationary $\mu$, let $h_\mu(j) := H_\mu(X_{j+1}\mid X_{1:j})$
and let $h_\mu$ be the entropy rate. The \emph{excess entropy} is
$\dExc(\mu) := \sum_{j\ge0}(h_\mu(j)-h_\mu)$~\cite{crutchfield2003regularities}; it is the limit of the predictive
information of~\cite{bialek2001predictability}.

\begin{proposition}[The window tax; elementary, after~\cite{bialek2001predictability,sharan2018memory,crutchfield2003regularities,jelassi2024repeat}]
\label{prop:window-tax}
Let $(\mu_j)_{j\in J}$ be countably many probability measures, let $w$ be
a probability vector on $J$ with entropy $H(w)$, Bayes mixture $\xi_w$ and
random index $\kappa\sim w$, let $X\sim\mu_\kappa$ given $\kappa$, with
$I_w(\kappa;X_{1:n})$ their mutual information, and let $W\in\Nb$.
\begin{enumerate}
\item[\textup{(i)}] \textup{(Identity.)} For every $n$ and every time-aware
window-$W$ predictor $q$,
\[
\sup_j R_n(q;\mu_j) \;\ge\; \sum_j w_j R_n(q;\mu_j) \;\ge\; I_w(\kappa;X_{1:n}) + \dMem^W_n(\xi_w).
\]
The second inequality is attained. Without the window constraint, the
least average regret is $I_w(\kappa;X_{1:n}) \le H(w)$.
\item[\textup{(ii)}] \textup{(Stationary cap.)} Let every
$\mu_j$ with $w_j>0$ be stationary. Then, for all $n$,
\[
\dMem^W_n(\xi_w) \;\le\; (n-W-1)^{+}\,\frac{H(w) + \sum_j w_j\,\dExc(\mu_j)}{W+1}.
\]
If moreover $H(w)<\infty$, $\dMem^W_n(\xi_w)$ grows linearly in $n$ unless
all such $\mu_j$ are $W$-Markov with a common kernel
(Proposition~\ref{thm:window}(iii)).
\item[\textup{(iii)}] \textup{(Copy witness.)} Let $\omega_\ell$ be the law of
$UUU\cdots$ with $U$ uniform on $\Ac^{\ell}$. It is stationary and
computable, with $K(\omega_\ell)\le K(\ell)+O(1)$ and
$\dExc(\omega_\ell)=\ell\log V$. Every time-aware window-$W$ predictor has
$R_n(q;\omega_{W+1}) \ge (n-W-1)^{+}\log V$, which meets the cap in (ii). The
witness is not isolated: equality holds whenever the mixture is a single
stationary component (or $H(w)=0$) that is $(W+1)$-Markov with
$h(0)=\dots=h(W)$, which includes every stationary order-one chain at $W=0$.
Meanwhile $R_n(M;\omega_{W+1})\le K(W+1)+O(1)$ for all $n$.
\end{enumerate}
\end{proposition}

\noindent(Proof: Appendix~\ref{app:proof-prop:window-tax}.)

\paragraph{In words.} An unrestricted Bayes mixture pays once to identify
the task, at most $H(w)$ bits (Proposition~\ref{prop:window-tax}(i)). A
model that sees $W$ tokens cannot keep that identification; the
proposition caps the cost and exhibits a source meeting the cap, but gives
no general rate. The cap (ii): on stationary tasks the memory term grows
at most at the identification cost plus the tasks' excess entropy, divided
by $W+1$, and linearly unless nothing beyond the window is there to be
inferred. The witness (iii): $\omega_{W+1}$ is the uniform mixture of the $V^{W+1}$
sequences of period $W+1$, of complexity at most $\log W+O(\log\log W)$. An
unrestricted predictor identifies the sequence once, for $(W+1)\log V$
bits; every window-$W$ predictor pays $\log V$ bits per token forever, the
whole identification cost again every $W+1$ tokens. On other stationary
families the tax lies anywhere between zero and the cap. For $W=10^{6}$ and
$V\approx1.5\times10^{5}$, a randomly drawn block repeated at lag $W+1$
costs about $17.2$ bits per token on the repeat. Evidence of in-context
learning at lengths below $W$ is silent on this regime.

Bialek, Nemenman and Tishby~\cite{bialek2001predictability} proposed
predictive information as the measure of what a learner can extract from a
sequence, and showed
that it grows like $\frac d2\log n$ for $d$-parameter families; the window
tax is the price of discarding the part that lies beyond the window. The
upper bound on a fixed-order predictor is due to
Sharan et~al.~\cite{sharan2018memory}, and (ii) is its stationary form for mixtures.
Copying lower bounds for bounded state
are in~\cite{jelassi2024repeat}; Appendix~\ref{app:deployed-window} gives
their log-loss form and shows that stationarity is needed in (ii).

\paragraph{Temperature.}
For a predictor $q$ and $T>0$, let
$q_T(a\mid x) \propto q(a\mid x)^{1/T}$. Since $(cq)_T=q_T$, tempering
normalises: $M_1$ is Solomonoff's normalisation $M_{\mathrm{norm}}$ of $M$.
For a computable $z\in\Ac^\infty$, let $K(z)$ be the length of a shortest
prefix program that computes $t\mapsto z_t$, and let
$L_\infty(q;z):=\lim_nL_n(q;z_{1:n})$.

\begin{corollary}[Only temperature one keeps Solomonoff's bound; elementary]
\label{cor:temperature}
There is a constant $c$ depending only on $U$ such that the following hold
for every computable $z$ and every computable i.i.d.\ source $\mu$ whose
marginal $\theta$ is not uniform on its support.
\begin{enumerate}
\item[\textup{(i)}] $T=1$: $R_\infty(M_1;\nu)\le K(\nu)+O(1)$ for every
computable measure $\nu$.
\item[\textup{(ii)}] $T<1$: $L_\infty(M_T;z)\le (K(z)+c)/T$, but
$\liminf_n \frac1n R_n(M_T;\mu) \ge \dl{\theta}{\theta_T} > 0$.
\item[\textup{(iii)}] $T>1$: the same linear bound on $\mu$ holds, and the
loss on $z$ is unbounded, $0^\infty$ included: for large $n$,
$L_n(M_T;z)$ equals $\Sigma_T(n):=\sum_{t\le n}2^{-K(t)/T}$ up to factors
depending only on $U$, $V$, $T$ and $K(z)$, and
$\Sigma_T(n)=n^{1-1/T\pm o(1)}$.
\end{enumerate}
\end{corollary}

\noindent(Proof and exact two-sided bounds:
Appendix~\ref{app:proof-cor:temperature}.)

The linear clauses are the converse of temperature scaling~\cite{guo2017calibration}: tempering a consistent predictor miscalibrates
it, nothing specific to $M$. Clause (iii) is, as far
as we know, not stated elsewhere; its mechanism is a Kraft sum over
one-step deviations. Up to factors in $K(z)$,
$M$ gives a deviation from $z$ at time $t$ conditional mass of order
$2^{-K(t)}$, a standard consequence of the dominance of $M$
(cf.~\cite{hutter2005uai}) that Appendix~\ref{app:deployed-temperature}
derives. Raising these masses to the power $1/T$ turns a convergent sum
into a divergent one. For any predictor, $T\le1$ multiplies loss on deterministic
streams by at most $1/T$, and $T>1$ turns a finite loss $L$ into
$O\bigl(V(n^{1-1/T}L^{1/T}+L)\bigr)$.

Sampling at temperature $T$ serves the measure
$Q_T(x):=\prod_t q_T(x_t\mid x_{<t})$. This is not the tempered joint law
$q(x)^{1/T}/Z$; the two differ already at horizon $2$. Tests on sampled
answers therefore test $Q_T$. Without truncation, $Q_T$ determines only
the normalised conditionals ($M_{\mathrm{norm}}$, not $M$), so leakage is
invisible to samples at every temperature.

\paragraph{Truncation, tokenisation and reasoning traces.} Top-$k$, top-$p$ and min-$p$
truncation, masks and majority votes are dogmatism by design: they give zero
probability to every excluded continuation, so no bound survives on it. A
tokeniser makes the read a marginal over tokenisations, and a reasoning trace
makes the answer law a marginal over traces. Appendix~\ref{app:deployed-more}
gives the details and what each does to the properties of
Section~\ref{sec:properties}.

\paragraph{What a served read reads.}
Serving stacks apply a checkpoint's default truncation unless it is
overridden, so a read must override it and say so. Commercial APIs return
only top-$k$ log-probabilities, from which full conditionals could be
recovered through logit bias on the API versions studied~\cite{finlayson2024logits,carlini2024stealing}; providers have since
restricted this. A server-side mask to an alphabet, renormalised step by
step, is over the full vocabulary itself a truncation; over the alphabet
the renormalised conditionals define a predictor, the one to which
Assumption~\ref{ass:elicit} is applied. It is neither the model's own
continuation law, which puts the discarded mass on other tokens, nor that
law conditioned on staying in the alphabet, which would reweight each
step by the probability that the rest of the continuation stays there.
A read of this kind must state the mask, the serving configuration and
the discarded mass; Paper~B describes such reads.

\section{Properties ranked, and those to avoid}
\label{sec:properties}

Solomonoff's predictor has a bundle of properties, of which universality is
only one. This section separates them, ranks them by weight for
trustworthiness, records which are attainable, and identifies those a deployed
system should lack. The ranking is a judgement; the attainability column is not.

\subsection{Two floors}
\label{sec:properties-floors}

Two properties, both attainable by computable systems, are floors: they say what
a trustworthy inductor is never worse than.

\begin{definition}[Non-dogmatism, per environment]
\label{def:nondogmatic}
A predictor $q$ is \emph{dogmatic} on an environment $z \in \Ac^{\infty}$ if
$\sup_t -\log q(z_t \mid z_{<t}) = \infty$, and \emph{non-dogmatic} on a class
of environments if it is dogmatic on none of them.
\end{definition}

The quantifier matters: on a computable deterministic $z$ Solomonoff's predictor
pays at most $K(z) + O(1)$ bits at any step (Fact~\ref{fact:si-solomonoff}), but
along $0^{t}1$ its final surprisal is unbounded in $t$, and a ceiling uniform over
all environments would force linear loss on every deterministic stream. A Bayes
mixture over a deterministic class has ceiling $\log(1/w_h)$ on each member $h$;
fixed share over $m$ experts, one of them uniform, with re-mixing rate $\alpha$
has ceiling $\log(m/\alpha) + \log V$~\cite{herbster1998}; truncated sampling
fails the definition (Section~\ref{sec:deployed}). Dogmatism is unrecoverable:
Bayesian updating cannot raise a zero.

\begin{definition}[Expert dominance]
\label{def:dominance}
A predictor $q$ \emph{dominates} an expert $p$ with constant $c$ on a set $X$ of
streams up to horizon $n$ if $L_{n'}(q;x) \le L_{n'}(p;x) + c$ for every
$x \in X$ and every $n' \le n$.
\end{definition}

A Bayes mixture giving $p$ weight $w_p$ dominates it with $c = \log(1/w_p)$
(Fact~\ref{thm:attainable}), and Solomonoff's predictor dominates every computable
expert with $c \le K(p) + O(1)$ (Fact~\ref{fact:si-dominance}), in particular the
Krichevsky--Trofimov (KT) estimator~\cite{krichevsky1981}; linear excess over KT
exposes a hole in the prior that the simplest expert fills. The floor has a
bounded-window scope: a predictor that sees only the last $W$ symbols over an
alphabet of $V$ has, on some i.i.d.\ source, per-step excess over KT tending to
about $(V-1)/(2W\ln 2)$ bits, the window tax of Proposition~\ref{prop:window-tax}
for the i.i.d.\ family (Appendix~\ref{app:readout-window}). A failure
within the window shows that KT is not a component of any mixture the model implements over that
horizon, not that the model is no mixture.

\paragraph{Calibration.}\label{rem:calibration}
A coherent predictor is calibrated exactly under its own law~\cite{dawid1982} and
under an evaluation distribution only asymptotically; Solomonoff's predictor is
calibrated on computable environments in the limit~\cite{solomonoff1978,hutter2005uai}, a tail property with no constant
(Fact~\ref{lem:si-tail}); Section~\ref{sec:closure} records what post-training
costs in calibration, and calibration itself has a price in generation: a
calibrated language model must hallucinate arbitrary facts at a rate close to
the fraction of such facts seen once in training~\cite{kalai2024calibrated}.

\subsection{The properties, ranked}
\label{sec:taxonomy}

In Table~\ref{tab:properties-taxonomy}, \emph{bounded reversal} means that a
coherent predictor's probability of a fixed event, a martingale under its own law,
rises from below $a$ to above $b > a$ at most $\min\{a,1-a\}/(b-a)$ times in
expectation (Doob's upcrossing inequality~\cite{doob1953} gives
$(1-a)/(b-a)$ and, since a martingale is a supermartingale, the same inequality in
its supermartingale form gives $a/(b-a)$; Dubins' inequality~\cite{dubins1966} bounds
the distribution of the number of upcrossings); \emph{label equivariance} means
that predictions commute with relabelling the alphabet.

\begin{table}[t]
\centering
\footnotesize
\setlength{\tabcolsep}{3pt}
\begin{tabular}{@{}r p{0.27\linewidth} p{0.11\linewidth} p{0.24\linewidth} p{0.28\linewidth}@{}}
\hline
\# & \raggedright\textbf{Property} & \raggedright\textbf{Attainable} & \raggedright\textbf{Result here} & \raggedright\textbf{Instrument in Paper~B}\tabularnewline
\hline
1 & \raggedright Coherence (is a measure)$^{a}$ & yes & Prop.~\ref{prop:coherence} & \raggedright coherence residual; rollout coherence\tabularnewline
2 & \raggedright Non-dogmatism & yes & Def.~\ref{def:nondogmatic} & \raggedright non-dogmatism\tabularnewline
3 & \raggedright Bounded reversal & yes & \raggedright Doob/Dubins upcrossing~\cite{doob1953,dubins1966} & \raggedright upcrossing ceiling\tabularnewline
4 & \raggedright Expert dominance (floor) & \raggedright within window & \raggedright Def.~\ref{def:dominance}, Prop.~\ref{prop:window-tax} & \raggedright expert dominance\tabularnewline
5 & \raggedright Order / sufficiency invariance$^{b}$ & yes & Prop.~\ref{prop:order} & \raggedright sufficiency and order invariance\tabularnewline
6 & \raggedright Label equivariance$^{b}$ & yes & \raggedright invariant prior & \raggedright label equivariance\tabularnewline
7 & \raggedright Calibration & \raggedright yes (tail) & Sec.~\ref{sec:properties-floors} & \raggedright reliability/ECE (standard)\tabularnewline
8 & \raggedright Complexity-weighted allocation & yes & Prop.~\ref{prop:slope} & \raggedright complexity slope, stratified budget\tabularnewline
9 & \raggedright Uniform constants on a class & \raggedright class-wise & Thm.~\ref{thm:frontier} & \raggedright cover saturation\tabularnewline
10 & \raggedright Merging / convergence & \raggedright yes (free) & App.~\ref{app:readout-prior} & \raggedright cheap-consistent foils\tabularnewline
11 & \raggedright Universality over computable $\nu$ & \textbf{no} & Cor.~\ref{cor:si-nocomputable} & \raggedright falsification only\tabularnewline
12 & \raggedright Machine invariance up to $O(1)$ & no & Cor.~\ref{cor:si-nocomputable}; \cite{wood2013nonequivalence} & \raggedright none\tabularnewline
\hline
\end{tabular}\\[3pt]
\parbox{\linewidth}{\footnotesize $^{a}$For the semimeasure $M$, after
normalisation. $^{b}$Of a Bayes mixture over a class exchangeable (respectively
closed under relabelling, with invariant prior) for the task; not of $M$.}
\caption{Properties of Solomonoff's predictor, ranked by weight for
trustworthiness.}
\label{tab:properties-taxonomy}
\end{table}

A property ranks higher the more its failure lets
a user be misled or exploited without a visible error. The properties that make an
inductor safe to rely on (rows 1--6) are all attainable, expert dominance within
the context window, while universality
(row 11), which cannot be had, ranks near the bottom because trust does not require it. The
constants that carry informative content (rows 4, 8, 9; Section~\ref{sec:si-tail})
need log-probabilities, or temperature or logit-bias control, the rows with exact black-box nulls (1, 5, 6) are
structural, and row 10, often invoked as evidence of inductive quality, carries
none.

A ranking of the ideal's properties should say where the ideal itself
fails, and both of its documented failures are failures of the budget it
leaves unspent. Solomonoff's predictor converges to every computable $\mu$
with $\mu$-probability one, which is a statement about sources and not
about individual sequences: \cite{huttermuchnik2007} show that some
universal semimeasures fail to converge on some $\mu$-Martin-L\"of random
sequences, while certain non-universal ones converge on all of them. And
when only selected positions carry structure (the even bits of a
sequence whose odd bits are noise), \cite{lattimore2011selected} show
that the normalised predictor $M_{\mathrm{norm}}$ still learns any
recursive pattern on the selected positions, whereas $M$ itself can fail
on very simple ones. That is the allocation thesis of
\S\ref{sec:frontier} turned on $M$: what $M$ leaks is budget not placed
on the continuations that are scored, and renormalising puts it back,
which is what sampling at temperature one serves
(Corollary~\ref{cor:temperature}). Universality buys a constant against
every computable source; what it buys at a place of the examiner's
choosing depends on where the rest of the budget goes.

\subsection{Properties that are undesirable in deployment}
\label{sec:undesirable}

Solomonoff induction is an ideal of belief formation about a fixed world. A
deployed system must also hold constraints, adapt to change and resist
manipulation, and each of these is served by violating one property of the
table, so a recommendation to make systems ``more Solomonoff-like'' is
underspecified until it says which properties, for which role.

\paragraph{Non-dogmatism is steerability.} A Bayes mixture is steered towards any
hypothesis of positive weight by demonstrations (Section~\ref{sec:attainable},
Proposition~\ref{prop:manipulable}). The defence is deliberate dogmatism about
designated hypotheses, imposed on the predictive distribution. A zero prior weight
does not suffice, because other hypotheses can reproduce the excluded one. By the
conservation law (Section~\ref{sec:conservation}), this dogmatism only moves budget
away from environments one never wishes to predict.

\paragraph{Occam allocation and corpus-optimality are choices.} Simplicity
weighting suits low-complexity environments, and its premium over a trained model
is small only under the conditions of Section~\ref{sec:closure}. A perfect corpus
predictor is the corpus law (Theorem~\ref{thm:closure}), and so, on our reading,
reproduces the corpus's errors and harms at corpus rates; post-training departs from the corpus law on purpose,
at the costs recorded there.

\paragraph{Exchangeability is the absence of an instruction hierarchy.} Order
invariance is desirable where demonstrations are conditionally i.i.d.\ by
specification (Proposition~\ref{prop:order}). In a conversation position is
informative: the system prompt should dominate the user turn, which is why an
instruction hierarchy is trained explicitly~\cite{wallace2024hierarchy}.

\paragraph{Beliefs and constraints.} The properties that remain desirable without
condition are the structural ones: coherence, the dominance floor within the
window, and non-dogmatism about the world. The rest are desirable for beliefs and
harmful for constraints, or right for a fixed world and wrong for a drifting one
(bounded reversal is inertia after a regime change, which fixed share~\cite{herbster1998} escapes at a tracking cost), and each has the helpful violation
named above. Our
position, motivated but not proved by the results, is that a trustworthy system
has Solomonoff-like beliefs and non-Solomonoff-like constraints. Solomonoff
induction is silent about objectives: the properties ranked here are epistemic
hygiene, the part of reliability that is a measurement problem.

\section{What served conditionals determine}
\label{sec:tests}

Tail-invariant statistics cannot see the constants
(Section~\ref{sec:si-tail}), and statistics on a fixed evaluation
distribution cannot see allocation off it (Theorem~\ref{thm:si-underdet}).
Merging of opinions is invariant under both (Appendix~\ref{rem:merging}). What remains readable is
coherence, invariance, constants against explicit reference values, and the
prior over a declared class. This section states these as definitions and
identities; estimators, noise floors and validation are in Paper~B.

\subsection{Elicitation and coherence}
\label{sec:tests-coherence}

\begin{assumption}[Elicitation]
\label{ass:elicit}
The served probability of an event $A$ at history $h$ equals the probability of
$A$ under the served model's own continuation law from $h$, both read at $T = 1$,
without truncation (top-$k$, top-$p$, min-$p$, grammar masks, server-side
defaults included) and on one fixed tokenisation. (For a masked read, the served
model is the masked, per-step renormalised predictor of
Section~\ref{sec:deployed}. Log-probabilities that the serving stack returns
before temperature scaling count as a read at $T = 1$, whatever temperature the
request carries; log-probabilities returned before masking read the unmasked
model, and renormalised over the allowed symbols they give the masked predictor's
conditionals. Whether a stack returns raw or processed values
must be checked and stated.)
\end{assumption}

\noindent The assumption is substantive: prompted probabilities differ from
direct measurements~\cite{hu2023prompting}, answers depend on prompt format~\cite{sclar2024sensitivity}, and a law on symbols is a marginal over
tokenisations (Section~\ref{sec:deployed}), so a character-level residual may
come from the tokeniser unless each symbol is one token. For post-trained models
verbalised confidence can be better calibrated than token probabilities~\cite{tian2023calibration}, so a report must state which served object is read.

\begin{proposition}[Coherence and arbitrage; (ii)--(iii) after de Finetti~\cite{definetti1937}]
\label{prop:coherence}
\label{prop:arbitrage}
Let $Q$ be the served conditionals and $A$ an event determined by the
continuation of $h$.
\textup{(i) (Null.)} If some probability measure has conditionals agreeing with
$Q$ on $\{h\} \cup \{h\sigma : \sigma \in \Ac\}$ and assigns $A$ the served
values, then
\begin{equation}
\label{eq:coh}
  \varrho(h;A) \;:=\; \Bigl| Q(A \mid h) - \sum_{\sigma \in \Ac} Q(\sigma \mid h)\, Q(A \mid h\sigma) \Bigr| \;=\; 0 .
\end{equation}
So $\varrho(h;A) > 0$ excludes every measure, hence every Bayes mixture under
every prior on every machine, and for an autoregressive model it refutes
Assumption~\ref{ass:elicit}; for the semimeasure $M$ the null holds after
normalisation.
\textup{(ii) (Arbitrage.)} Let served conditionals be prices at which the system
takes either side of unit claims on $A$ and on the next symbol, at $h$ and every
$h\sigma$. The portfolio that buys, for each $\sigma$, $Q(A \mid h\sigma)$ units of
the $\sigma$-claim at $h$ and, if $\sigma$ occurs, one $A$-claim at $h\sigma$ with
the proceeds synthesises the $A$-claim without refunds. Traded against the direct
claim, it pays a riskless profit $\varrho(h;A)$ per unit.
\textup{(iii) (Converse.)} If all served prices are probabilities of one
measure, no riskless portfolio exists; $\varrho \equiv 0$ for one event is
necessary for this, not sufficient.
(Proof: Appendix~\ref{app:proof-prop:coherence}.)
\end{proposition}

\noindent\label{sec:served-demo}Under a masked read, a positive residual refutes
Assumption~\ref{ass:elicit} for the masked, renormalised predictor, not for the
model's unmasked continuation law.

\subsection{Invariance and constants}
\label{sec:tests-invariance}

\begin{proposition}[Invariance nulls]
\label{prop:order}
\label{prop:suff}
Let $Q$ be a Bayes mixture over a class $\mathcal{H}$ under any prior.
\textup{(i) (Order.)} If demonstrations $d_1, \dots, d_t$ are conditionally
i.i.d.\ given the hypothesis, $Q(\cdot \mid d_{\tau(1)}, \dots, d_{\tau(t)})$
does not depend on the permutation $\tau$.
\textup{(ii) (Sufficiency.)} If contexts $c, c'$ satisfy (a)
$(P_h(c))_{h}$ and $(P_h(c'))_{h}$ are proportional and (b) $P_h(\cdot \mid c) = P_h(\cdot \mid c')$ whenever
$P_h(c) > 0$, then $Q(\cdot \mid c) = Q(\cdot \mid c')$. Neither condition can
be dropped. (Proof: Appendix~\ref{app:proof-prop:order}.)
\end{proposition}

\noindent The i.i.d.\ hypothesis in (i) is essential: in the hidden-Markov
account of in-context learning of Xie et al.\ \cite{xie2022bayesian},
demonstrations are not exchangeable. For deterministic $z$, the
\emph{truth loss} at horizon $n$ is $L_n(Q;z_{1:n}) = -\log Q(z_{1:n})$.

\begin{proposition}[Truth-loss levels under a description-length prior]
\label{prop:slope}
Let $Q$ be a Bayes mixture over deterministic hypotheses with prior $w$, and
let the truth $h^\dagger$ have description length $\ell(h^\dagger)$. Then
$L_n(Q;h^\dagger_{1:n}) \le \log(1/w_{h^\dagger})$ for all $n$, with equality
once $h^\dagger$ is the only consistent hypothesis of positive weight. If
$w_h \ge 2^{-\ell(h)-c}$ for all $h$, then $L_n - \ell(h^\dagger) \le c$ on
every environment of the class, with equality after identification when
$w_h = 2^{-\ell(h)-c}$; a uniform prior gives the flat level
$\log|\mathcal{H}|$; and no prior gives a small level on many environments at
once (Corollary~\ref{cor:kraftfloor}). (Proof: Appendix~\ref{app:proof-prop:slope}.)
\end{proposition}

\noindent Its measurable content: after identification the levels lie on
$\ell+c$ (unit slope in $\ell$) under a $2^{-\ell-c}$ prior and at
$\log|\mathcal{H}|$ (zero slope) under a uniform prior; before identification
only the upper bound holds, so a fitted slope has no exact reference value.

\paragraph{Reading budget statistics against the floor.}\label{rem:floorreading}
With a reference mixture $P_{\rm ref}$ over a task family,
$L_n(Q;z_{1:n}) = \log\bigl(P_{\rm ref}(z_{1:n})/Q(z_{1:n})\bigr) - \log P_{\rm ref}(z_{1:n})$,
pathwise regret plus the reference's own cost. On many distinct deterministic
tasks the second term is bounded below by the Kraft floor of
Section~\ref{sec:conservation} for every predictor, so raw budgets should be
reported net of $-\log P_{\rm ref}$, or beside its distribution.

\subsection{Reading the prior}
\label{sec:readprior}

``The model's prior'' exists only relative to a declared class, even within the
universal class (Section~\ref{sec:tail-bayesian}). Given the class, reading the
prior is an inverse problem. Iterated in-context learning, which feeds a model's
answers back to it as data, samples a model's prior over a task's answers~\cite{zhugriffiths2025priors}; the reading here instead fixes a declared class
and inverts served probabilities.

\paragraph{The implied prior.}\label{sec:selfeval}
Let $\mathcal{H}$ be a countable class of computable measures and $C$ a finite
probe set of prefixes on which $\mathcal{H}$ has finitely many distinct columns
$(P_h(c))_{c \in C}$, as every finite or deterministic class has; put
$G_{c,h} := P_h(c)$ and $q_c := Q(c)$. The \emph{membership
residual} $\min\{\|Gw - q\|_2 : w \ge 0,\ \sum w \le 1\}$ vanishes iff some Bayes
mixture over $\mathcal{H}$ agrees with $Q$ on $C$; all minimisers share their
fitted values, the weights on distinct columns are identified up to a polytope,
and the implied
coding $-\log w^{Q}_h$ is Kraft-valid (Proposition~\ref{prop:implied}). A positive
residual means $Q$ is not a mixture, or is one over a larger class. Over
deterministic classes it reads support, not allocation.

\begin{proposition}[Membership over a deterministic class is a support test]
\label{prop:support}
Let every $h \in \mathcal{H}$ be deterministic, $Q$ a measure and $L$ a probe
length. Call a prefix off-class if no $h$ is consistent with it and in-class
otherwise. For $|c| \le L$, let $u(c)$ and $u_1(c)$ be the $Q$-mass of the
off-class length-$L$ extensions of $c$ and the off-class continuation mass at $c$,
and let $I(c)$ be the set of symbols $a$ for which $ca$ is in-class. \textup{(i)} The residual is at least
$(\sum_{c \in O} Q(c)^2)^{1/2}$ over the off-class probes $O \subseteq C$.
\textup{(ii)} If $C = \Ac^{\le L}$, it is at most $(\sum_{c} u(c)^2)^{1/2}$ and
vanishes iff $Q$ puts no mass off-class, each in-class $x \in \Ac^{L}$ then
getting class weight $Q(x)$. \textup{(iii)} If $C$ contains $c$ and its in-class
extensions $ca$, $a \in I(c)$, it is at least $Q(c)\,u_1(c)/\sqrt{|I(c)|+1}$. (Proof:
Appendix~\ref{app:proof-prop:support}.)
\end{proposition}

\noindent Conditionals alone give the same reading by distillation into a
Bayesian student, provided the probe set is closed under prefixes and the student
carries a full-support column, and a corpus can supply the exact likelihood of one
hypothesis for the inversion to weigh (Appendices~\ref{app:readout-prior}
and~\ref{app:readout-corpus}); checks on known reference systems are in
Paper~B.

\section{Discussion}
\label{sec:discussion}

Claim (C) would require a computable system to hold Solomonoff's allocation, and
the known results show that none can. That is not a failure of the training
objective: what it does not fix is the hypothesis class and the prior, and by
the conservation law (Section~\ref{sec:conservation}) allocation, and how much
of the budget is used, are all there is to fix (the ledger of
Proposition~\ref{prop:frontier-ledger}). Table~\ref{tab:discussion-practice}
collects what the results imply for LLM practice.

\begin{table}[!tp]
\centering
\scriptsize
\setlength{\tabcolsep}{3pt}
\renewcommand{\arraystretch}{1.1}
\makeatletter
\let\A@startpbox\@@startpbox
\def\@@startpbox#1{\A@startpbox{#1}\raggedright\let\\\tabularnewline}
\makeatother
\begin{tabular}{p{2.6cm}p{2.7cm}p{4.6cm}p{4.8cm}}
\hline
\textbf{LLM practice} & \textbf{Result} & \textbf{What it certifies} & \textbf{Reporting rule} \\
\hline
Benchmark scores, held-out evaluation & Thm.~\ref{thm:si-underdet} & performance on the evaluated context distribution and horizon; nothing about allocation off it & state the context distribution and horizon; claim nothing beyond them \\
Adaptive evaluation, red teaming & Sec.~\ref{sec:eval-adaptive} & a successful attack falsifies; no finite campaign certifies & report attacks found and query budget, never a certificate of absence \\
Scaling trends & Thms.~\ref{thm:si-diag}, \ref{thm:si-classdiag} & shrinking excess loss on the evaluation law; convergence at best pointwise along the family, never uniform; each member has a diagonal environment & report constants on declared classes, not a trend towards Solomonoff's predictor \\
Convergence and merging as evidence & Sec.~\ref{sec:si-tail} & coverage only; invariant under finite changes and prior reweighting & report total excess loss, not limits or late per-step excess \\
Pretraining loss, in-context learning & Prop.~\ref{prop:pretrain-certify} & Bayes constants on a task family only through its prior mass & declare the task family; report losses net of the reference cost (Sec.~\ref{rem:floorreading}) \\
Compression rates, compression-based bounds~\cite{deletang2024compression,lotfi2024nonvacuous,lotfi2024tokenlevel} & Thm.~\ref{thm:si-underdet}, Figure~\ref{fig:positions} & loss on the data compressed, or on the law a bound covers; no constant on a declared class & state the data; compare with an adaptive classical coder on the same data \\
Data curation, mixture weights, synthetic data & Thm.~\ref{thm:closure}, Fact~\ref{thm:conservation} & a choice of prior; gains on one family are paid on others & report mixture weights and sources as the allocation chosen \\
Post-training (RLHF, DPO), best-of-$n$~\cite{beirami2024bestofn} & Sec.~\ref{sec:closure} & a reallocation: regret change affine in the reward (Gibbs tilt) or in a monotone transform of it (best-of-$n$); calibration loss documented empirically & state base or post-trained; report calibration for both \\
Sampling temperature & Cor.~\ref{cor:temperature} & Solomonoff's bound survives only at $T = 1$ & state $T$; read probabilities at $T = 1$ or untempered \\
Truncation, masks, majority vote & Sec.~\ref{rem:truncation}, Def.~\ref{def:nondogmatic} & dogmatism: no bound on excluded continuations & state top-$k$, top-$p$, min-$p$, masks and vote size; name the masked predictor read \\
Reasoning traces, test-time compute & Sec.~\ref{sec:deployed} & the answer law is a marginal over traces; self-consistency estimates its mode & report trace sampling and the number of traces \\
Tools and retrieval & Sec.~\ref{sec:attainable} & a joint mixture $q_\delta$ keeps each component's bound, up to $\log\frac1{1-\delta}$ or $\log\frac1\delta$; interpolation and selection need not & name the combination; report both components' guarantees \\
Long context & Prop.~\ref{prop:window-tax}, Def.~\ref{def:dominance} & floors within the window $W$ only; identification repaid beyond it & state $W$ and the horizon of every floor \\
Tokenisation, quantisation, batching & Sec.~\ref{rem:tokenisation}, Ass.~\ref{ass:elicit} & the measure actually read: a marginal over tokenisations, a random measure under serving & report the full serving configuration and replicate spread \\
Steering by demonstrations, refusals & Secs.~\ref{sec:icl}, \ref{sec:undesirable} & in a Bayes mixture, any hypothesis of weight $w>0$, in shots growing like $\log(1/w)$; for LMs, a hypothesis many-shot curves support & treat refusals as constraints on the predictive; report shots needed \\
``The model is Bayesian'' & Props.~\ref{prop:coherence}, \ref{prop:order}, \ref{prop:support}; Sec.~\ref{sec:tail-bayesian} & nothing without a declared class; with one, checkable nulls & declare the class; report coherence and membership residuals, order invariance and the implied prior's identifiability range \\
\hline
\end{tabular}
\caption{LLM practices, the results that bear on them, what they certify, and
what a report should state.}
\label{tab:discussion-practice}
\end{table}

\paragraph{The weaker forms, revisited.} Figure~\ref{fig:positions} records
what survives of each weaker form. The compression view's strongest evidence
lies off the corpus: a text-trained model compresses image patches and audio in
context better than PNG and FLAC, on 2048-byte chunks~\cite{deletang2024compression}.
Such a rate certifies no constant on a declared class
(Theorem~\ref{thm:si-underdet}). Whether it beats an adaptive low-order context
model (context-tree weighting~\cite{willems1995}, PPM) on the same chunks is not
reported by Del\'etang et al.~\cite{deletang2024compression} and, to our
knowledge, not since~\cite{heurteldepeiges2024compression}. If it does, the
network's allocation on data its pretraining law barely charges goes beyond what
in-context adaptive coding obtains; if not, the rate is explained by it.
Pretraining loss leaves that allocation almost free:
Proposition~\ref{prop:pretrain-certify}(i) at small prior mass, or (iii) at
disjoint support. There the simplicity-bias form
applies~\cite{valleperez2019,mingard2025occam,goldblum2024nfl}, and meta-learned
predictors~\cite{grau2024universal,ortega2019metalearning} are uncertified;
which prior a trained network holds is an empirical question for
Section~\ref{sec:tests} and Paper~B.

\paragraph{Limitations.} The Kraft-sum reading of the frontier is exact only for mutually singular
families (Theorem~\ref{thm:frontier}). For finite overlapping families the
achievable region is known (the overlapping frontier,
Proposition~\ref{prop:frontier-overlap}(i)); for countable non-singular families
only $C_\infty>1$ is proved (Proposition~\ref{prop:frontier-ledger}(ii)), and the
achievable region and the maximising priors are open. The window, temperature,
post-training and corpus statements concern idealised objects (window-$W$
predictors, exact tempering, the exact KL-regularised optimiser at a fixed
prompt, a corpus as a task mixture). Checks on known systems and
served-model readings are Paper~B's. The readings of
Section~\ref{sec:readprior} remove the declared prior but not the declared
class, which is the irreducible residue of any measurement of a model's prior.

\paragraph{What the measurements show.} Paper~B~\cite{paperB} carries these
instruments to served models: a panel of twenty-six checkpoints of several
pre-training families, sizes, architectures and serving precisions, each read as a
predictor over a three-symbol alphabet through a bare-digit rendering, restricted to
the digits and renormalised step by step, with eight re-read under a second rendering.
Three of its findings bear on the argument here. The served conditionals put mass on
continuations that no member of the declared class produces, as any model not
confined to the class must (Proposition~\ref{prop:support}). No checkpoint's median stays within
the $3.32$ bits by which a mixture carrying the Krichevsky--Trofimov code at weight
$0.10$ can trail it (the best trails it by $4.5$ bits, the median by $36.5$ and the
worst by $128$ at twenty-four symbols), but over $256$ symbols the three nearest that
bound settle to seven to thirteen bits, as a mixture carrying the code at a smaller
weight would, while three others keep losing ground; the ordering of these
distances also follows the mass each model puts outside the alphabet, so how much
of it belongs to the checkpoint is open. And the reading closest to
Solomonoff's allocation, a model's position between an enumerated mixture and a
complexity-blind anchor, tracks, except in one family, the mass the model puts off the
alphabet and moves with the rendering, so no effect of scale on allocation is
established; read as a semimeasure, as $M$ is, with the leaked mass counted as loss, the
high-leak checkpoints nearest the Krichevsky--Trofimov code on the class fall to the middle
of the panel. Separately, the serving floor of a served model's readings spans a factor
of seventy-five across the panel, is predicted on the second engine neither by weight
precision nor by architecture, and depends on the rendering, which is why
each reading is certified against its own serving floor. These are
measurements of particular deployments under stated renderings, not theorems. They
are consistent with the division argued here: mixture-like behaviour on a declared
class, with constants that differ by model, is measurable in served models, while
their allocation is not yet separable from the interface through which it is read.

\paragraph{Conclusion.} What language models can inherit from Solomonoff
induction is not its allocation: no computable model holds it, and none of the
routes offered for~(C), whether evaluation, scale, corpus fit or asymptotic
criteria, could establish that one does. What they can inherit is its form on
declared classes: Bayes mixtures attain Solomonoff-form guarantees there, and a
pretraining loss certifies in-context constants on a task family only through
that family's prior mass. Because every predictor spends the same Kraft budget,
the questions that remain are where a given model spends it and which
Solomonoff-like properties a deployed system should have: Solomonoff-like beliefs
under non-Solomonoff-like constraints. On a declared class, coherence residuals,
invariance nulls and the implied prior read that allocation from served
conditionals, and a companion paper turns them into tests with exact nulls. What
those tests report of any particular deployment is that paper's to establish and
not this one's; what this paper establishes is that, on a declared class, the
allocation is the quantity that carries Solomonoff's guarantee and that no
asymptotic criterion can see it. Whether a given read measures the model's
allocation rather than its interface is an empirical question, and the companion
finds it confounded, under the renderings it reads, with the mass the model puts
outside the alphabet.

\clearpage\appendix
\begingroup\sloppy
\makeatletter\def\@listI{\leftmargin\leftmargini\parsep\z@\topsep 3\p@ \@plus\p@\itemsep\z@ \@plus\p@}\let\@listi\@listI\renewcommand\section{\@startsection{section}{1}{\z@}{-2.2ex \@plus -.6ex \@minus -.2ex}{1.2ex \@plus .2ex}{\normalfont\Large\bfseries}}\renewcommand\subsection{\@startsection{subsection}{2}{\z@}{-1.8ex\@plus -.5ex \@minus -.2ex}{.8ex \@plus .2ex}{\normalfont\large\bfseries}}\g@addto@macro\normalsize{\setlength\abovedisplayskip{6pt \@plus 2pt \@minus 3pt}\setlength\belowdisplayskip{6pt \@plus 2pt \@minus 3pt}}\normalsize\renewcommand\paragraph{\@startsection{paragraph}{4}{\z@}{1.2ex \@plus .4ex \@minus .2ex}{-1em}{\normalfont\normalsize\bfseries}}\makeatother

\providecommand{\tlo}{t^{-}}
\providecommand{\thi}{t^{+}}

\noindent\textbf{Notation.}\par\smallskip
\noindent{\footnotesize\setlength{\tabcolsep}{4pt}%
\begin{tabular}{@{}l p{0.56\linewidth} l@{}}
$\Ac$, $V$ & alphabet; $V:=|\Ac|$ & \S\ref{sec:si-prelim} \\
$W$, $\phi$ & context window; prob.\ floor & \S\ref{sec:deployed} \\
$M$ & universal semimeasure & Def.~\ref{def:si-M} \\
$q_T$ & tempered conditionals & \S\ref{sec:deployed} \\
$K(\cdot)$ & prefix complexity & Def.~\ref{def:si-computable} \\
$Q_T$ & law served by sampling & \S\ref{sec:deployed} \\
$M_N$ & $M$, uniform before step $N$ & \eqref{eq:tail-MN} \\
$\Sigma_T(n)$ & $\sum_{t\le n}2^{-K(t)/T}$ & Cor.~\ref{cor:temperature} \\
$L_n(q;x)$ & log-loss $-\log q(x_{1:n})$ & Def.~\ref{def:si-predictor} \\
$Z_n(q)$, $C_n(\Fam)$ & Kraft sum; $C_n:=\sup_qZ_n(q)$ & \S\ref{app:frontier-overlap} \\
$R_n(q;\mu)$, $R_\infty$ & regret; $R_n(q;z):=L_n(q;z_{1:n})$ & Def.~\ref{def:si-predictor} \\
$C_\infty(\Fam)$ & best Kraft constant & Thm.~\ref{thm:frontier} \\
$\rho$, $\eps$ & context law; bound in~\eqref{eq:si-eps} & Ass.~\ref{ass:si-eps} \\
$\psi=\psi_n(q)$ & excess $\dl{P_n}{q_n}$ over the window & \S\ref{sec:icl} \\
$x^{q}$, $\mu_q$ & diagonal sequence; point mass & Thm.~\ref{thm:si-diag} \\
$\tlo$, $\thi$ & extreme $t$ with $\db\lambda t\le\psi$ & \S\ref{sec:icl}, \S\ref{app:training} \\
$\xi_w$ & Bayes mixture $\sum_kw_k\mu_k$ & \eqref{eq:bg-mixture} \\
$s^*(\eta,\psi)$ & bound in Prop.~\ref{prop:pretrain-certify}(ii) & \S\ref{sec:icl} \\
$\qbar$ & measure part of $q$ & Lem.~\ref{lem:measurepart} \\
$\mathrm{Mem}^W_n(\mu)$, $\mathcal{E}(\mu)$ & memory term; excess entropy & \S\ref{sec:deployed} \\
$\Kraft_J$ & Kraft set $\{r:\sum_k2^{-r_k}\le1\}$ & \S\ref{sec:frontier} \\
$\CVST(S,\tau)$ & $S$-bit, $\tau$-time budget class & Def.~\ref{def:rb-budget} \\
$\Lambda(F)$ & number of distinct lengths in $F$ & Fact~\ref{lem:si-tail} \\
$P$, $\pi$ & pretraining law $\sum_k\pi_k\mu_k$; task prior & \S\ref{sec:icl} \\
$G$ & $G(i,x,a)=q_i(a\mid x)$; matrix in Prop.~\ref{prop:implied} & Thm.~\ref{thm:si-classdiag} \\
$c(\eta)$ & $\eta-(1-2^{-\eta})\log e$ & Prop.~\ref{prop:pretrain-certify} \\
$\Delta_J$, $H(\lambda)$ & simplex on $J$; Shannon entropy & \S\ref{app:frontier-overlap} \\
$\Delta_k(q)$ & $R_n(q;\mu_k)-R_n(P;\mu_k)$ & \S\ref{app:training} \\
$\xi_{\mathrm{unif}}$ & uniform measure (conditionals $1/V$) & \S\ref{app:frontier-overlap} \\
$\omega_\ell$ & law of $UUU\cdots$, $U$ uniform on $\Ac^{\ell}$ & Prop.~\ref{prop:window-tax} \\
$\Gamma_{x}$ & cylinder of sequences extending $x$ & Cor.~\ref{cor:blind} \\
$\varrho(h;A)$ & coherence residual & Prop.~\ref{prop:coherence} \\
$D_\infty(p\|q)$ & $\log\max_{\mathrm{supp}\,p}(p/q)$ & Rem.~\ref{rem:posttrain} \\
\end{tabular}}
\par\smallskip\noindent{\scriptsize Some letters are reused, always with a
local definition: $\lambda$ is a prior mass (\S\ref{sec:icl},
App.~\ref{app:training}), a mixing vector (App.~\ref{app:frontier-overlap})
or a uniform random position (Def.~\ref{def:deployed-window}); $\gamma$ is
a conditional floor, $\gamma(\cdot)$ the Elias code (App.~\ref{sec:rb});
$\sigma$ is a symbol, a reference measure (App.~\ref{app:frontier-overlap})
or the overhead of Ass.~\ref{ass:rb-universal}; $h$ is a history or a
hypothesis, $h_V:=\lfloor\log V\rfloor$; $U$ is a machine
(Defs.~\ref{def:si-M}, \ref{def:rb-machine}) or a uniform block in
$\omega_\ell$; $W$ is a window, $W^{Q}$ a weight polytope; $r$ is a
reward, a regret profile or a rate; $c$ is a context or a constant.\par}

\section{Consequences of the Kraft inequality}
\label{app:proofs-background}\label{app:bg-conservation}

\begin{corollary}[Counting under the uniform measure]
\label{cor:typical}
For every predictor $q$, $n$ and $b\ge 0$,
$\bigl|\{x\in\Ac^{n}:L_n(q;x)\le n\log V-b\}\bigr|\,V^{-n}\le 2^{-b}$.
\end{corollary}

\begin{proof}[Proof of Corollary~\ref{cor:typical}]\label{app:proof-cor:typical}
Apply Fact~\ref{thm:conservation} to $\Ac^{n}$ with threshold $n\log V-b$.
\end{proof}

\begin{corollary}[Mean-loss floor]
\label{cor:meanfloor}
For every predictor $q$ and finite antichain $S\subseteq\Ac^{*}$,
$\frac{1}{|S|}\sum_{x\in S}L_{|x|}(q;x)\ge\log|S|$,
with equality iff $q(x)=1/|S|$ for every $x\in S$.
\end{corollary}

\begin{proof}[Proof of Corollary~\ref{cor:meanfloor}]\label{app:proof-cor:meanfloor}
By Jensen for the convex $t\mapsto 2^{-t}$ and Fact~\ref{lem:antichain},
$2^{-\frac{1}{|S|}\sum_xL_{|x|}(q;x)}\le\frac1{|S|}\sum_xq(x)\le\frac1{|S|}$;
Jensen is tight iff the losses are equal, Kraft iff $S$ carries all the
mass.
\end{proof}

\begin{corollary}[Exchange rate]
\label{cor:exchange}
Let $q,q'$ be predictors, $b\in\mathbb{R}$, $S$ an antichain and
$A\subseteq S$ with $L_n(q';x)\le L_n(q;x)-b$ for all $x\in A$. Then
$q'(S\setminus A)\le 1-2^{b}q(A)$, and in particular $q(A)\le 2^{-b}$.
If moreover $S$ is a complete antichain and $q$ is a measure, the strings in
$S\setminus A$ lose at least $(2^{b}-1)\,q(A)$ of their $q$-mass, which is
informative for $b>0$ and vacuous otherwise. The case $S=\Ac^{n}$ is the one
used below.
\end{corollary}

\begin{proof}[Proof of Corollary~\ref{cor:exchange}]\label{app:proof-cor:exchange}
$q'(A)=\sum_{x\in A}2^{-L_n(q';x)}\ge 2^{b}q(A)$, and
$q'(A)+q'(S\setminus A)\le 1$ by Fact~\ref{lem:antichain}, which needs only
that $S$ is an antichain. Neither step uses the sign of $b$. The last claim
needs $q(S)=1$, hence a complete antichain.
\end{proof}

\begin{corollary}[Finite-horizon Kraft law]
\label{thm:profile}
Let $q$ be a predictor and $z^{(1)},\dots,z^{(m)}\in\Ac^{\infty}$ have
pairwise distinct prefixes $z^{(k)}_{1:n}$. Then
$\sum_{k=1}^{m}2^{-R_n(q;z^{(k)})}\le 1$. Hence at most $2^{b}$ of them
have $R_n(q;z^{(k)})\le b$, at least half cost more than $\log(m/2)$
bits, and the mean is at least $\log m$.
\end{corollary}

\begin{proof}[Proof of Corollary~\ref{thm:profile}]\label{app:proof-thm:profile}
The prefixes form an antichain, and $2^{-R_n(q;z^{(k)})}=q(z^{(k)}_{1:n})$.
Fact~\ref{thm:conservation} gives the law and the count ($b=\log(m/2)$
for the half), Corollary~\ref{cor:meanfloor} the mean.
\end{proof}

\begin{corollary}[Kraft floor on Bayesian loss ceilings]
\label{cor:kraftfloor}
A Bayes mixture $Q$ over deterministic hypotheses $h$ with weights $(w_h)$
has $Q\ge w_h\,\delta_h$, so its loss on the sequence of $h$ never exceeds
the \emph{loss ceiling} $c(h):=\log(1/w_h)$. Let
$h^{\dagger}_1,\dots,h^{\dagger}_m$ be the distinct true hypotheses of $m$
tasks and $c_k:=c(h^{\dagger}_k)$. Then
$\sum_k2^{-c_k}\le 1$, with the counts of Corollary~\ref{thm:profile}.
\end{corollary}

\begin{proof}[Proof of Corollary~\ref{cor:kraftfloor}]\label{app:proof-cor:kraftfloor}
$\sum_k2^{-c_k}=\sum_kw_{h^{\dagger}_k}\le\sum_hw_h\le 1$; the counts
follow as in Corollary~\ref{thm:profile}.
\end{proof}

\section{The frontier of regret profiles}
\label{app:frontier}

\subsection{The singular case}
\label{app:frontier-proofs}

\paragraph{Conventions.} Notation is that of Section~\ref{sec:frontier}
and~\eqref{eq:bg-mixture}. For a probability measure $\mu$ and a measure
$\nu$ of mass at most one, $\dl{\mu}{\nu}:=\int\log\frac{d\mu}{d\nu}\,d\mu$
if $\mu\ll\nu$, and $+\infty$ otherwise; equivalently
$\dl{\mu}{\nu}=\dl{\mu}{\nu/\nu(\Ac^\infty)}+\log\frac1{\nu(\Ac^\infty)}\ge 0$.
The log-sum inequality, joint convexity and
$\dl{\mu}{\nu}=\lim_n\dl{\mu_n}{\nu_n}$~\cite[Ch.~7]{gray2011entropy}
still hold, a subscript $n$ denoting the marginal on $\Ac^{n}$.

\begin{lemma}[Measure part of a predictor; after~\cite{levinvyugin1977,bienvenu2017semimeasures}]
\label{lem:measurepart}
For a predictor $q$, $\qbar(x):=\lim_{m'\to\infty}\sum_{y\in\Ac^{m'}}q(xy)$
defines the largest measure on $\Ac^\infty$ dominated by $q$ on cylinders.
It has mass at most $1$, zero leakage $1-\qbar(\Ac^\infty)$ if $q$ is a
measure, and $R_\infty(q;\mu)=\dl{\mu}{\qbar}$ for every probability
measure $\mu$.
\end{lemma}

\begin{proof}[Proof of Lemma~\ref{lem:measurepart}]\label{app:proof-lem:measurepart}
The sum is nonincreasing in $m'$, so the limit $\qbar(x)$ exists and is at most $q(x)$, and
exchanging limit and finite sum gives $\qbar(x)=\sum_a\qbar(xa)$. Cylinders are clopen in the compact
$\Ac^\infty$, so $\qbar$ is countably additive on the cylinder algebra and
extends by Carath\'eodory. If $\nu\le q$ on cylinders,
$\nu(x)=\sum_{y\in\Ac^{m'}}\nu(xy)\le\sum_{y\in\Ac^{m'}}q(xy)\to\qbar(x)$.
As $q_n\ge\qbar_n$,
$R_n(q;\mu)\le\dl{\mu_n}{\qbar_n}\le\dl{\mu}{\qbar}$. Conversely, by the
log-sum inequality $R_{n+m'}(q;\mu)\ge\dl{\mu_n}{q^{(m')}_n}$, where
$q^{(m')}_n(x):=\sum_{y\in\Ac^{m'}}q(xy)\downarrow\qbar_n(x)$ on a finite
set, so the bound increases to $\dl{\mu_n}{\qbar_n}$; take the supremum
over $n$.
\end{proof}

\paragraph{Disjoint carriers.} For pairwise singular $\Fam$ and $j\ne k$,
pick $B_{kj}$ with $\mu_k(B_{kj})=1=1-\mu_j(B_{kj})$; put
$E_k:=\bigcap_{j\ne k}B_{kj}$ and $\Omega_k:=E_k\setminus\bigcup_{j\ne k}E_j$.
By countability $\mu_k(E_k)=1$ and $\mu_k(E_j)=0$ for $j\ne k$, so
$\mu_k(\Omega_k)=1$ and the $\Omega_k$ are disjoint.

\begin{proposition}[Ledger and converse]
\label{prop:frontier-ledger}
Let $q$ be a predictor, $\qbar_k$ the $\mu_k$-absolutely continuous part
of $\qbar$, $c_k:=\qbar_k(\Ac^\infty)$ and
$D_k:=\dl{\mu_k}{\qbar_k/c_k}$, with $c_k2^{-D_k}:=0$ if $c_k=0$.
\begin{enumerate}
\item[\textup{(0)}] \textup{(Per-component identity, any family.)}
$2^{-R_\infty(q;\mu_k)}=c_k2^{-D_k}$ for every $k$. This is the Lebesgue
decomposition of $\qbar$ with respect to $\mu_k$ alone and uses nothing about
the other components.
\item[\textup{(i)}] \textup{(Ledger.)} If $\Fam$ is pairwise mutually
singular, so that the carriers $\Omega_k$ may be taken disjoint and the second
term below is therefore nonnegative,
\[
1-\sum_k2^{-R_\infty(q;\mu_k)}
=\underbrace{\bigl(1-\qbar(\Ac^\infty)\bigr)}_{\text{leakage}}
+\underbrace{\Bigl(\qbar(\Ac^\infty)-\sum_kc_k\Bigr)}_{\text{mass off the family}}
+\underbrace{\sum_kc_k\bigl(1-2^{-D_k}\bigr)}_{\text{misfit within the family}} .
\]
\item[\textup{(ii)}] \textup{(Kraft characterises singularity.)} If
$\mu_i\not\perp\mu_j$ for some $i\ne j$, the predictor
$\bar\mu:=\tfrac12(\mu_i+\mu_j)$ has
$2^{-R_\infty(\bar\mu;\mu_i)}+2^{-R_\infty(\bar\mu;\mu_j)}>1$.
\end{enumerate}
\end{proposition}

\begin{proof}[Proof of Theorem~\ref{thm:frontier} and Proposition~\ref{prop:frontier-ledger}]\label{app:proof-thm:frontier}
\emph{Inequality in (i).} For $\nu$ of mass at most one and $\mu(\Omega)=1$,
$\dl{\mu}{\nu}=\log\frac1{\nu(\Omega)}+\dl{\mu}{\nu|_\Omega/\nu(\Omega)}
\ge\log\frac1{\nu(\Omega)}$, with equality iff
$\nu|_\Omega=\nu(\Omega)\mu$. By Lemma~\ref{lem:measurepart},
$2^{-R_\infty(q;\mu_k)}\le\qbar(\Omega_k)$, and
$\sum_k\qbar(\Omega_k)\le 1$.

\emph{Proposition~\ref{prop:frontier-ledger}(i).} By the Lebesgue
decomposition, $\qbar_k=\qbar|_{G_k}$ with $\mu_k(G_k^c)=0$. As
$\qbar|_{G_k\setminus\Omega_k}\ll\mu_k$ lives on the $\mu_k$-null
$\Omega_k^c$, it vanishes, so we may take $G_k\subseteq\Omega_k$. The $G_k$ are then disjoint, and
$\qbar(\Ac^\infty)-\sum_kc_k=\qbar((\cup_kG_k)^c)\ge 0$. If
$\mu_k\ll\qbar$, $\mu_k$ is carried by $G_k$, where $\qbar=\qbar_k$,
so $\dl{\mu_k}{\qbar}=\dl{\mu_k}{\qbar_k}=\log(1/c_k)+D_k$; otherwise both
sides are infinite, and $c_k=0$ forces this case. By
Lemma~\ref{lem:measurepart}, $2^{-R_\infty(q;\mu_k)}=c_k2^{-D_k}$. Sum over
$k$ and write
$\sum_kc_k=1-(1-\qbar(\Ac^\infty))-(\qbar(\Ac^\infty)-\sum_kc_k)$.

\emph{(ii).} The measure part $\sum_jw_j\mu_j$ of $\xi_w$ dominates
$w_k\mu_k$, so $R_\infty(\xi_w;\mu_k)\le\log(1/w_k)$, and gives $\Omega_k$
mass exactly $w_k$, so the inequality above gives the reverse.

\emph{(iii).} Kraft sum $1$ forces equality throughout:
$\qbar(\Omega_k)=2^{-R_\infty(q;\mu_k)}=:w_k$, $\sum_kw_k=1$ and
$\qbar|_{\Omega_k}=w_k\mu_k$, so $\qbar=\sum_kw_k\mu_k$ is a probability
measure. As $q_n\ge\qbar_n$ and $q_n(\Ac^n)\le 1=\qbar_n(\Ac^n)$, $q$ and
$\qbar$ agree on cylinders. The converse is (ii). Below Kraft sum $1$ one
coordinate can be lowered inside $\Kraft_J$; if $\sum_k2^{-r_k}=1$ and
$\Kraft_J\ni s\le r$, then $s=r$.

\emph{Proposition~\ref{prop:frontier-ledger}(ii).} As $\bar\mu$ is its own
measure part, Lemma~\ref{lem:measurepart} gives, with
$f:=d\mu_i/d\bar\mu\in[0,2]$, $R_\infty(\bar\mu;\mu_i)=\E_{\mu_i}\log f\le 1$,
with equality iff $f=2$ $\mu_i$-a.s.; then $\mu_i$ is carried by
$\{f=2\}$, which is $\mu_j$-null as $d\mu_j/d\bar\mu=2-f$, contradicting
$\mu_i\not\perp\mu_j$. So $R_\infty(\bar\mu;\mu_i)<1$, and likewise for
$\mu_j$.

\emph{Final claim.} On a pairwise singular family (i) gives
$C_\infty\le 1$, attained by $q=\mu_k$; otherwise
Proposition~\ref{prop:frontier-ledger}(ii) gives two terms summing to more
than $1$ at $q=\bar\mu$.
\end{proof}

\begin{corollary}[The exact price of specialisation]
\label{cor:specialise}
Let $\mu$ be a probability measure, $q$ a predictor and
$r:=R_\infty(q;\mu)$.
\begin{enumerate}
\item[\textup{(i)}] If $\mu$ and $\nu_i$, $i\in I$ with $I$ countable,
are pairwise mutually singular, then
$\sum_{i\in I}2^{-R_\infty(q;\nu_i)}\le 1-2^{-r}$. For a
single $\nu\perp\mu$,
$R_\infty(q;\nu)\ge\log\frac{1}{1-2^{-r}}
=\log\frac1r+\log\frac1{\ln2}+O(r)$ as $r\to 0$, with equality iff
$q=2^{-r}\mu+(1-2^{-r})\nu$ on every cylinder.
\item[\textup{(ii)}] If $r=0$, then $q=\mu$ on every cylinder and
$R_\infty(q;\nu)=\dl{\nu}{\mu}$ for every probability measure $\nu$, with no
support assumption.
\item[\textup{(iii)}] \textup{(Mixing versus interpolating.)} For
$\nu\perp\mu$ and $\delta\in(0,1)$, the mixture of joint laws
$(1-\delta)q+\delta\nu$ has regret at most $\log(1/\delta)$ on $\nu$, and
lies on the curve of \textup{(i)} iff $\qbar$ is a full-mass mixture of
$\mu$ and $\nu$. Interpolating conditionals at a fixed weight,
$(1-\delta)\mu(\cdot\mid x)+\delta\nu(\cdot\mid x)$, with $\nu$
deterministic and $\mu$'s conditionals at least $\gamma>0$, gives
$R_n(\cdot\,;\nu)\ge n\log\frac{1}{1-(1-\delta)\gamma}\to\infty$.
\end{enumerate}
\end{corollary}

\begin{proof}[Proof of Corollary~\ref{cor:specialise}]\label{app:proof-cor:specialise}
(i) Apply Theorem~\ref{thm:frontier}(i)--(iii) to the family
$(\mu,(\nu_i)_{i\in I})$ and expand $1-2^{-r}=r\ln2-\tfrac12(r\ln2)^2+O(r^3)$.
(ii) Theorem~\ref{thm:frontier}(iii) for $\{\mu\}$ gives $q=\mu$ on
cylinders, so $\qbar=\mu$; apply Lemma~\ref{lem:measurepart}.
(iii) The mixture dominates $\delta\nu$ and has measure part
$(1-\delta)\qbar+\delta\nu$; apply Theorem~\ref{thm:frontier}(iii) to
$\{\mu,\nu\}$. (In Theorem~\ref{thm:nocollapse}, $\qbar$ gives the
diagonal atom no mass, so the condition is $q=\mu$.) Under
interpolation, $\mu$'s conditional on $\nu$'s next symbol is at most
$1-\gamma$, so each step along $\nu$ has conditional probability at most
$(1-\delta)(1-\gamma)+\delta=1-(1-\delta)\gamma$.
\end{proof}

\subsection{Overlapping families}
\label{app:frontier-overlap}

\paragraph{Setting.}
Let $J$ be finite, $n\in\Nb\cup\{\infty\}$, $\lambda\in\Delta_J$ (the simplex of
probability vectors on $J$),
$\kappa\sim\lambda$, $X\mid\{\kappa=k\}\sim\mu_k$, and
$\xi_\lambda:=\sum_k\lambda_k\mu_k$. Put
$I_\lambda(n):=I(\kappa;X_{1:n})=\sum_k\lambda_k\dl{\mu_{k,n}}{\xi_{\lambda,n}}$
and, with $H$ the Shannon entropy, $H_\lambda(\kappa\mid X_{1:n}):=H(\lambda)-I_\lambda(n)$ (monotone
limits at $n=\infty$). Put $Z_n(q):=\sum_k2^{-R_n(q;\mu_k)}$ and
$C_n(\Fam):=\sup_qZ_n(q)$; if $Z_n(q)>0$, the \emph{Gibbs weights} are $\lambda(q)_k:=2^{-R_n(q;\mu_k)}/Z_n(q)$. ``$q=\xi_\lambda$'' means
equal length-$n$ marginals at finite $n$ and equality on every cylinder at $n=\infty$, where
$\dl{\xi_\lambda}{q}:=\dl{\xi_\lambda}{\qbar}$.

\begin{proposition}[The frontier for overlapping families; after~\cite{wald1950,topsoe1979,gallager1979,davisson1980}]
\label{prop:frontier-overlap}
\leavevmode
\begin{enumerate}
\item[\textup{(i)}] \textup{(Region.)} A finite $r\in[0,\infty)^J$ is
achievable ($R_n(q;\mu_k)\le r_k$ for all $k$, for some $q$) iff
$\lambda\cdot r\ge I_\lambda(n)$ for every $\lambda\in\Delta_J$. Every
Pareto-optimal finite profile is attained by some $\xi_\lambda$ and by no
other predictor, every full-support mixture is Pareto-optimal, and
$\min_q\max_kR_n(q;\mu_k)=\max_\lambda I_\lambda(n)$.
\item[\textup{(ii)}] \textup{(Log-ledger.)} If $Z_n(q)>0$, then
$\log Z_n(q)=H_{\lambda(q)}(\kappa\mid X_{1:n})-\dl{\xi_{\lambda(q),n}}{q_n}$;
on a pairwise singular family at $n=\infty$,
$\log Z_\infty(q)=-\dl{\xi_{\lambda(q)}}{\qbar}$.
\item[\textup{(iii)}] \textup{(Kraft constant.)} $\lambda\mapsto
H_\lambda(\kappa\mid X_{1:n})$ is concave, and
$\max_\lambda H_\lambda(\kappa\mid X_{1:n})=\log C_n(\Fam)$. For every $q$,
$\log Z_n(q)\le H_{\lambda(q)}(\kappa\mid X_{1:n})\le\log C_n$, with
equality throughout iff $q=\xi_{\lambda^*}$ for a maximiser $\lambda^*$.
Every maximiser has $\lambda(\xi_{\lambda^*})=\lambda^*$, i.e.\
$R_n(\xi_{\lambda^*};\mu_k)=\log(1/\lambda^*_k)-\log C_n$; the maximisers
form a convex set on which $\lambda\mapsto\xi_\lambda$ is injective; and
$C_\infty=1$ iff the family is pairwise singular.
\end{enumerate}
(Attributions: Appendix~\ref{app:frontier-prior}.)
\end{proposition}

\begin{proof}[Proof of Proposition~\ref{prop:frontier-overlap}]\label{app:proof-prop:frontier-overlap}
\emph{Gibbs.} Let $R_k:=R_n(q;\mu_k)$ and
$J':=\{k:R_k<\infty\}\neq\emptyset$. For $\lambda$ supported in $J'$,
$H(\lambda)-\sum_k\lambda_kR_k=\log Z_n(q)-\dl{\lambda}{\lambda(q)}$;
otherwise the left side is $-\infty$. So
$\log Z_n(q)=\max_\lambda[H(\lambda)-\sum_k\lambda_kR_k]$, attained only at
$\lambda(q)$.

\emph{Compensation.} For finite $n$, integrating
$\log\frac{\mu_k}{q}=\log\frac{\mu_k}{\xi_\lambda}+\log\frac{\xi_\lambda}{q}$
on $\Ac^n$ against $\mu_k$ and averaging over $\lambda$ gives
$\sum_k\lambda_kR_n(q;\mu_k)=I_\lambda(n)+\dl{\xi_{\lambda,n}}{q_n}$. As
$n\to\infty$ every term is nondecreasing, and
Lemma~\ref{lem:measurepart} identifies the last.

\emph{(ii).} Take $\lambda=\lambda(q)$ in Gibbs and compensation. On a
pairwise singular family $H_\lambda(\kappa\mid X_{1:\infty})=0$ for every
$\lambda$, by the last clause of (iii), proved below; its proof uses only the
general part of (ii), so the order is not circular.

\emph{Continuity and concavity.} Let $\sigma:=|J|^{-1}\sum_j\mu_j$,
$f_k:=d\mu_k/d\sigma\in[0,|J|]$, and
$\Phi(v):=\sum_kv_k\log(\sum_jv_j/v_k)$ on $[0,\infty)^J$, continuous,
$1$-homogeneous and concave (the perspective of Shannon entropy), with
$0\le\Phi(v)\le(\sum_kv_k)\log|J|$. By homogeneity
$H_\lambda(\kappa\mid X_{1:n})=\sum_{x\in\Ac^n}\Phi((\lambda_k\mu_k(x))_k)=
\int\Phi((\lambda_kf^{(n)}_k)_k)\,d\sigma$, where
$f^{(n)}_k:=\E_\sigma[f_k\mid X_{1:n}]\to f_k$ $\sigma$-a.s.~\cite{doob1953};
dominated convergence (bound $|J|\log|J|$) gives
$H_\lambda(\kappa\mid X_{1:\infty})=\int\Phi((\lambda_kf_k)_k)\,d\sigma$,
continuous in $\lambda$. Concavity in $\lambda$ is that of $\Phi$, so the
maximum in (iii) exists.

\emph{(iii).} Let $H^*:=\max_\lambda H_\lambda(\kappa\mid X_{1:n})$. If
$Z_n(q)>0$, (ii) gives $\log Z_n(q)\le H_{\lambda(q)}\le H^*$, with
equality throughout iff $q=\xi_{\lambda(q)}$ and $\lambda(q)$ is a
maximiser. Conversely, for a maximiser $\lambda^*$ and
$q=\xi_{\lambda^*}$, $R_n(q;\mu_k)\le\log(1/\lambda^*_k)$ on the support,
so Gibbs and compensation give
$\log Z_n(q)\ge H(\lambda^*)-\sum_k\lambda^*_kR_n(q;\mu_k)=H_{\lambda^*}
=H^*$. Hence $\log C_n=H^*$, equality holds at $\xi_{\lambda^*}$, and Gibbs uniqueness gives
$\lambda(\xi_{\lambda^*})=\lambda^*$. By concavity the maximisers form a
convex set, and $\xi_{\lambda^*}=\xi_{\lambda^{**}}$ forces
$\lambda^*=\lambda(\xi_{\lambda^*})=\lambda^{**}$. The last clause is
the final claim of Theorem~\ref{thm:frontier}.

\emph{(i).} ``Only if'' is compensation. The achievable set
$\mathcal{R}\subseteq[0,\infty)^J$ is convex (mixtures of joint laws are
predictors, and each $R_{n'}$ is convex in the joint law) and closed
(predictors form a compact countable product of simplices; if witnesses
$q^{(i)}$ of $r^{(i)}\to r$ converge coordinatewise to $q$, lower
semicontinuity of finite-horizon relative entropy gives $R_{n'}(q;\mu_k)\le r_k$ for $n'\le n$).
\emph{``If'':} for $r\notin\mathcal{R}$, strict separation gives
$\lambda\in\Delta_J$ ($\lambda\ge0$ as $\mathcal{R}$ is upward closed)
with $\lambda\cdot r<\inf_{\mathcal{R}}\lambda\cdot s$, and the finite
profiles of $q_\vartheta:=(1-\vartheta)\xi_\lambda+\vartheta\,\xi_{\lambda_0}$, with
$\lambda_0$ any fully supported prior on $J$, satisfy
$\lambda\cdot R_n(q_\vartheta;\Fam)\le I_\lambda(n)+\log\frac1{1-\vartheta}$,
the profiles being finite in every coordinate because
$q_\vartheta\ge\vartheta\lambda_{0,k}\mu_k$ gives
$R_n(q_\vartheta;\mu_k)\le\log\frac1{\vartheta\lambda_{0,k}}$ for each $k$, which the
uniform measure would not give at $n=\infty$ against a component singular to
it, so
$\lambda\cdot r<I_\lambda(n)$.
\emph{Pareto-optimal $s^*$:} if all $s^*_k>0$, separating
$\mathcal{R}$ from the nonempty open convex $\{s:s<s^*\}$ gives
$\lambda\in\Delta_J$ with
$\lambda\cdot s^*=\min_{\mathcal{R}}\lambda\cdot s=I_\lambda(n)$; a witness
$q$ then has $\lambda\cdot R_n(q;\Fam)\le I_\lambda(n)$, so $q=\xi_\lambda$
by compensation, and $R_n(\xi_\lambda;\Fam)\le s^*$ forces equality. If
$s^*_k=0$, a witness has $R_n(q;\mu_k)=0$, so $q=\mu_k=\xi_{e_k}$, and
Pareto-optimality forces $R_n(\xi_{e_k};\Fam)=s^*$.
\emph{Full-support mixtures:} if $\lambda>0$ and
$\mathcal{R}\ni s\le R_n(\xi_\lambda;\Fam)$, then
$I_\lambda\le\lambda\cdot s\le\lambda\cdot R_n(\xi_\lambda;\Fam)=I_\lambda$,
so $s=R_n(\xi_\lambda;\Fam)$. \emph{Diagonal:} $t\mathbf 1\in\mathcal{R}$
iff $t\ge\max_\lambda I_\lambda(n)$; the maximum exists by continuity, the
minimum over $q$ by closedness.
\end{proof}

\providecommand{\Ev}{\mathrm{Ev}}
\providecommand{\argmin}{\operatorname*{arg\,min}}

\section{Evaluation, scale and asymptotic criteria}
\label{app:limits}

Notation follows Section~\ref{sec:si-prelim}; the \emph{program bound} is
the last clause of Fact~\ref{fact:si-program}.

\subsection{Evaluation}
\label{app:limits-eval}

\begin{proof}[Proof of Theorem~\ref{thm:si-underdet}]\label{app:proof-thm:si-underdet}
(i) Since $D \ge 0$,
$\eps \ge \E_\rho[D\mathbf 1_B] = \rho(B)\,\E_\rho[D\mid B]$; the second
claim is Markov's inequality. (ii) For $c \in B$,
$\dl{\mu(\cdot\mid c)}{r(\cdot\mid c)} \le \sum_a \mu(a\mid c)\log\frac1{r(a\mid c)} \le \log\frac1\gamma$,
so $\E_\rho\dl{\mu}{q'} \le \E_\rho[D\mathbf 1_{B^c}] + \delta\log(1/\gamma)$.
(iii) Couple the runs through shared contexts, continuations and internal
randomness; the verdicts coincide unless some $c_k \in B$, of
probability $1-(1-\delta)^N \le N\delta$.
\end{proof}

\begin{corollary}[No finite sampled evaluation certifies agreement with a reference]
\label{cor:si-uncert}
Let $q$ satisfy Assumption~\ref{ass:si-eps} and pass an $N$-query
$\rho$-evaluation with probability at least $1-\alpha$. For every reference
predictor $P^{*}$ and $B$ with $\rho(B) = \delta > 0$ some $q'$ (a) satisfies
Assumption~\ref{ass:si-eps} with $\eps + \delta\log(2V)$, (b) passes with probability at least
$1-\alpha-N\delta$, and (c) has
$\TV(q'(\cdot\mid c),P^{*}(\cdot\mid c)) \ge \tfrac12 - \tfrac1{2V}$ on $B$.
\end{corollary}

\begin{proof}[Proof of Corollary~\ref{cor:si-uncert}]\label{app:proof-cor:si-uncert}
For $c \in B$ pick $a^{\circ}(c) \in \argmin_a P^{*}(a\mid c)$ and set
$r(a\mid c) := \tfrac1{2V} + \tfrac12\mathbf 1\{a = a^{\circ}(c)\}$. Then
$\min_a r \ge 1/(2V)$, Theorem~\ref{thm:si-underdet}(ii)--(iii) give (a)
and (b), and
$\TV(r,P^{*}) \ge r(a^{\circ}\mid c) - P^{*}(a^{\circ}\mid c) \ge \tfrac12 + \tfrac1{2V} - \tfrac1V$.
\end{proof}

\begin{theorem}[Adaptive underdetermination]
\label{thm:si-adaptive}
An \emph{$N$-query adaptive protocol} $\mathcal{E}$ chooses each context
$c_k$ from the transcript $(c_j, q(\cdot\mid c_j), y_j)_{j<k}$ and internal
randomness $\omega$, observes $q(\cdot\mid c_k)$ and data
$y_k = f(c_k,\zeta_k)$ with $\zeta_k$ independent of $q$, and emits a
verdict $Z$. Let $\nu_{\mathcal{E},q}(c)$ be the probability that
$\mathcal{E}$ queries $c$ against $q$ (so
$\sum_c \nu_{\mathcal{E},q}(c) \le N$), take $\mathcal{R}_\delta$ and
$\bar\rho_\delta \ge \delta$ as in Section~\ref{sec:eval-adaptive}, and let
$\Ev_{\mathcal{E},q}(\delta) := \inf\{\nu_{\mathcal{E},q}(B) :
B \subseteq \mathcal{R}_\delta,\ \rho(B) \ge \delta\}$ (the
\emph{evasion constant}). For any reference
$P^{*}$, $\gamma \in (0,1/V]$ and $\beta > 0$ there are
$B \subseteq \mathcal{R}_\delta$ with $\delta \le \rho(B) < 2\delta$ and
$q'$ equal to $q$ off $B$ such that:
\begin{enumerate}
\item[\textup{(i)}] $\TV(\mathrm{law}(Z\mid q),\mathrm{law}(Z\mid q'))
\le \nu_{\mathcal{E},q}(B) \le \Ev_{\mathcal{E},q}(\delta) + \beta$;
\item[\textup{(ii)}] $\TV(q'(\cdot\mid c), P^{*}(\cdot\mid c)) \ge
1 - \tfrac1V - (V-1)\gamma$ on $B$, and $q'$ satisfies
Assumption~\ref{ass:si-eps} with $\eps + 2\delta\log(1/\gamma)$;
\item[\textup{(iii)}] $\Ev_{\mathcal{E},q}(\delta) = 0$ if some subset of
$\mathcal{R}_\delta$ of mass $\delta$ is never queried, and always
$\Ev_{\mathcal{E},q}(\delta) \le 2N\delta/\bar\rho_\delta$;
\item[\textup{(iv)}] if $\mathcal{E}$ separates $q$ from every such $q'$
with error at most $\alpha$ on each side, then
$\Ev_{\mathcal{E},q}(\delta) \ge 1-2\alpha$, hence
$N \ge (1-2\alpha)\,\bar\rho_\delta/(2\delta)$;
\item[\textup{(v)}] \textup{(Tightness.)} the $N$-query $\rho$-evaluation
rejecting when a sampled conditional differs from $P^{*}$'s separates
$P^{*}$ from every predictor differing from it on mass at least $\delta$,
with error at most $\alpha$, once $N \ge \ln(1/\alpha)/\delta$.
\end{enumerate}
\end{theorem}

\begin{proof}[Proof of Theorem~\ref{thm:si-adaptive}]\label{app:proof-thm:si-adaptive}
Choose $B_0 \subseteq \mathcal{R}_\delta$ with $\rho(B_0) \ge \delta$ and
$\nu_{\mathcal{E},q}(B_0) \le \Ev_{\mathcal{E},q}(\delta) + \beta$, and let
$B$ be the shortest initial segment of an enumeration of its atoms with
mass at least $\delta$ ($B := B_0$ if none is finite); as atoms have mass
at most $\delta$, $\delta \le \rho(B) < 2\delta$. On $B$ put
$q'(a^{\circ}(c)\mid c) := 1-(V-1)\gamma$ and $q'(a\mid c) := \gamma$
otherwise, with $a^{\circ}(c) \in \argmin_a P^{*}(a\mid c)$.
(i) Share $\omega$ and $(\zeta_k)$; by induction on $k$ the transcripts
agree until a query lands in $B$, of probability at most
$\nu_{\mathcal{E},q}(B)$. (ii) $\TV \ge q'(a^{\circ}\mid c) -
P^{*}(a^{\circ}\mid c)$; the cost is Theorem~\ref{thm:si-underdet}(ii).
(iii) With $u(c) := \nu_{\mathcal{E},q}(c)/\rho(c)$ on $\mathcal{R}_\delta$
and $s^{*}$ the least $s$ with
$\rho(\mathcal{R}_\delta\cap\{u\le s\}) \ge \delta$ (right-continuity),
let $B$ be $\mathcal{R}_\delta\cap\{u<s^{*}\}$ plus atoms of $\{u=s^{*}\}$
up to mass $\delta$. Its ratios are at most those off it, so its
$\rho$-average of $u$ is at most $N/\bar\rho_\delta$ and
$\nu_{\mathcal{E},q}(B) < 2N\delta/\bar\rho_\delta$.
(iv) Two-sided error $\alpha$ needs verdict distance $1-2\alpha$; apply (i)
for all $\beta$, then (iii). (v) Under $P^{*}$ nothing is rejected; under
the alternative all samples miss the disagreement set with probability at
most $(1-\delta)^N \le e^{-N\delta} \le \alpha$.
\end{proof}

\subsection{Benchmarks, red teaming, scaling curves and contamination}
\label{sec:eval-practice}

\paragraph{Accuracy and fidelity are distinct coordinates.}\label{rem:si-dissociation}
An accuracy score is a coarse function of a model's conditionals, and
the two can come apart. Take $P^{*}$ to be the conditionals of an
exhaustively enumerated reference mixture and let the evaluation be an
answer-accuracy benchmark. Corollary~\ref{cor:si-uncert} then says that
passing at any accuracy level is compatible with near-maximal
disagreement with that reference on every context of a set of mass
$\delta$, at resolution $N\delta$. A dissociation of this kind, high
answer accuracy alongside distributional fidelity below the reference,
is observed in companion work~\cite{ficl2026}. Where a served model
exposes sampled text but not token log-probabilities, fidelity to a
reference can only be estimated from samples, at the costs analysed in
Paper~B.

Emergent abilities, which appear as sudden jumps in task metrics with
scale~\cite{wei2022emergent}, are a second instance.
Schaeffer et~al.~\cite{schaeffer2023mirage} show that many such jumps are artefacts
of discontinuous metrics applied to smoothly improving per-token
likelihoods. In the terms above, the jump is a property of the verdict
map, not of the conditionals, and conversely a smooth metric can hide
large changes in the conditionals off the evaluated set.

\paragraph{Red teaming falsifies but does not certify.} Red teaming,
whether by people~\cite{ganguli2022redteaming}, by other language
models~\cite{perez2022redteaming}, or by gradient-guided search for
adversarial suffixes~\cite{zou2023universal}, is an adaptive protocol in
the sense of Theorem~\ref{thm:si-adaptive}. A failure it finds
falsifies agreement. A campaign that finds none bounds disagreement only on sets
carrying a large share of its query mass: evidence about the searched
region, not a certificate for the contexts deployment will produce.

\paragraph{Scaling laws measure $\eps$ on $\rho$.} Under the log-loss,
the population loss on the corpus prefix law is the source's conditional
entropy plus the averaged divergence of Assumption~\ref{ass:si-eps}.
Scaling laws fit held-out loss as power laws in parameters, data and
compute~\cite{kaplan2020scaling}. The parametric fit of
Hoffmann et~al.~\cite{hoffmann2022chinchilla} separates an irreducible term from
reducible terms that fall with parameters and training tokens. The
fitted irreducible term is a power-law asymptote, not a measured source
entropy, and its fitted value is sensitive to the fitting
procedure~\cite{besiroglu2024chinchilla}; data-constrained laws keep it
fixed and let repetition discount the effective data
terms~\cite{muennighoff2023data}. If it is read as the source's entropy, the
reducible part is the per-token $\eps$, with $\rho$ the corpus prefix
law, and a better scaling curve shrinks $\eps$. By
Theorem~\ref{thm:si-underdet}(i), the resulting control over a set of
contexts degrades as $\eps/\rho(B)$, and it says nothing where $\rho$
puts no mass. No value of $\eps$ closes this gap
(Section~\ref{sec:scale}), and the asymptotic conclusions a vanishing
$\eps$ supports cannot see allocation (Section~\ref{sec:si-tail}).

\paragraph{Contamination.} Theorem~\ref{thm:si-underdet}(iii) assumes
that the model is fixed independently of the probe draws. Contamination
breaks that assumption. A model that has been trained on a benchmark's
items is the modification of clause (ii) run in reverse, with agreement
placed on exactly the queried set, and its score certifies nothing off
that set. Memorisation of training sequences is measurable and grows
with scale~\cite{carlini2023quantifying}.
Oren et~al.~\cite{oren2024contamination} detect contamination without access to
training data by exchangeability: an uncontaminated model has no reason
to prefer a benchmark's canonical item order over random permutations,
the order-invariance null of Proposition~\ref{prop:order} used in
reverse. Dynamic and continually
refreshed benchmarks~\cite{kiela2021dynabench} restore independence by
drawing probes after the model is fixed.

An evaluation that is to bear on inductive claims should therefore report
its probe distribution, its resolution $1/N$, whether the protocol was
fixed before the model, and whether conditionals were read or sampled
(Section~\ref{sec:discussion}).

\providecommand{\CVST}{\mathcal{C}_V}
\providecommand{\argmin}{\operatorname*{arg\,min}}

\subsection{Scale}
\label{app:limits-scale}

\begin{proof}[Proof of Theorem~\ref{thm:si-diag}]\label{app:proof-thm:si-diag}
(i) The $V$ conditionals sum to at most $1$, so the least is at most
$1/V$. (ii) A fixed routine, given a shortest program for $q$, computes at
each stage the $V$ rationals $q(\cdot\mid x_{<t})$ (total) and emits the
least minimiser: a monotone program of length $K(q)+O(1)$ for $\mu_q$.
Apply the program bound.
\end{proof}

\begin{corollary}[No computable predictor has Solomonoff's property]
\label{cor:si-nocomputable}
No computable predictor $q$ has $R_\infty(q;\mu) < \infty$ for every
computable measure $\mu$; hence for no $f$ does it satisfy
$R_\infty(q;\mu) \le f(K(\mu))$ for all computable $\mu$. This covers every
closed computable pipeline of Definition~\ref{def:si-computable}.
\end{corollary}

\begin{proof}[Proof of Corollary~\ref{cor:si-nocomputable}]\label{app:proof-cor:si-nocomputable}
For rational conditionals use Theorem~\ref{thm:si-diag}. For real
computable ones, approximate to additive precision $1/(4V)$ and emit a
symbol with minimal approximation; its probability is at most $3/(2V)$, so
each step costs at least $\log(2V/3) > 0$, at complexity $K(q)+O(1)$~\cite{legg2006}.
\end{proof}

\begin{proof}[Proof of Theorem~\ref{thm:si-classdiag}]\label{app:proof-thm:si-classdiag}
At most one $i$ qualifies at each $t$ (codewords read from the least
significant digit are prefix-free), and as $(\ell_i)$ is nondecreasing
only finitely many $i$ need checking; so the schedule is computable.
(i) The steps $t \le n$ with $i(t) = i$ include every $t$ with
$2^{\ell_i} \le t-1 < n$ in one residue class modulo $2^{\ell_i}$, at least
$n2^{-\ell_i} - 2 \ge \tfrac12 nd_i - 2$ of them, each costing $q_i$ at
least $\log V$. (ii) A fixed routine with shortest programs for $G$ and
$(d_i)$ computes $i(t)$, the rationals $G(i(t),x^{*}_{<t},\cdot)$ and their
least minimiser; apply the program bound.
\end{proof}

\begin{corollary}[Concrete transformer language models]
\label{cor:si-transformers}
Fix $\Ac$. A \emph{concrete transformer LM} comprises a finite
architecture, finitely many parameters in a stated finite-precision
arithmetic, a deterministic decoding procedure (chain of thought included, and
any tool whose code is a deterministic program bundled into the description)
and an \emph{enforced clock}, one fixed total computable function of the
context length shared by the whole family, halting decoding with a default
output. These form a uniformly computable family of total rational-computable
predictors, so one computable $x^{*}$ of complexity $c^{*}$
(Theorem~\ref{thm:si-classdiag}) gives each eventually linear regret while $M$
pays at most $c^{*}$. A pipeline that queries an external resource
$\mathcal{O}$ is outside this family: by the discussion after
Definition~\ref{def:si-computable} it is computable only relative to
$\mathcal{O}$, so $x^{*}$ is $\mathcal{O}$-computable and the second clause
holds with $M^{\mathcal{O}}$ in place of $M$. Nothing is claimed for plain $M$
against such a pipeline.
\end{corollary}

\begin{proof}[Proof of Corollary~\ref{cor:si-transformers}]\label{app:proof-cor:si-transformers}
Enumerate descriptions, mapping invalid ones to the uniform predictor.
Forward passes halt, the clock makes decoding halt, and rounded outputs
are rational, as are their renormalisations; so $G(i,x,a)$ is total.
\end{proof}

\begin{remark}[The diagonal is cheap]
\label{rem:si-cheap}
With descriptions enumerated by length, one polynomial clock and
$d_i = 2^{-i}$, the active index satisfies $i(t) \le \log_2 t + 1$, so
$x^{*}_t$ is computable in time polynomial in $t$. A user continuing with a
least likely served symbol realises a finite-horizon instance of
Theorem~\ref{thm:si-diag}.
\end{remark}

\begin{theorem}[Diagonal environments behind every prefix; after~\cite{putnam1963,legg2006}]
\label{thm:breakers}
Let $q$ be total rational-computable, $p \in \Ac^{*}$,
$z_p := p\,y$ with $y_{t+1}$ the least minimiser of $q(\cdot\mid p\,y_{1:t})$,
and $q_{\max}(p) := \max_a q(a\mid p)$. Then:
(i) $R_n(M;z_p) \le K(p)+K(q)+O(1)$;
(ii) for $n > |p|$, $R_n(q;z_p) \ge L_{|p|}(q;p) + (n-|p|)\log V$, and the
loss at step $|p|+1$ is at least
$\max\{\log V, \log\frac{V-1}{1-q_{\max}(p)}\}$;
(iii) $R_n(q;z_p) - R_n(M;z_p) \ge (n-2|p|)\log V - K(q) - O(\log(|p|+2))$;
(iv) $p\mapsto z_p$ is injective on $\Ac^{\ell}$, giving $V^{\ell}$ such
environments, each of complexity at most
$\ell\log V + K(q) + O(\log(\ell+2))$.
\end{theorem}

\begin{proof}[Proof of Theorem~\ref{thm:breakers}]\label{app:proof-thm:breakers}
(i) ``Emit $p$, then a least likely symbol of $q$ at each context'' has
length $K(p)+K(q)+O(1)$. (ii) After $p$ each symbol has probability at most
$1/V$; at the first step the $V-1$ symbols other than the most likely share
at most $1-q_{\max}(p)$. (iii) Subtract (i) from (ii), with
$K(p) \le |p|\log V + O(\log(|p|+2))$. (iv) Distinct $p$ differ within
$\ell$ symbols; the bound is (i).
\end{proof}

\subsection{A model family and the quantifier exchange}
\label{sec:si-classdiag}

A scaling family generated by one program from an index, spanning
architectures, sizes, and training runs by a fixed computable procedure
on fixed data, is effectively enumerable. The next theorem constructs a
single environment against all of its members at once.

\begin{theorem}[Diagonal environment for a family; after~\cite{putnam1963,legg2006}]
\label{thm:si-classdiag}
Let $(q_i)_{i\ge1}$ be uniformly computable total rational-computable
predictors, so that a single total machine $G$ satisfies
$G(i,x,a) = q_i(a\mid x)$. Let $(d_i)$ be a computable nonincreasing
sequence of positive rationals with $\sum_i d_i \le 1$. Put
$\ell_i := \lceil\log(1/d_i)\rceil$, and let $\sigma_i$ be the canonical
prefix-free codeword of length $\ell_i$~\cite[Ch.~5]{cover2006}. Step $t$
is assigned to $i(t) := i$ when the $\ell_i$ lowest-order binary digits of
$t-1$, read from the lowest-order digit up, spell $\sigma_i$ and $t-1$
has at least $\ell_i$ digits, and to
$i(t) := 1$ if no $i$ qualifies. Define
$x^{*}_t := \min\argmin_a q_{i(t)}(a\mid x^{*}_{<t})$ and let $\mu^{*}$ be
the point measure on $x^{*}$. Then:
\begin{enumerate}
\item[\textup{(i)}] for every $i$ and $n$,
$R_n(q_i;\mu^{*}) \ge (\tfrac12 n d_i - 2)\log V$;
\item[\textup{(ii)}] $x^{*}$ is emitted by one monotone program of length
$c^{*} := K(G) + K((d_i)) + O(1)$, independent of $i$, so
$K(\mu^{*}) \le c^{*}$ and $R_n(M;\mu^{*}) \le c^{*}$ for all $n$.
\end{enumerate}
\end{theorem}

\noindent(Proof: Appendix~\ref{app:proof-thm:si-classdiag}.)

\paragraph{The family diagonal has no force on any particular model at an observable horizon.}\label{rem:classdiag-force}
Because $(d_i)$ is nonincreasing and summable, $d_i \le 1/i$. Clause
(i) therefore forces a loss on member $i$ only at horizons beyond $4i$.
A concrete model of $10^{11}$ bits sits at an index of order
$2^{10^{11}}$ in any enumeration of such models ordered by description
length; an enumeration that lists it early buys force on that model only,
which is Theorem~\ref{thm:si-diag} again. The environment
$\mu^{*}$ thus forces nothing on it at any horizon that will ever be
observed. The caveat applies to every statement below about ``all
models'': Theorem~\ref{thm:si-classdiag} concerns families and the
quantifier exchange discussed next. The result with force
on a given model is Theorem~\ref{thm:si-diag}, together with its
resource-bounded form (Section~\ref{sec:deployed}).

Subject to that caveat, every concrete transformer language model whose
decoding loop runs under an enforced step clock is a member of one such
family, and so defeated by one fixed sequence
(Corollary~\ref{cor:si-transformers}). The clock is needed because a
budget bounded only by some computable function cannot be checked from
a description, and so cannot be enumerated.

\paragraph{Pointwise versus uniform: the quantifier exchange.}\label{rem:si-pointwise}\label{sec:scale-exchange}\label{rem:scale-schedules}\label{rem:scale-limitcomp}
Solomonoff's property has the form ``there is a predictor that, for every
environment and every horizon, pays a bounded constant''. The evidence
from scaling has the form ``for every environment and horizon that we
test, there is a model in the family that does well''. Here the model
depends on the horizon. Theorem~\ref{thm:si-classdiag} shows that the
second form is compatible with the failure of the first on one fixed
environment of constant complexity. Every finite benchmark, on every
distribution and at every horizon $n$, is consistent with late members
performing excellently, $\mu^{*}$ included, while every fixed member
eventually has linear regret on $\mu^{*}$. How late ``late'' is depends
on the schedule, which the adversary chooses: for any computable
nondecreasing $n(i)$ with $\sum_i 1/n(i) \le 1$, $d_i = 1/n(i)$ defeats
member $i$ beyond horizon $4n(i)$. Scaling
curves~\cite{kaplan2020scaling,hoffmann2022chinchilla} are evidence for the
pointwise form on $\rho$. A better curve shrinks $\eps$ and leaves
Theorems~\ref{thm:si-diag} and~\ref{thm:si-tail} untouched. Nor does
passing to the limit help. $M$'s conditional is $\Delta^{0}_{2}$
\cite{leike2015computability}: limit-computable, so computable families
converging pointwise to it exist, and pointwise convergence is not what
fails. What fails is selection: a rule that selects, ensembles, anneals or
switches between members by any computable function of the data is itself
one computable predictor, with its own diagonal, so no computable rule can
tell at a finite stage which member to use, and a family approaches
Solomonoff's predictor only relative to an uncomputable selection oracle.

\subsection{Asymptotic criteria}
\label{app:limits-tail}

\begin{fact}[Finite modifications; folklore, cf.\ \cite{blackwelldubins1962,kalailehrer1994}]
\label{lem:si-tail}
Let $q' = q$ off a finite $F \subseteq \Ac^{*}$ and
$\Lambda(F) := |\{|x| : x \in F\}|$. (i) If $q'(\cdot\mid x) \ge \gamma > 0$
on $F$, then $R_n(q';\mu) \le R_n(q;\mu) + \Lambda(F)\log(1/\gamma)$ for
every $\mu$ and $n$. (ii) If both have conditionals at least $\gamma$ on $F$
(a \emph{finite full-support modification}), then
$|R_n(q;\mu) - R_n(q';\mu)| \le \Lambda(F)\log(1/\gamma)$. Properties
preserved both ways by every such modification are
\emph{tail-invariant}; among them are the \emph{consistency properties}
finite total regret, vanishing per-step excess loss and $R_n/n \to 0$ on
each source of a class.
\end{fact}

\begin{proof}[Proof of Fact~\ref{lem:si-tail}]\label{app:proof-lem:si-tail}
By \eqref{eq:si-regret},
$R_n(q';\mu) - R_n(q;\mu) = \sum_{x\in F,|x|<n}\sum_a
\mu(xa)\log\frac{q(a\mid x)}{q'(a\mid x)}$. If some $q(a\mid x) = 0$ with
$\mu(xa) > 0$, (i) is trivial. Otherwise each logarithm is at most
$\log(1/\gamma)$, and for each length the strings of $F$ form an antichain
of $\mu$-mass at most $1$ (Fact~\ref{lem:antichain}), so the sum is at most
$\Lambda(F)\log(1/\gamma)$. (ii) is (i) both ways; per-step excess losses
coincide after step $\max_{x\in F}|x|+1$.
\end{proof}

\begin{theorem}[The Solomonoff bound is not tail-invariant]
\label{thm:si-tail}
Let $M_N$ be as in~\eqref{eq:tail-MN} and $\mu_0$ the point measure on
$0^{\infty}$. (i) $M_N$ is a finite full-support modification of $M$ with
$\Lambda(F) = N$, so $R_\infty(M_N;\mu) < \infty$ and
$R_n(M_N;\mu)/n \to 0$ for every computable $\mu$. (ii)
$R_n(M_N;\mu_0) \ge \min(n,N)\log V$ and
$R_\infty(M_N;\mu_0) = N\log V + O(1)$. (iii) For every
$f:\Nb\to[0,\infty)$, ``$R_\infty(\cdot;\mu) \le f(K(\mu))$ for every
computable $\mu$'' fails for $M_N$ once $N > f(K(\mu_0))/\log V$.
\end{theorem}

\begin{proof}[Proof of Theorem~\ref{thm:si-tail}]\label{app:proof-thm:si-tail}
(i) $M(a\mid x) > 0$ everywhere; apply Fact~\ref{lem:si-tail}(ii) and
Fact~\ref{fact:si-solomonoff}. (ii) The first $N$ steps cost $\log V$ each,
and $\sum_{t=N+1}^{n} -\log M(0\mid 0^{t-1}) = \log\frac{M(0^N)}{M(0^n)}
\le -\log M(0^n) = O(1)$ by the program bound. (iii) follows.
\end{proof}

\begin{proposition}[Prior-blindness; cf.\ \cite{wood2013nonequivalence,leike2015bad}]
\label{prop:si-prior}
Let $\Ac = \{0,1\}$, $z^{(j)} := 0^{j}1^{\infty}$ with point measure
$\mu_j$ ($K(\mu_j) = O(\log(j+2))$), and $\xi_w$ the Bayes mixture
\eqref{eq:bg-mixture} with weights $w_j > 0$, $\sum_j w_j \le 1$, and
uniform conditionals at contexts of probability zero.
(i) $R_n(\xi_w;\mu_j) = \log(1/w_j)$ for $n \ge j+1$, and
$R_n(\xi_w;\mu_j) = \log(1/\sum_{i\ge n}w_i) \le \log(1/w_j)$ for $n \le j$.
(ii) $\xi_w$ has finite total regret on every $\mu_j$ and per-step excess
loss exactly $0$ for $t > j+1$.
(iii) For any $f:\Nb\to\Nb$, $w_j := 2^{-f(j)-j-1}$ gives
$R_\infty(\xi_w;\mu_j) = f(j)+j+1$, with $\xi_w$ computable when $f$ is,
whereas $R_\infty(M;\mu_j) \le K(\mu_j)+O(1) = O(\log(j+2))$.
\end{proposition}

\begin{proof}[Proof of Proposition~\ref{prop:si-prior}]\label{app:proof-prop:si-prior}
(i) For $n \ge j+1$ the prefix $0^{j}1^{n-j}$ is inconsistent with
$z^{(i)}$ for $i > j$ (a $0$ at position $j+1$) and for $i < j$ (a $1$ at
position $i+1 \le j$), so $\xi_w(z^{(j)}_{1:n}) = w_j$; for $n \le j$,
$\xi_w(0^n) = \sum_{i\ge n}w_i$. (ii) $\xi_w(1\mid 0^j1^k) = 1$ for
$k \ge 1$. (iii) The weights sum to at most $1$; the conditionals at $0^j$
are ratios of tail sums, computable since $\sum_{i\ge k}w_i \le 2^{-k}$, and
null contexts are recognisable from their form.
\end{proof}

\begin{corollary}[Asymptotic criteria are constant along the frontier]
\label{cor:blind}
Let $(\mu_k)_{k\in J}$ be pairwise mutually singular probability measures,
$J$ countable, $|J| \ge 2$.
(i) Every $q$ with $R_\infty(q;\mu_k) < \infty$ for all $k$ has every
consistency property of Fact~\ref{lem:si-tail} on every member. This
includes every $\xi_w$ with all $w_k > 0$, $\sum_k w_k = 1$,
and each leaky variant $(1-\theta)\xi_w$, $\theta \in [0,1)$, with profile
$(\log\frac1{w_k} + \log\frac1{1-\theta})_k$; their profiles differ by
unbounded amounts.
(ii) \textup{(Frontier membership is not tail-invariant.)} If all
$w_k > 0$, $\sum_k w_k = 1$, and $\mu_{k_0}(xa) > 0$ for every $a$ at some
context $x$, then replacing $\xi_w(\cdot\mid x)$ by a full-support
$p \ne \xi_w(\cdot\mid x)$ gives a finite full-support modification with
$\sum_k 2^{-R_\infty(\cdot;\mu_k)} < 1$.
\end{corollary}

\begin{proof}[Proof of Corollary~\ref{cor:blind}]\label{app:proof-cor:blind}
(i) A finite $R_\infty$ is a convergent series of nonnegative terms in
\eqref{eq:si-regret}. The measure part of $(1-\theta)\xi_w$ is
$(1-\theta)\sum_k w_k\mu_k$ (Lemma~\ref{lem:measurepart}); the profile
follows from Theorem~\ref{thm:frontier}(ii) and its carrier argument.
(ii) $\xi_w(xa) \ge w_{k_0}\mu_{k_0}(xa) > 0$, so the modification $q'$
has full support at $x$. With $g(a) := p(a)/\xi_w(a\mid x)$,
$q' = g(a)\,\xi_w$ on Borel subsets of the cylinder $\Gamma_{xa}$ of sequences extending
$xa$. Were the Kraft sum
$1$, then $q' = \sum_k w'_k\mu_k$ with $\sum_k w'_k = 1$
(Theorem~\ref{thm:frontier}(iii)); evaluating on
$\Omega_{k_0}\cap\Gamma_{xa}$, with the disjoint carriers $\Omega_k$ of
Appendix~\ref{app:frontier-proofs}, gives
$w'_{k_0}\mu_{k_0}(xa) = g(a)\,w_{k_0}\mu_{k_0}(xa)$ for all $a$. So $g$ is
constant, hence $g \equiv 1$: a contradiction.
\end{proof}

\begin{corollary}[The implication in (C) is invalid on every class with a deterministic member]
\label{cor:si-invalid}
Let $\Mc$ be a uniformly computable countable class of sources with a
deterministic member $\mu_0$, $w$ positive computable weights with
$\sum_\mu w_\mu\le1$ and computable tail sums, $\eta \in (0,1)$,
$\xi_{\mathrm{unif}}$ the uniform measure, and
$\bar\xi := (1-\eta)\xi_w + \eta\,\xi_{\mathrm{unif}}$. Then $\bar\xi$ is
computable and has each consistency property $\Phi$ of
Fact~\ref{lem:si-tail} on $\Mc$. So does the computable $\bar\xi_N$,
uniform on $\Ac^{<N}$ and equal to $\bar\xi$ elsewhere, while $R_\infty(\bar\xi_N;\mu_0) \ge N\log V$. Hence for no $f$
does $\Phi$ entail ``$R_\infty(\cdot;\mu) \le f(K(\mu))$ for all
$\mu \in \Mc$'' among computable predictors.
\end{corollary}

\begin{proof}[Proof of Corollary~\ref{cor:si-invalid}]\label{app:proof-cor:si-invalid}
$\xi_w$ is computable by Fact~\ref{thm:attainable}, which uses the uniform
computability of $\Mc$, and $\bar\xi(x) \ge \eta V^{-|x|} > 0$, so the
conditionals of $\bar\xi$ are computable. Since $\bar\xi \ge (1-\eta)w_\mu\mu$,
$R_\infty(\bar\xi;\mu) \le \log\frac1{(1-\eta)w_\mu}$, and finite total
regret gives the other consistency properties (proof of
Corollary~\ref{cor:blind}(i)). By
Fact~\ref{lem:si-tail}(i) with $F = \Ac^{<N}$ and $\gamma = 1/V$,
$R_n(\bar\xi_N;\mu) \le R_n(\bar\xi;\mu) + N\log V$, with per-step excess
losses equal after step $N$. On $\mu_0$ the first $N$ steps cost $\log V$
each; take $N > f(K(\mu_0))/\log V$.
\end{proof}

\providecommand{\db}[2]{d\!\left(#1 \,\|\, #2\right)}
\providecommand{\tlo}{t^{-}}
\providecommand{\thi}{t^{+}}

\section{Corpora, attainability and steering}
\label{app:corpus}

\subsection{Closure}

\begin{proof}[Proof of Theorem~\ref{thm:closure}]\label{app:proof-thm:closure}
(i) Corollary~\ref{cor:specialise}(ii) with $r=0$, the limit being that of
the nondecreasing $\dl{\nu_n}{\mu_n}$ (Lemma~\ref{lem:measurepart}).
(ii) Finite relative entropy on $\Ac^\infty$ forces $\nu\ll\mu$~\cite[Ch.~7]{gray2011entropy}; if $\nu\le c\,\mu$, each
$\dl{\nu_n}{\mu_n}\le\log c$.
(iii) For $\nu\perp\mu$, $\TV(\nu_n,\mu_n)\to1$, so
$\dl{\nu_n}{\mu_n}\to\infty$ by Bretagnolle--Huber~\cite[Ch.~2]{tsybakov2009}; distinct stationary ergodic measures are
mutually singular. For non-atomic $\mu$,
$\mu(z_{1:n})\downarrow\mu(\{z\})=0$.
(iv) Regret depends on $q$ only through its cylinder values, those of
$\mu$ by (i). If some $\mu(x)=0$, the computable environment $x0^\infty$
already has infinite regret; otherwise $\mu$'s conditionals
$\mu(xa)/\mu(x)$ are computable, and Corollary~\ref{cor:si-nocomputable}
applies to $\mu$.
\end{proof}

\begin{corollary}[The closure]
\label{cor:closure}
Let $\mu$ be computable and $R_\infty(q;\mu)=0$.
\textup{(a)}~$q$ has unbounded regret on every $\nu\perp\mu$, and, if
$\mu$ is non-atomic, on every deterministic environment; no computable
predictor has finite regret on every computable environment.
\textup{(b)}~If $\mu=\sum_iw_i\mu_i+(1-\sum_iw_i)\mu'$, then
$R_\infty(q;\mu_i)\le\log(1/w_i)$.
\textup{(c)}~No computable $\mu$ satisfies $\nu\le c_\nu\mu$ with
$c_\nu<\infty$ for every computable $\nu$; the simplicity prior must be
supplied, not fitted.
\end{corollary}

\begin{proof}[Proof of Corollary~\ref{cor:closure}]
(a) Theorem~\ref{thm:closure}(iii)--(iv) and
Corollary~\ref{cor:si-nocomputable}. (b) Theorem~\ref{thm:closure}(ii)
with $c=1/w_i$. (c) Theorem~\ref{thm:closure}(ii) with $q:=\mu$ gives
$R_\infty(\mu;\nu)\le\log c_\nu$ for every computable $\nu$. Domination of
the point mass on $x0^\infty$ forces $\mu(x)>0$ for every $x$, so $\mu$ is
a computable predictor with finite regret on every computable $\nu$,
contradicting the last clause of (a).
\end{proof}

\begin{remark}[Two premiums]
\label{rem:price-of-universality}
Against the source, universality costs $R_\infty(M;\mu)\le K(\mu)+O(1)$
(Fact~\ref{fact:si-solomonoff}), with $K(\mu)$ unknown for natural text;
through a trained model (Section~\ref{sec:closure}) it is computable.
Both premiums are bounded (the first against the source, the second
pathwise against the model), and each is informative only where its
reference predicts well; the reverse trade (Theorem~\ref{thm:closure}(iii))
is not bounded.
\end{remark}

\subsection{Attainability and the hierarchy}

\begin{fact}[Class-restricted attainability; \cite{hutter2003optimality,hutter2005uai}]
\label{thm:attainable}
Let $\Cc=(\mu_i)_{i\ge1}$ be a uniformly computable family and $w$ a
computable prior with $w_i>0$, $\sum_iw_i\le1$ and computable tail sums.
Then $\xi_w:=\sum_iw_i\mu_i$ is a computable semimeasure
and $R_\infty(\xi_w;\mu_i)\le\log(1/w_i)$ for every $i$. If moreover
$\xi_w(x)>0$ for every $x$, its conditionals $\xi_w(xa)/\xi_w(x)$ are
computable.
\end{fact}

\begin{proof}[Proof of Fact~\ref{thm:attainable}]
To precision $\eta$: take $N$ with $\sum_{i>N}w_i<\eta/2$ and the first
$N$ terms to precision $\eta/(2N)$. As $\xi_w\ge w_i\mu_i$,
$R_n(\xi_w;\mu_i)\le\log(1/w_i)$ for every $n$. The conditionals are ratios
of computable reals with positive denominators, hence computable.
\end{proof}

\begin{proposition}[The complexity form is unattainable]
\label{prop:nokform}
Let $(z^{(j)})_{j\ge1}$ be a uniformly computable family of pairwise
distinct deterministic environments and $q$ a computable predictor. No
constant $c$ gives $R_\infty(q;z^{(j)})\le K(z^{(j)})+c$ for every $j$.
\end{proposition}

\begin{proof}[Proof of Proposition~\ref{prop:nokform}]
$n\mapsto q(z^{(j)}_{1:n})$ is nonincreasing and computable uniformly in
$(j,n)$, so $v_j:=\lim_nq(z^{(j)}_{1:n})$ is upper semicomputable uniformly
in $j$ and $R_\infty(q;z^{(j)})=-\log v_j$. For finite $J'$ the prefixes
$z^{(j)}_{1:n}$, $j\le J'$, form an antichain for large $n$, so
$\sum_jv_j\le1$ (Fact~\ref{lem:antichain}) and $f(j):=-\log v_j$ is
unbounded and lower semicomputable. If $f(j)\le K(z^{(j)})+c$, then
$f(j)\le K(j)+c'$. Enumerating lower approximations of $f$ until some
$j_m$ has $f(j_m)>m+c'$ computes $j_m$ from $m$, so
$m<K(j_m)\le2\log m+O(1)$, false for large $m$.
\end{proof}

\begin{theorem}[The hierarchy is strict and never terminates]
\label{thm:nocollapse}
Let $q$ be a total rational-computable predictor, $\delta\in(0,1)$
computable, $\mu_q$ the point measure on the diagonal sequence $x^q$ of
Theorem~\ref{thm:si-diag}, and $q'_\delta:=(1-\delta)q+\delta\mu_q$ (a
mixture of joint laws, with uniform conditionals at contexts $x$ where
$q'_\delta(x)=0$). Then $q'_\delta$ is computable,
\textup{(i)}~$R_n(q'_\delta;\mu)\le R_n(q;\mu)+\log\frac1{1-\delta}$ for
every $\mu$ and $n$, and
\textup{(ii)}~$R_\infty(q'_\delta;\mu_q)\le\log\frac1\delta$, while
$R_n(q;\mu_q)\ge n\log V$. Interpolating conditionals at the fixed weight
$\delta$ instead loses at least $n\log\frac1{\delta+(1-\delta)/V}$ on
$x^q$.
\end{theorem}

\begin{proof}[Proof of Theorem~\ref{thm:nocollapse}]
$\mu_q$ is computable (Theorem~\ref{thm:si-diag}(ii)). The values
$q'_\delta(x)=(1-\delta)q(x)+\delta\mu_q(x)$ are computable, and
$q'_\delta(x)=0$ iff the rational $q(x)$ vanishes and $x$ is not a prefix
of $x^q$, which is decidable; so the conditionals of $q'_\delta$ are
computable. (i) follows from
$q'_\delta\ge(1-\delta)q$, (ii) from $q'_\delta(x^q_{1:n})\ge\delta$
and Theorem~\ref{thm:si-diag}(i), and the last claim from $q$'s
conditional on $x^q_t$ being at most $1/V$.
\end{proof}

\begin{remark}[The price of coverage]
\label{rem:price}
If a computable predictor $q$ has finite regret on every deterministic $z$
with $K(z)\le s$, then $K(q)\ge s-O(1)$: its own diagonal has infinite
regret and complexity at most $K(q)+O(1)$ (Theorem~\ref{thm:si-diag}(ii)
for rational conditionals, the proof of
Corollary~\ref{cor:si-nocomputable} for real ones), so it is not among
those $z$. Legg~\cite{legg2006} proves the analogue for eventual zero--one prediction.
\end{remark}

\begin{remark}[Concrete rungs]
\label{rem:rungs}
Krichevsky--Trofimov pays $\tfrac12\log n+O(1)$ per free parameter on
memoryless sources~\cite{krichevsky1981}; context-tree weighting, an
explicit prior and parametric cost on finite-memory tree sources~\cite{willems1995};
enumerated bounded-program mixtures~\cite{ficl2026}
instantiate Fact~\ref{thm:attainable}; Levin search and the speed prior
reach time-bounded classes (Appendix~\ref{sec:rb}).
\end{remark}

\begin{corollary}[Mixtures with an enumerated component]
\label{cor:neurosymbolic}
Let $q_N$ be any predictor, $q_E$ a mixture over a uniformly computable
class $\Cc_E$ with computable weights having computable tail sums and
$q_E(x)>0$ for all $x$, $\delta\in(0,1)$, and
$q_\delta:=(1-\delta)q_N+\delta q_E$ with conditionals
$q_\delta(xa)/q_\delta(x)$. For every $\mu$ and $n$,
$R_n(q_\delta;\mu)\le\min\{R_n(q_N;\mu)+\log\frac1{1-\delta},\,
R_n(q_E;\mu)+\log\frac1\delta\}$, and $q_\delta(x)\le\max\{q_N(x),q_E(x)\}$
for every $x$. If $q_N$ and $\delta$ are computable, so is $q_\delta$, which then has
its own real-valued diagonal (Corollary~\ref{cor:si-nocomputable}).
\end{corollary}

\begin{proof}[Proof of Corollary~\ref{cor:neurosymbolic}]
$q_\delta\ge(1-\delta)q_N$ and $q_\delta\ge\delta q_E$ pointwise, and
$q_\delta(x)$ is a convex combination. For computable $q_N$ and $\delta$ the values
$q_\delta(x)\ge\delta q_E(x)>0$ are computable
(Fact~\ref{thm:attainable}), hence so are their ratios.
\end{proof}

\begin{remark}[Selection trails linearly]
\label{rem:neurosymbolic}
A selector uses at each step the conditional of one component, chosen
from the history. For $V=2$ and $\gamma<\frac12$, let $q_A$, $q_B$ give
$0$, respectively $1$, probability $1-\gamma$ at every step, and let the
sequence emit the symbol given probability $\gamma$ by the selected
component (computable when the rule is). The selector pays
$n\log\frac1\gamma$ and the two components together
$n(\log\frac1\gamma+\log\frac1{1-\gamma})$, so it trails the better by at
least $\frac n2\log\frac{1-\gamma}\gamma$.
\end{remark}

\subsection{Steering by demonstrations}

\begin{proposition}[Steering; after~\cite{xie2022bayesian,wolf2024alignment}]
\label{prop:manipulable}
Let $Q=\sum_{h\in\mathcal H}w_hP_h$ over a countable class, $w_{h_y}>0$,
and $c=d_1\cdots d_m$ with $P_{h_y}(c)>0$; put
$\Lambda_h(c):=P_h(c)/P_{h_y}(c)$. Then
$\TV(Q(\cdot\mid c),P_{h_y}(\cdot\mid c))\le1-\pi(h_y\mid c)
\le w_{h_y}^{-1}\sum_{h\ne h_y}w_h\Lambda_h(c)$.
\textup{(i)}~If $\Lambda_h(c)\le2^{-rm}$ for all $h\ne h_y$, the distance
is at most $\eta$ once $m\ge(\log(1/w_{h_y})+\log(1/\eta))/r$.
\textup{(ii)}~If competitors $\mathcal H'$ of mass $w_{\mathcal H'}$ have
every per-demonstration likelihood ratio at least $2^{-r}$, then
$\pi(h_y\mid c)\ge1-\eta$ requires
$m\ge(\log(w_{\mathcal H'}/w_{h_y})+\log((1-\eta)/\eta))/r$.
\end{proposition}

\begin{proof}[Proof of Proposition~\ref{prop:manipulable}]
$Q(\cdot\mid c)=\sum_h\pi(h\mid c)P_h(\cdot\mid c)$ with
$\pi(h\mid c)\propto w_h\Lambda_h(c)$, so by convexity the distance is at
most $\sum_{h\ne h_y}\pi(h\mid c)\le\sum_{h\ne h_y}w_h\Lambda_h(c)/w_{h_y}$.
(i) $\sum_{h\ne h_y}w_h\Lambda_h(c)\le2^{-rm}$, so the distance is at most $2^{-rm}/w_{h_y}$.
(ii) $\pi(h_y\mid c)\le w_{h_y}/(w_{h_y}+w_{\mathcal H'}2^{-rm})$.
\end{proof}

\noindent\label{rem:manipulable-scope}Clause~(ii) bounds concentration of the posterior, not steering of the
served law. Setting $w_{h_y}=0$ does not prevent steering if other
hypotheses reproduce $h_y$'s predictions.

\subsection{Tools, retrieval and test-time compute}
\label{app:tools-more}

\paragraph{Tools and retrieval.}
A pipeline that queries retrieval or tools satisfies
Definition~\ref{def:si-computable} only relative to what it queries, and
the diagonal of Section~\ref{sec:scale} applies relative to the same
resources. What
remains is how components are combined, and the results distinguish three
ways.
\begin{itemize}
\item \emph{Mixture of joint laws.} $q_\delta=(1-\delta)q_N+\delta q_E$,
with posterior-weighted conditionals, has regret at most
$\min\{R_n(q_N;\mu)+\log\frac1{1-\delta},\,R_n(q_E;\mu)+\log\frac1\delta\}$
on every environment (Corollary~\ref{cor:neurosymbolic}).
\item \emph{Per-step interpolation.} Interpolating conditionals at a
fixed weight, as the kNN-LM does~\cite{khandelwal2020knnlm} and as
RAG-Token does across retrieved documents~\cite{lewis2020rag}, loses that
bound. Against a deterministic environment that one component predicts
exactly, while the other keeps its conditionals at least $\gamma$, the
regret grows linearly in $n$ (Corollary~\ref{cor:specialise}(iii)).
RAG-Sequence, which marginalises documents over whole sequences, is a
mixture of joint laws.
\item \emph{Selection.} Routing a query to a tool and substituting its
answer, as in learned tool use~\cite{schick2023toolformer}, is a
selector. A selector can trail its better component by an amount linear
in $n$ (Remark~\ref{rem:neurosymbolic} gives an example).
\end{itemize}
With $q_E$ an enumerated mixture, $q_\delta$ carries certified constants
on the class $\Cc_E$ and the network's uncertified constants elsewhere,
so its guarantee has two tiers and should be reported that way. The
competence sets of the two components overlap: trained networks copy in
context~\cite{olsson2022induction}, prefer simple
sequences~\cite{goldblum2024nfl}, and extrapolate numeric sequences
zero-shot~\cite{gruver2023timeseries}. So the gain is a certificate on a declared
class, not new coverage of a disjoint region. There is no gain in depth:
on every string the mixture is no better than its better component.

\paragraph{Test-time compute.}
Chain of thought and search leave the end-to-end map computable, so they
do not escape the diagonal. They change which resource-bounded class a
model can cover, at a cost of enumeration exponential in the description
length $S$ (Section~\ref{sec:deployed}, Appendix~\ref{sec:rb}).

\section{Training: certification, mixtures and refitting}
\label{app:training}

\paragraph{Setting.}
Notation as in Section~\ref{sec:icl}, with $\sum_k\pi_k=1$,
$\psi_n(q):=\dl{P_n}{q_n}$,
$\thi(\lambda,\psi):=\max\{t\in[\lambda,1]:\db\lambda t\le\psi\}$
($\db\lambda\cdot$ decreases on $(0,\lambda]$, increases on $[\lambda,1]$)
and $c(\eta)=\tfrac{\ln2}2\eta^2+O(\eta^3)$. A family is \emph{singular at
horizon $n$} if the $\mu_{k,n}$ have disjoint supports $A_k$;
$\Delta_k(q):=R_n(q;\mu_k)-R_n(P;\mu_k)$,
$\pi_{\mathcal J}:=\sum_{k\in\mathcal J}\pi_k$.

\begin{proposition}[Clauses added to Proposition~\ref{prop:pretrain-certify}]
\label{thm:training-inherit}
Let $\psi:=\psi_n(q)<\infty$.
\textup{(i)}~With $u:=\sqrt{2\psi\ln2/\lambda}$,
$\log\frac\lambda{\tlo(\lambda,\psi)}\le\min\{(\psi+(1-\lambda)\log\frac1{1-\lambda})/\lambda,\,\log\frac1{1-u}\}$
(second term if $u<1$). The bound is nonincreasing in $\lambda$; for
every family the attaining $q$ equals $(\tlo/\lambda)P_n$ on $E$ and
$((1-\tlo)/(1-\lambda))P_n$ off $E$.
\textup{(ii)}~On a family singular at horizon $n$, with
$0<\pi_{\mathcal J}<1$ and $\eta>0$, the Bayes mixture
$q^{\mathcal J,\eta}$ with weights $\pi_k2^{-\eta}$ on $\mathcal J$ and
$\pi_k(1-\pi_{\mathcal J}2^{-\eta})/(1-\pi_{\mathcal J})$ off $\mathcal J$
has $\psi_n=\db{\pi_{\mathcal J}}{\pi_{\mathcal J}2^{-\eta}}$ and
$\Delta_k=\eta$ on $\mathcal J$, so $s^*$ is attained there; always
$s^*<\psi/c(\eta)$ for $\psi>0$; and every $q$ has
$R_n(q;\mu_k)\ge\log(1/\thi(\pi_k,\psi))$.
\textup{(iii)}~For every $E\subseteq\Ac^n$,
$q_n(E)\le\thi(P_n(E),\psi)$, and $\nu_n(E)\ge\thi(P_n(E),\psi)$ implies
$R_{n'}(q;\nu)\ge\db{\nu_n(E)}{\thi(P_n(E),\psi)}$ for $n'\ge n$. Under
disjoint supports, regret $R$ on $\nu$ requires
$\psi\ge\log\frac1{1-2^{-R}}\ge2^{-R}/\ln2$.
\end{proposition}

\begin{proof}[Proof of Propositions~\ref{prop:pretrain-certify} and~\ref{thm:training-inherit}]\label{app:proof-prop:pretrain-certify}
Sums over strings run over $\Ac^n$.
\emph{(i). Grouping.} For $\lambda=1$ the claim is the identity
$R_n(q;P)-R_n(P;P)=\psi=\log(1/\tlo(1,\psi))$. Let $\lambda<1$. On
$\mathrm{supp}\,P_n$ put $h:=\lambda\nu_n/P_n\in[0,1]$,
$a:=\sum_xq(x)h(x)$, $b:=\sum_xq(x)(1-h(x))$ (positive, else
$\psi=\infty$; $a+b\le1$), $Q_1:=qh/a$, $Q_2:=q(1-h)/b$ and
$\nu':=P_n(1-h)/(1-\lambda)$. The joint laws $\Pi(1,x):=\lambda\nu_n$,
$\Pi(2,x):=P_n(1-h)$ and $\hat Q$ ($P_n$ replaced by $q$) share the
conditional of the index, so $\dl{\Pi}{\hat Q}=\psi$; grouping by the
index,
$\psi=\lambda\log\frac\lambda a+\lambda\dl{\nu_n}{Q_1}
+(1-\lambda)\log\frac{1-\lambda}b+(1-\lambda)\dl{\nu'}{Q_2}$.
As $\dl{\nu_n}{Q_1}=\E_\nu\log(P_n/q_n)+\log(a/\lambda)$,
$\Delta:=R_n(q;\nu)-R_n(P;\nu)=\log(\lambda/a)+\dl{\nu_n}{Q_1}$. If
$\Delta>0$, $y:=\lambda2^{-\Delta}$ has $a\ge y$, $b\le1-y$ and
$\psi\ge\lambda\Delta+(1-\lambda)\log\frac{1-\lambda}{1-y}=\db\lambda y$,
so $y\ge\tlo(\lambda,\psi)$.

\emph{Envelopes.} For $t:=\tlo<\lambda$ ($\db\lambda t=\psi$), dropping
$(1-\lambda)\log\frac{1-\lambda}{1-t}\ge-(1-\lambda)\log\frac1{1-\lambda}$
gives the first envelope; $\db\lambda t=\int_t^\lambda\frac{\lambda-v}{v(1-v)}dv\ge\frac{(\lambda-t)^2}{2\lambda}$
nats gives $1-t/\lambda\le u$; expanding $\db\lambda t$ and
$\log(\lambda/t)$ to second order in $\lambda-t$ gives the asymptotic.
\emph{Attainment.} The attaining $q$ of
Proposition~\ref{thm:training-inherit}(i) has $\Delta=\log(\lambda/t)$ and
$\psi_n(q)=\db\lambda t$.

\emph{Monotonicity.} $B(\lambda):=\log(\lambda/\tlo)$ vanishes for
$\psi=0$. For $\psi>0$, $\varkappa:=B\ln2$ makes
$\tlo=\lambda e^{-\varkappa}$ solve
$g(\lambda,\varkappa):=\lambda\varkappa+(1-\lambda)\ln\frac{1-\lambda}{1-\lambda e^{-\varkappa}}=\psi\ln2$
(nats). For $\lambda\in(0,1)$ and $\varkappa>0$, with
$s:=(1-\lambda)e^{-\varkappa}/(1-\lambda e^{-\varkappa})\in(0,1)$,
$\partial_\lambda g=s-1-\ln s>0$ and $\partial_\varkappa g=\lambda(1-s)>0$. As
$g(\lambda,0)=0$ and $g(\lambda,\varkappa)\to\infty$ as $\varkappa\to\infty$,
$g(\lambda,\cdot)$ increases strictly from $0$ to $\infty$, so the root
$\varkappa(\lambda)$ is unique and, by the implicit function theorem, $C^1$ with
$\varkappa'(\lambda)=-\partial_\lambda g/\partial_\varkappa g<0$; continuity covers
$\lambda=1$.

\emph{(ii).} Let $\mathcal J:=\{k:\Delta_k(q)\ge\eta\}$, $\pi_{\mathcal J}>0$.
If $\pi_{\mathcal J}=1$, $\psi=\sum_k\pi_k\Delta_k\ge\eta$ and $s^*=1$.
Otherwise $\nu:=\sum_{k\in\mathcal J}\pi_k\mu_k/\pi_{\mathcal J}$ has
$\nu_n\le P_n/\pi_{\mathcal J}$ and $\E_\nu\log(P_n/q_n)\ge\eta$, so the
proof of (i) gives $\psi\ge f(\pi_{\mathcal J})$ for
$f(z):=\db{z}{z2^{-\eta}}$, convex (joint convexity), zero at $0$ and
positive on $(0,1]$; hence $\pi_{\mathcal J}\le s^*$, and
$f(z)\ge zf'(0^+)=z\,c(\eta)$ gives $s^*\le\psi/c(\eta)$.
\emph{Sharpness.} On $A_k$, $P_n=\pi_k\mu_{k,n}$ and
$q^{\mathcal J,\eta}_n=w_k\mu_{k,n}$ (weights summing to one), so
$R_n(q^{\mathcal J,\eta};\mu_k)=\log(1/w_k)$ and
$\psi_n=\sum_k\pi_k\log(\pi_k/w_k)=f(\pi_{\mathcal J})$. If
$f(z_0)=z_0c(\eta)$ with $z_0>0$, $f$ would be linear on $[0,z_0]$, hence,
being real-analytic on $[0,1)$, linear on $[0,1)$, and by continuity
$f(1)=c(\eta)$, contradicting $f(1)=\eta>c(\eta)$.
\emph{Floor.} Data processing on $\{A_k,A_k^c\}$ gives
$\db{\pi_k}{q_n(A_k)}\le\psi$, so
$R_n(q;\mu_k)\ge\log(1/q_n(A_k))\ge\log(1/\thi(\pi_k,\psi))$.

\emph{(iii).} Data processing on $\{E,E^c\}$ gives
$\db{P_n(E)}{q_n(E)}\le\psi$ and $R_n(q;\nu)\ge\db{\nu_n(E)}{q_n(E)}$, with
$q_n(E)\le\thi\le\nu_n(E)$ and $\db{\nu_n(E)}\cdot$ decreasing on
$[0,\nu_n(E)]$; regret is nondecreasing in $n$. For disjoint supports and
$E:=\mathrm{supp}\,\nu_n$, $R_n(q;\nu)\ge\log(1/q_n(E))$,
$\psi\ge\log(1/q_n(E^c))$ and $q_n(E)+q_n(E^c)\le1$; equality forces
$q_n|_E=q_n(E)\nu_n$, $q_n|_{E^c}=q_n(E^c)P_n$ and full mass. Finally
$1-2^{-\psi}\le\psi\ln2$ and $-\log(1-x)\ge x/\ln2$.
\end{proof}

\begin{corollary}[Lower pretraining loss along a sequence]
\label{cor:training-certify}
Let $\psi_j:=\psi_n(q^{(j)})\to0$.
\textup{(i)}~$R_n(q^{(j)};\mu_k)\to R_n(P;\mu_k)$ for each $k$; if
$\pi_{\min}:=\inf_k\pi_k>0$, then
$\sup_k\Delta_k(q^{(j)})\le\log(\pi_{\min}/\tlo(\pi_{\min},\psi_j))\to0$;
if the family is singular at horizon $n$ and $\inf_k\pi_k=0$, then for all
$\psi,\eta>0$ some $k$ and $q$ with $\psi_n(q)\le\psi$ have
$\Delta_k(q)=\eta$.
\textup{(ii)}~On a family singular at horizon $n$, $q^{\mathcal J,\eta}$
has larger pretraining loss than $P$ and smaller regret on every task
outside $\mathcal J$.
\textup{(iii)}~If $\nu_n\perp P_n$, then
$R_n(q^{(j)};\nu)\ge\log\frac1{1-2^{-\psi_j}}\to\infty$.
\end{corollary}

\begin{proof}[Proof of Corollary~\ref{cor:training-certify}]
(i) $\psi_j\ge\log(1/q^{(j)}_n(\Ac^n))$, so $q^{(j)}_n(\Ac^n)\to1$ and, by
Pinsker, $q^{(j)}_n\to P_n$ pointwise; the finitely many terms of
$\dl{\mu_{k,n}}{q^{(j)}_n}$ converge since $\mu_{k,n}\le P_n/\pi_k$. The
uniform bound is Proposition~\ref{prop:pretrain-certify}(i), monotone in
$\lambda$. For the last claim, $f(s)\le s\eta$ by convexity; take
$\pi_k\le\psi/\eta$ and $q=q^{\{k\},\eta}$.
(ii) $\psi_n(q^{\mathcal J,\eta})>0$, and the weights off $\mathcal J$
exceed $\pi_k$. (iii) Proposition~\ref{thm:training-inherit}(iii).
\end{proof}

\begin{proposition}[Excess-loss ledger]
\label{prop:training-ledger}
Set $\pi(k\mid x):=\pi_k\mu_k(x)/P(x)$ if $P_n(x)>0$ and $0$ otherwise,
$v_k(q):=\sum_{x\in\Ac^n}q(x)\pi(k\mid x)$ and
$Q_k:=q\,\pi(k\mid\cdot)/v_k(q)$. If $\psi_n(q)<\infty$, every $v_k(q)>0$,
$\sum_kv_k(q)=q_n(\mathrm{supp}\,P_n)$,
$\psi_n(q)=\sum_k\pi_k\log\frac{\pi_k}{v_k(q)}+\sum_k\pi_k\dl{\mu_{k,n}}{Q_k}$
and $R_n(q;\mu_k)\le\log\frac1{v_k(q)}+\dl{\mu_{k,n}}{Q_k}$, with equality
if $\mu_{k,n}$ and $\sum_{j\ne k}\pi_j\mu_{j,n}$ have disjoint supports.
The identity holds for any convention off $\mathrm{supp}\,P_n$; only the
split depends on it.
\end{proposition}

\begin{proof}[Proof of Proposition~\ref{prop:training-ledger}]
$\Pi(k,x):=\pi_k\mu_k(x)$ and $\hat Q(k,x):=q(x)\pi(k\mid x)$ share the
conditional $\pi(k\mid x)$ on $\mathrm{supp}\,P_n$, so
$\dl{\Pi}{\hat Q}=\psi_n(q)$; group by $k$, adding $(b-a)\log e$ to each
term $a\log(a/b)$ so that every term is nonnegative and the series may be
regrouped, then subtracting the additions again (they sum to
$(q_n(\mathrm{supp}\,P_n)-1)\log e$). If $v_k(q)=0$, $q=0$ on
$\mathrm{supp}\,\mu_{k,n}$ and $\psi=\infty$. The inequality uses
$q\ge v_k(q)Q_k$; under the support condition $\pi(k\mid x)=1$ on
$\mathrm{supp}\,\mu_{k,n}$.
\end{proof}

\begin{corollary}[Data mixtures]
\label{cor:training-mixture}
Let $(\mu_k)_{k\in J}$ be singular at horizon $n$ and $P_w:=\sum_kw_k\mu_k$
with $w_k>0$, $\sum_kw_k=1$.
\textup{(i)}~$R_n(P_w;\mu_k)=\log(1/w_k)$, the Kraft sum is $1$, and
$\psi_n(P_w)=\dl\pi w$ against a deployment mixture $\pi$; for overlapping
sources the first and last hold only with $\le$.
\textup{(ii)}~Raising $w_i$ to $\tilde w_i$ and rescaling the others by
$(1-\tilde w_i)/(1-w_i)$ changes regret by $-\log(\tilde w_i/w_i)$ on
source $i$ and by $\log\frac{1-w_i}{1-\tilde w_i}$ on each other source;
doubling $w_i<\frac12$ buys one bit for
$\log\frac{1-w_i}{1-2w_i}\approx w_i\log e$ bits on each other source.
\textup{(iii)}~For finite $J$, $\min_q\max_kR_n(q;\mu_k)=\log|J|$, attained
among mixtures only by uniform $w$.
\textup{(iv)}~For $m$ deterministic documents, distinct at horizon $n$,
with multiplicities $d_i$ and $d:=\sum_id_i$, the corpus-optimal predictor
has loss $\log(d/d_i)$ on document $i$ and infinite loss on every other
length-$n$ string; $\ell$ more copies of document $i$ lower its loss by
$\log\frac{(d_i+\ell)d}{d_i(d+\ell)}$ and raise each other document's by
$\log\frac{d+\ell}d$; deduplication is the minimax point of (iii).
\end{corollary}

\begin{proof}[Proof of Corollary~\ref{cor:training-mixture}]
(i) On $A_k$, $P_w=w_k\mu_k$; for overlapping sources use $P_w\ge w_k\mu_k$
and joint convexity. (ii), (iv) Substitute into (i). (iii) By
Fact~\ref{lem:antichain} on the disjoint $A_k\subseteq\Ac^n$,
$\sum_k2^{-R_n(q;\mu_k)}\le\sum_kq_n(A_k)\le1$, so the maximum is at least
$\log|J|$, with equality iff all regrets equal $\log|J|$.
\end{proof}

\begin{remark}[Post-training at a fixed prompt; after~\cite{korbak2022rl,rafailov2023dpo}]
\label{rem:posttrain}
\label{thm:training-posttrain}
Let $p$ be the pretrained law on $\Ac^n$ given a fixed prompt, $r$ a
reward, $\beta>0$ and $r':=r/(\beta\ln2)$. The maximiser of
$\E_{q'}r-\beta\,\mathrm{KL}_{\mathrm{nats}}(q'\|p)$ is $q:=p\,2^{r'}/Z$,
$Z:=\E_p2^{r'}$ (Donsker--Varadhan). Write
$b_{\mathrm{KL}}:=\dl qp$, $\tilde b_{\mathrm{KL}}:=\dl pq$ and
$\theta(\nu):=R_n(q;\nu)-R_n(p;\nu)$ for $\nu\ll p$.
\textup{(a)}~$\theta(\nu)=\log Z-\E_\nu r'=\E_qr'-\E_\nu r'-b_{\mathrm{KL}}
=\tilde b_{\mathrm{KL}}+\E_pr'-\E_\nu r'$ and
$b_{\mathrm{KL}}+\tilde b_{\mathrm{KL}}=\E_qr'-\E_pr'$.
\textup{(b)}~With the R\'enyi divergence of order $\infty$,
$D_\infty(p\|q):=\log\max_{\mathrm{supp}\,p}(p/q)$,
$\sup_\nu\theta(\nu)=D_\infty(p\|q)=\log Z-\min_{\mathrm{supp}\,p}r'$
(at a point mass), $\inf_\nu\theta(\nu)=-D_\infty(q\|p)$; if $r'$ has range $\Delta_r$,
$-\Delta_r\le\theta(\nu)\le\Delta_r-b_{\mathrm{KL}}$.
\textup{(c)}~For $b_0>0$, $0<p(A)\le1-2^{-b_0}$ and $L>0$, the reward
$r'=-L\mathbf 1_A$ has $b_{\mathrm{KL}}\le b_0$ but
$\tilde b_{\mathrm{KL}}\ge L\,p(A)-b_0$ and $\theta(\nu)\ge L\,\nu(A)-b_0$.
\textup{(d)}~Let $p=\sum_kw_k\mu_{k,n}$, singular at horizon $n$,
all $w_k>0$; put $Z_k:=\E_{\mu_k}2^{r'}$, $w'_k:=w_kZ_k/Z$,
$g_k:=\log Z_k-\E_{\mu_k}r'\ge0$. Then $R_n(q;\mu_k)=\log\frac1{w'_k}+g_k$,
$\tilde b_{\mathrm{KL}}=\dl w{w'}+\sum_kw_kg_k$ and
$\sum_k2^{-R_n(q;\mu_k)}=1-\sum_kw'_k(1-2^{-g_k})$. So $q$ attains the
finite-horizon Kraft bound with equality iff $r$ is $\mu_{k,n}$-a.s.\
constant for every $k$, and then $q=\sum_kw'_k\mu_{k,n}$; for large
$\beta$, $g_k=\mathrm{Var}_{\mu_k}(r)/(2\beta^2\ln2)+O(\beta^{-3})$.
\textup{(e)}~\textup{(Coherence across prompts.)} If one reward $r'$ on
$\Ac^n$ (complete strings) and one $\beta$ are used at every prompt $x$,
the optimiser over continuations of $x$, with reference $p(\cdot\mid x)$,
is $Q(\cdot\mid x)$ for the joint tilt $Q:=p\,2^{r'}/Z$, so incoherence
needs a reward $r$ or coefficient $\beta$ depending on the prompt/response boundary. For two uniform
bits $yz$ and one $\beta=1$, the response-only reward with
$s(yz)=\mathbf 1\{y=z\}$ (empty prompt) and $s(z)=\mathbf 1\{z=1\}$
(prompt $y$) gives
$q(z=0\mid y=0)=e/(1+e)$ under the first tilt and $1/(1+e)$ under the
second.
\end{remark}

\begin{proof}[Proof of Remark~\ref{rem:posttrain}]\label{app:proof-rem:posttrain}
(a) On $\mathrm{supp}\,p$, $\log(p/q)=\log Z-r'$; take expectations under
$\nu$, $q$, $p$ and eliminate $\log Z$.
(b) $\theta(\nu)$ averages $\log Z-r'$ over $\nu$, and
$\min r'\le\log Z\le\max r'-b_{\mathrm{KL}}$ by (a).
(c) $Z=1-p(A)+p(A)2^{-L}\ge2^{-b_0}$, so
$b_{\mathrm{KL}}=\log\frac1Z-L\,q(A)\le b_0$,
$\tilde b_{\mathrm{KL}}=\log Z+L\,p(A)$ and $\theta(\nu)=\log Z+L\,\nu(A)$.
(d) On $A_k$, $q=w_k\mu_k2^{r'}/Z$, giving $R_n(q;\mu_k)$; $Z=\sum_kw_kZ_k$
gives $\sum_kw'_k=1$, the Kraft sum $\sum_kw'_k2^{-g_k}$ and
$\tilde b_{\mathrm{KL}}=\sum_kw_k(\log\frac Z{Z_k}+g_k)$. The Kraft sum is
at most $\sum_kq(A_k)\le1$ (Fact~\ref{lem:antichain};
cf.\ Corollary~\ref{thm:profile}) and, as all $w'_k>0$, equals $1$ iff
every $g_k=0$, i.e.\ (strict convexity of $t\mapsto2^t$) iff $r'$ is
$\mu_{k,n}$-a.s.\ constant. The expansion is the second cumulant of
$X=r/\beta$, in bits.
(e) $Q(xy)=p(x)p(y\mid x)2^{r'(xy)}/Z$, so
$Q(y\mid x)=p(y\mid x)2^{r'(xy)}/\sum_{y'}p(y'\mid x)2^{r'(xy')}$, the
Gibbs maximiser at prompt $x$. The example follows from $q\propto p\,e^{s}$.
\end{proof}

\begin{proposition}[Maximum-likelihood refitting on generated data]
\label{prop:training-collapse}
Let $(\mu_k)_{k\in J}$, $2\le|J|<\infty$, be known and singular at horizon
$n$, $q_0:=\sum_k\pi^{(0)}_k\mu_k$ with all $\pi^{(0)}_k>0$, and let
generation $t+1$ refit the weights by maximum likelihood to $N$ fresh
sequences from $q_t$:
$\pi^{(t+1)}=(N_k/N)_k$, $(N_k)\sim\mathrm{Mult}(N,\pi^{(t)})$.
\textup{(i)}~Every $q_t$ has Kraft sum $1$.
\textup{(ii)}~$\E\pi^{(t)}_k=\pi^{(0)}_k$ and
$\E\sum_k\pi^{(t)}_k(1-\pi^{(t)}_k)=(1-\tfrac1N)^t\sum_k\pi^{(0)}_k(1-\pi^{(0)}_k)$.
\textup{(iii)}~$\Pr[R_n(q_1;\mu_k)=\infty]=(1-\pi^{(0)}_k)^N$, so
$\E R_n(q_t;\mu_k)=\infty$ for $t\ge1$.
\textup{(iv)}~Almost surely $\pi^{(t)}=e_\kappa$ from some finite $t_0$ on,
with $\Pr[\kappa=k]=\pi^{(0)}_k$; the profile is then a vertex and
$\psi_n(q_t)=\infty$ against $P:=q_0$ (as against any $P$ with all
$\pi_k>0$).
\textup{(v)}~With the Dirichlet$(\frac12)$ posterior-mean refit
$\pi^{(t+1)}_k=(N_k+\frac12)/(N+\frac{|J|}2)$, every
$R_n(q_t;\mu_k)\le\log(2N+|J|)$ for $t\ge1$, and the chain has no absorbing state.
\end{proposition}

\begin{proof}[Proof of Proposition~\ref{prop:training-collapse}]
(i) Corollary~\ref{cor:training-mixture}(i).
(ii) Given $\pi^{(t)}$, $N_k/N$ has mean $\pi^{(t)}_k$ and variance
$\pi^{(t)}_k(1-\pi^{(t)}_k)/N$.
(iii) $N_k=0$ with that probability, and $0$ is absorbing.
(iv) For the $[0,1]$-valued martingale $\pi^{(t)}_k$, orthogonal increments
give
$\E\sum_t\pi^{(t)}_k(1-\pi^{(t)}_k)=N\sum_t\E(\pi^{(t+1)}_k-\pi^{(t)}_k)^2\le N$,
so the sum is a.s.\ finite; for $N\ge2$, on the lattice
$\{0,\frac1N,\dots,1\}$ each nonzero term is at least $(N-1)/N^2$, so
$\pi^{(t)}_k\in\{0,1\}$ (absorbing) eventually (at $N=1$ the first resample
lands on a vertex, so the conclusion is immediate), and bounded convergence gives
$\Pr[\pi^{(\infty)}_k=1]=\pi^{(0)}_k$; then $q_t=\mu_\kappa$ gives no mass
to the supports of the $\mu_{j,n}$, $j\ne\kappa$, which $P$ charges, so
$\psi_n(q_t)=\infty$.
(v) For $t\ge1$, $N_k\ge0$ gives the weight floor $1/(2N+|J|)$, and positive weights
make every count vector reachable from every state.
\end{proof}

\providecommand{\dMem}{\mathrm{Mem}}
\providecommand{\dPos}{\mathrm{Pos}}
\providecommand{\dExc}{\mathcal{E}}
\providecommand{\dwin}{\mathrm{win}}
\providecommand{\dcat}{\mathrm{cat}}

\subsection{Post-training, data mixing and calibration}
\label{app:posttrain-main}

\paragraph{Post-training reallocates.}
At a fixed prompt and horizon, the exact maximiser of a KL-regularised
reward objective is a Gibbs tilt of the pretrained
law~\cite{korbak2022rl,rafailov2023dpo}. By the post-training identity
(Remark~\ref{rem:posttrain}), the tilt changes the regret on each
environment by an affine function of that environment's mean reward, a
price no KL budget bounds, and it is a change of prior over a task
family, a move along the Kraft frontier at horizon $n$, iff the reward
is almost surely constant within every task (for tasks singular at horizon $n$). Practical PPO and DPO runs
do not reach this optimiser. If one reward on complete strings (prompt
and response together) and one coefficient $\beta$ are used at every
prompt, the exact per-prompt optimisers are conditionals of one joint
tilt. The objective itself, and not just optimisation and serving, can
make a post-trained model incoherent only when the reward or the
coefficient depends on where the prompt/response boundary falls: a reward
that scores the response alone, or a prompt-dependent $\beta_x$ with one
complete-string reward. Best-of-$n$ sampling and reward-model reranking
are tilts by a monotone transform of the reward (for best-of-$n$ with a
continuous reward the density ratio to the base law is $nF(r)^{n-1}$,
$F$ the reward's distribution function under the base law), so the
affine identity applies with the transformed reward in place of $r$.
Best-of-$n$ stays within $\ln n-(n-1)/n$ nats of the base law in
KL~\cite{beirami2024bestofn}, and as a proxy reward is optimised the gold
reward eventually falls, with a KL dependence fitted by simple
functional forms~\cite{gao2023overoptimization}.

\paragraph{Data mixing and refitting.}
For sources singular at the horizon (disjoint supports over the window),
doubling a source of weight $w<\frac12$ buys exactly one bit on it and
costs exactly $\log\frac{1-w}{1-2w}$ bits on each other source, about $w\log e$ for
small $w$; repetition is the
extreme case of up-weighting (the data-mixture corollary,
Corollary~\ref{cor:training-mixture}). For overlapping sources $\log(1/w)$
is only an upper bound. About four epochs of repeated data are nearly as useful as
fresh data, and further repetition is worth ever
less~\cite{muennighoff2023data}. Maximum-likelihood refitting on generated
data keeps the Kraft sum at one while the allocation drifts to a vertex
(Proposition~\ref{prop:training-collapse}, for known components singular at the
horizon and refits of the weights only), the absorbing chain noted by
Shumailov et al.~\cite{shumailov2024collapse}.

\paragraph{Calibration after post-training.}
The visible cost of reallocation is calibration. The GPT-4 base model is
well calibrated on a multiple-choice benchmark, while the post-trained
model is markedly less so~\cite{openai2023gpt4}. Kadavath et al.\ find
RLHF policies miscalibrated until the temperature is raised to about
$2.5$~\cite{kadavath2022know}, and output diversity falls after
RLHF~\cite{kirk2024rlhf}. Reinforcement learning with verifiable rewards
raises pass@1, while the base model reaches higher pass@$k$ at large
$k$~\cite{yue2025rl}: probability moves onto rewarded solutions and away from
others the base model could sample, a direct instance of reallocation.
By Corollary~\ref{cor:exchange}, if post-training gains $b$ bits on every
rewarded solution in a set $A$ of base mass $q(A)$, the post-trained
model keeps at most $1-2^{b}q(A)$ of its mass elsewhere: for $q(A)=0.2$
and $b=2$, $0.2$ against $0.8$ before. This is an illustration, consistent in
direction with the pass@$k$ loss at large $k$ rather than derived from it. Ouyang et al.\ report a benchmark ``alignment
tax''~\cite{ouyang2022}, and Bai et al.\ helpfulness--harmlessness
trade-offs~\cite{bai2022hh}. None of this is
anomalous under the affine price, and calibration under
the model's own law (Section~\ref{rem:calibration}) is calibration
against a reallocated law. Reading a served post-trained model therefore
reads a post-trained allocation, not the pretraining prior, and a report
should state which kind of checkpoint it reads (Paper~B reads post-trained
and base checkpoints; Section~\ref{sec:deployed}).

\section{The deployed regime}
\label{app:deployed}

\subsection{Context windows}
\label{app:deployed-window}

\begin{definition}[Memory and position terms]
\label{def:deployed-window}
With window predictors as in Section~\ref{sec:deployed}, a source $\mu$ and
$t\ge1$, put $Y_t:=\dwin_W(X_{<t})$ and $Y^{+}_t:=(t,Y_t)$, and let
$\lambda$ be uniform on $\{1,\dots,n\}$, independent of $X\sim\mu$. Then
$\dMem^{W}_{n}(\mu):=\sum_{t=1}^{n} I_\mu(X_t;X_{<t}\mid Y^{+}_t)$ and
$\dPos^{W}_{n}(\mu):=n\,I(\lambda;X_\lambda\mid Y_\lambda)$. For stationary
$\mu$, $h_\mu(j)$ is nonincreasing in $j$ with limit $h_\mu$, and
$\dExc(\mu)\in[0,\infty]$.
\end{definition}

The summand $I(X_t;X_{<t}\mid Y^{+}_t)=H(X_t\mid Y^{+}_t)-H(X_t\mid X_{<t})$
vanishes for $t\le W+1$, where $Y_t=X_{<t}$.

\begin{proposition}[Window complements]
\label{thm:window}
Let $n\ge1$ and $W,\mu_j,w,\xi_w,\kappa$ be as in
Proposition~\ref{prop:window-tax}.
\begin{enumerate}
\item[\textup{(i)}] \textup{(Exact floor.)} For every $\mu$ and
time-aware window-$W$ predictor $q$,
\[
R_n(q;\mu)=\dMem^{W}_{n}(\mu)+\sum_{t\le n}\E_\mu\,
\dl{\mu(X_t\in\cdot\mid Y^{+}_t)}{q(\cdot\mid Y^{+}_t)} .
\]
The least regret
is $\dMem^{W}_{n}(\mu)$ over time-aware window-$W$ predictors and
$\dMem^{W}_{n}(\mu)+\dPos^{W}_{n}(\mu)$ over window-$W$ predictors,
attained by $\mu(X_t=a\mid Y^{+}_t=(t,s))$ and by
$\Pr(X_\lambda=a\mid Y_\lambda=s)$.
\item[\textup{(ii)}] \textup{(Mixture identity.)} For every predictor $q$,
$\sum_j w_jR_n(q;\mu_j)=I_w(\kappa;X_{1:n})+R_n(q;\xi_w)$, with
$I_w(\kappa;X_{1:n})\le\min\{H(w),n\log V\}$.
\item[\textup{(iii)}] \textup{(Rate.)} If every $\mu_j$ with $w_j>0$ is
stationary, then $\dMem^{W}_{n}(\xi_w)\le(n-W-1)^{+}r_W(w)$ and
$\dMem^{W}_{n}(\xi_w)/n\to r_W(w):=h_{\xi_w}(W)-h_{\xi_w}$, with
$(W+1)\,r_W(w)\le\dExc(\xi_w)\le H(w)+\sum_jw_j\dExc(\mu_j)$. If moreover
$H(w)<\infty$, then
$r_W(w)=I_{\xi_w}(\kappa;X_{W+1}\mid X_{1:W})+\sum_jw_j\bigl(h_{\mu_j}(W)-h_{\mu_j}\bigr)$,
and $r_W(w)=0$ iff there is a kernel $\beta$ with
$\mu_j(X_t\in\cdot\mid X_{<t})=\beta(\cdot\mid X_{t-W:t-1})$ $\mu_j$-a.s.\
for all $j$ with $w_j>0$ and all $t>W$, and then some window-$W$ predictor
has $R_n(q;\mu_j)\le W\log V$ for all such $j$ and all $n$.
\item[\textup{(iv)}] \textup{(Copy witnesses.)} If $W<\ell$, then
$\dMem^{W}_{n}(\omega_\ell)=(n-\ell)^{+}\log V$, with equality in
$R_n(q;\omega_\ell)\ge(n-\ell)^{+}\log V$ for the uniform predictor. The
mixture $\xi_v:=\sum_{\ell\ge1}v_\ell\,\omega_\ell$ with
$v_\ell:=\frac1{\log(\ell+1)}-\frac1{\log(\ell+2)}$ is computable in
polynomial time per step and has
$R_\infty(\xi_v;\omega_\ell)\le\log\ell+2\log\log\ell+O(1)$.
\end{enumerate}
\end{proposition}

\begin{proof}[Proof of Propositions~\ref{prop:window-tax} and~\ref{thm:window}]
\label{app:proof-prop:window-tax}
\emph{Exact floor.} Fix $t$ and write $\mu_t(a\mid Y^{+}_t):=\mu(X_t=a\mid Y^{+}_t)$;
$Y^{+}_t$ is countably valued. If $q(a\mid Y^{+}_t)=0<\mu_t(a\mid Y^{+}_t)$
with positive probability, both sides are $+\infty$ (the memory summand is
at most $\log V$). Otherwise split
$\log\frac{\mu(a\mid X_{<t})}{q(a\mid Y^{+}_t)}=\log\frac{\mu(a\mid X_{<t})}{\mu_t(a\mid Y^{+}_t)}+\log\frac{\mu_t(a\mid Y^{+}_t)}{q(a\mid Y^{+}_t)}$
under $\E\sum_a\mu(a\mid X_{<t})$ (negative parts are integrable as
$u|\log u|\le(e\ln2)^{-1}$). The first term gives $I(X_t;X_{<t}\mid Y^{+}_t)$;
as $\E[\mu(a\mid X_{<t})\mid Y^{+}_t]=\mu_t(a\mid Y^{+}_t)$, the tower
property turns the second into $\E\,\dl{\mu_t}{q(\cdot\mid Y^{+}_t)}\ge0$, zero iff
$q=\mu_t$ a.s. Summing with \eqref{eq:si-regret} gives
Proposition~\ref{thm:window}(i) for time-aware predictors. For window-$W$
predictors the mismatch sum is
$n\,\E\,\dl{\Pr(X_\lambda\in\cdot\mid\lambda,Y_\lambda)}{q(\cdot\mid Y_\lambda)}$,
and the same split with $Y_\lambda$ for $Y^{+}_t$ yields
$n\,I(\lambda;X_\lambda\mid Y_\lambda)$ plus a nonnegative mismatch.

\emph{Mixture identity.} For $x\in\Ac^{n}$ with $\xi_w(x)>0$, multiply
$\log\frac{\mu_j(x)}{q(x)}=\log\frac{\mu_j(x)}{\xi_w(x)}+\log\frac{\xi_w(x)}{q(x)}$
by $w_j\mu_j(x)$ and sum: this gives $I_w(\kappa;X_{1:n})$ (absolutely
convergent, negative terms at most $w_j/(e\ln2)$) plus $R_n(q;\xi_w)$, both
sides being infinite if $q(x)=0<\xi_w(x)$. And
$I\le\min\{H(\kappa),H(X_{1:n})\}$.

\emph{Proposition~\ref{prop:window-tax}(i).} By the exact floor for $\xi_w$
in the mixture identity, a time-aware window-$W$ predictor has average
regret at least $I_w+\dMem^{W}_{n}(\xi_w)$, with equality for the
time-aware window conditionals of $\xi_w$; unconstrained, $q=\xi_w$ attains
$I_w\le H(w)$.

\emph{Stationary cap.} Drop components with $w_j=0$; a mixture of stationary
measures is stationary. For stationary $P$ and $t\ge W+1$,
$H_P(X_t\mid Y^{+}_t)=h_P(W)$ and $H_P(X_t\mid X_{<t})=h_P(t-1)$, so
$\dMem^{W}_{n}(P)=\sum_{t=W+2}^{n}(h_P(W)-h_P(t-1))$, with nonnegative,
nondecreasing summands converging to $h_P(W)-h_P$, whence the bound and
limit in Proposition~\ref{thm:window}(iii). The gaps $h_P(j)-h_P$ are
nonincreasing, so $(W+1)(h_P(W)-h_P)\le\dExc(P)$. By concavity and the
chain rule,
$\sum_jw_jH_{\mu_j}(X_{1:n})\le H_{\xi_w}(X_{1:n})\le H(w)+\sum_jw_jH_{\mu_j}(X_{1:n})$,
so $h_{\xi_w}\ge\sum_jw_jh_{\mu_j}$; with
$\dExc(P)=\lim_n(H_P(X_{1:n})-nh_P)$ and monotone convergence,
$\dExc(\xi_w)\le H(w)+\sum_jw_j\dExc(\mu_j)$, the cap in
Proposition~\ref{prop:window-tax}(ii).

\emph{Dichotomy.} Let $H(w)<\infty$. Dividing the entropy bounds by $n$
gives $h_{\xi_w}=\sum_jw_jh_{\mu_j}$, and the chain rule gives
$h_{\xi_w}(W)=\sum_jw_jh_{\mu_j}(W)+I(\kappa;X_{W+1}\mid X_{1:W})$: the
formula for $r_W(w)$, with two nonnegative terms. A stationary $\mu_j$
has $h_{\mu_j}(W)=h_{\mu_j}$ iff it is $W$-Markov (equality forces
$I(X_{t+1};X_{1:t-W}\mid X_{t-W+1:t})=0$ for $t\ge W$), and the information
term vanishes iff all $\mu_j$ with $w_j>0$ charging a window $s$ share the
law of $X_{W+1}$ given $X_{1:W}=s$. So $r_W(w)=0$ iff a common kernel
exists; then every summand of $\dMem^{W}_{n}$ vanishes, and the predictor
equal to $\beta(\cdot\mid s)$ on charged windows (uniform on shorter
contexts) pays at most $W\log V$. Otherwise $\dMem^{W}_{n}(\xi_w)=r_W(w)n+o(n)$ with
$r_W(w)>0$.

\emph{Copy witness.} The shift of $U^{\infty}$ is $(\varsigma U)^{\infty}$
with $\varsigma$ cyclic rotation, so $\omega_\ell$ is stationary;
$\omega_\ell(x)=V^{-\min(|x|,\ell)}$ if $x$ has period $\ell$ and $0$
otherwise, computable from $\ell$. $H(X_{1:n})=\min(n,\ell)\log V$ and
$h=0$ give $\dExc(\omega_\ell)=\ell\log V$. For $t\le\ell$, $X_t=U_t$ is
independent of $X_{<t}$. For $t>\ell$, $X_t=U_{i(t)}$ with
$i(t):=((t-1)\bmod\ell)+1$, so $H(X_t\mid X_{<t})=0$, while the window holds
$U_{i(t-1)},\dots,U_{i(t-W)}$, none equal to $U_{i(t)}$ as $W<\ell$; so $X_t$ is
uniform given $Y^{+}_t$ and the summand is $\log V$. The exact floor gives
Proposition~\ref{thm:window}(iv), the uniform predictor having zero
mismatch, and, at $\ell=W+1$ (cap $(n-W-1)^{+}\log V$ as $H(w)=0$),
Proposition~\ref{prop:window-tax}(iii); the bound for $M$ is
Fact~\ref{fact:si-solomonoff}. As a uniform mixture of the $V^{W+1}$
sequences $u^\infty$, $\omega_{W+1}$ has mixture information $(W+1)\log V$
for $n\ge W+1$. For $\xi_v$: the weights telescope to $1/\log 2=1$;
$v_\ell\ge((\ell+2)\ln2\,\log^{2}(\ell+2))^{-1}$ and
Fact~\ref{thm:attainable} gives $R_\infty(\xi_v;\omega_\ell)\le\log(1/v_\ell)$;
and $\xi_v(x)=\sum_{\ell\le|x|,\ x\text{ has period }\ell}v_\ell V^{-\ell}+V^{-|x|}/\log(|x|+2)$
needs $O(|x|^{2})$ comparisons.
\end{proof}

\paragraph{Stationarity is needed.}
With $W=1$, $V=2$, $X_1,X_2$ uniform, $X_3:=X_1$ and later symbols i.i.d.\
uniform, the source is not $1$-Markov, yet its memory term is $1$ bit at
every $n\ge3$, so the dichotomy fails. The cap fails on the same source, and
for the reason the proof identifies: its conditional entropies $h(j)$ run
$1,1,0,1,1,\dots$, so the gaps $h(j)-h$ are not nonincreasing, the formal
excess entropy is negative, and the right-hand side of (ii) falls below the
memory term. Stationarity is therefore load-bearing for both halves of (ii),
not only for the dichotomy.

\paragraph{Bounded recurrent state.}
\label{rem:deployed-state}
If the conditional at time $t$ depends on the past only through a
recurrently updated state of $b$ bits, then
$R_n(q;\omega_\ell)\ge\lfloor(n-\ell)/\ell\rfloor\,(\ell\log V-b)^{+}$, the
log-loss form of the copying bound of~\cite{jelassi2024repeat}. On
$\ell$ consecutive steps after time $\ell$ the truth's loss is $0$, and the
predictor's expected loss is at least the entropy of the block given the
state at their start, which is at least $\ell\log V-b$. A window is a
state of $W\log V$ bits, for which this gives $\log V$ per period rather
than per step.

\subsection{Temperature}
\label{app:deployed-temperature}

Take $q_T$ as in Section~\ref{sec:deployed}, uniform if
$\sum_bq(b\mid x)^{1/T}=0$. For $z\in\Ac^\infty$ put
$p_t:=q(z_t\mid z_{<t})$ and $\eta_t:=\sum_{a\ne z_t}q(a\mid z_{<t})$, and let $\bar a_t\ne z_t$ be arbitrary.

\begin{theorem}[Temperature for arbitrary predictors]
\label{thm:temperature}
Let $q$ be a predictor (a semimeasure suffices) and $z\in\Ac^\infty$ with
$p_t>0$ for all $t$.
\begin{enumerate}
\item[\textup{(i)}] \textup{(Exact loss.)}
$L_n(q_T;z)=\sum_{t\le n}\log\bigl(1+\sum_{a\ne z_t}(q(a\mid z_{<t})/p_t)^{1/T}\bigr)$.
\item[\textup{(ii)}] \textup{($T\le1$.)} $L_n(q_T;z)\le L_n(q;z)/T$.
\item[\textup{(iii)}] \textup{($T\ge1$.)}
$\sum_{t\le n}\eta_t^{1/T}\le L_n(q_T;z)\le3V\bigl(\sum_{t\le n}\eta_t^{1/T}+L_n(q;z)\bigr)$
and $\sum_{t\le n}\eta_t^{1/T}\le n^{1-1/T}(L_n(q;z)\ln2)^{1/T}$.
\item[\textup{(iv)}] \textup{(Tempered Kraft sum.)} With
$d_t:=q(z_{<t}\bar a_t)$, $\sum_td_t\le1$, $\eta_t\ge d_t$, and for
$T\ge1$, $L_n(q_T;z)\ge\sum_{t\le n}d_t^{1/T}$. So $L_n(q_T;z)\le B$ forces
$\max_{t\le n}L_t(q;z_{<t}\bar a_t)\ge T\log(n/B)$.
\item[\textup{(v)}] \textup{(I.i.d.\ sources.)} If $\mu$ is i.i.d.\ with
marginal $\theta$ and $R_\infty(q;\mu)<\infty$, then for $T\ne1$,
$\liminf_n\frac1nR_n(q_T;\mu)\ge\dl{\theta}{\theta_T}$, positive unless
$\theta$ is uniform on its support. $R_\infty<\infty$ is needed: for
$T<1$, conditionals $\propto\theta^{T}$ have $q_T=\theta$.
\end{enumerate}
\end{theorem}

\begin{proof}
(i) $-\log q_T(z_t\mid z_{<t})=\log(\sum_bq(b\mid z_{<t})^{1/T}/p_t^{1/T})$.
(ii) For $r:=1/T\ge1$, $\sum_bq(b)^{r}\le(\sum_bq(b))^{r}\le1$, so
$-\log q_T(z_t)\le-r\log p_t$.

(iii) Let $r:=1/T\le1$, $u_t:=\eta_t/p_t$, $\ell_t:=-\log q_T(z_t\mid z_{<t})$
and $c:=(V-1)^{1-r}$; subadditivity and concavity give
$\eta_t^{r}\le\sum_{a\ne z_t}q(a)^{r}\le c\,\eta_t^{r}$. Lower bound:
$\ell_t\ge\log(1+\eta_t^{r})\ge\eta_t^{r}$. Upper bound: if
$p_t\ge\frac12$, $u_t\le2\eta_t$ and
$\ell_t\le\log(1+cu_t^{r})\le2c\,\eta_t^{r}/\ln2\le3V\eta_t^{r}$; if
$p_t<\frac12$, $u_t^{r}\le1/p_t$ and
$\ell_t\le\log(1+c)+\log\frac1{p_t}\le3V\log\frac1{p_t}$. H\"older gives
$\sum_t\eta_t^{r}\le n^{1-r}(\sum_t\eta_t)^{r}$, and
$\eta_t\le1-p_t\le\ln(1/p_t)$.

(iv) $\eta_t\ge q(\bar a_t\mid z_{<t})\ge q(z_{<t}\bar a_t)=d_t$. The strings
$z_{<t}\bar a_t$ are pairwise incomparable, so Fact~\ref{lem:antichain}
gives $\sum_td_t\le1$; combine with (iii). If
$b:=\max_{t\le n}L_t(q;z_{<t}\bar a_t)$, then $n2^{-b/T}\le B$.

(v) Write $q_t:=q(\cdot\mid X_{<t})$ and $\hat q_t:=q_t/\|q_t\|_1$. Since
$\sum_t\E\,\dl{\theta}{q_t}<\infty$ and
$\dl{\theta}{q_t}=\dl{\theta}{\hat q_t}+\log\frac1{\|q_t\|_1}$ with both
terms nonnegative, Pinsker gives $\hat q_t\to\theta$ in probability, hence
$(q_t)_T=(\hat q_t)_T\to\theta_T$. Fatou and lower semicontinuity along
a.s.\ convergent subsequences give
$\liminf_t\E\,\dl{\theta}{(q_t)_T}\ge\dl{\theta}{\theta_T}$, and
Ces\`aro averaging of $R_n(q_T;\mu)=\sum_{t\le n}\E\,\dl{\theta}{(q_t)_T}$
gives the claim. Finally $\theta_T=\theta$ iff $\theta(a)^{1/T-1}$ is
constant on $\mathrm{supp}\,\theta$.
\end{proof}

\begin{proof}[Proof of Corollary~\ref{cor:temperature}]
\label{app:proof-cor:temperature}
Clause (iii) holds in the exact form: for $T>1$ and all $n$,
\[
2^{-(K(z)+c)/T}\,\Sigma_T(n) \;\le\; L_n(M_T;z) \;\le\; 3V\bigl(2^{(2K(z)+c)/T}\,\Sigma_T(n) + K(z)+c\bigr),
\]
where $c_T\,n^{1-1/T}(\log n)^{-2/T}\le\Sigma_T(n)\le n^{1-1/T}$ for
$n\ge n_T$, with $c_T>0$. Fix $c_0=c_0(U)$ such that
$M(z_{1:n})\ge2^{-K(z)-c_0}$ for all $n$ (Fact~\ref{fact:si-dominance} for
the point mass on $z$), $M(z_{<t}\bar a_t)\ge2^{-K(z)-K(t)-c_0}$ and
$f(t)\le2^{K(z)+c_0-K(t)}$ below, and put $c:=2c_0$: this one enlarged
constant serves in (ii) and in both bounds of (iii).
(i) Normalisation divides each conditional by at most $1$, so
$M_1\ge M$ on cylinders; apply Fact~\ref{fact:si-solomonoff}.
(ii) The bound on $z$ is Theorem~\ref{thm:temperature}(ii) with
$L_\infty(M;z)\le K(z)+c_0$. The linear
bound, also for $T>1$, is Theorem~\ref{thm:temperature}(v) with $q=M$, as
$R_\infty(M;\mu)\le K(\mu)+O(1)$.
(iii) \emph{Lower bound.} Take $\bar a_t$ least among symbols other than
$z_t$; $z_{<t}\bar a_t$ is computable from prefix programs for $z$
and $t$, whence the bound on $M(z_{<t}\bar a_t)$, and
Theorem~\ref{thm:temperature}(iv) gives
$L_n(M_T;z)\ge2^{-(K(z)+c_0)/T}\Sigma_T(n)$. \emph{Upper bound.}
$f(t):=\sum_{a\ne z_t}M(z_{<t}a)$ is lower semicomputable uniformly from a
program for $z$, with $\sum_tf(t)\le1$ (Fact~\ref{lem:antichain}), whence
the bound on $f(t)$ by the coding theorem~\cite[Ch.~4]{li2019kolmogorov}.
So $\eta_t=f(t)/M(z_{<t})\le2^{2K(z)+c-K(t)}$; insert this and
$L_n(M;z)\le K(z)+c_0$ into Theorem~\ref{thm:temperature}(iii).
\emph{The sum $\Sigma_T$.} The upper bound is H\"older with
$\sum_t2^{-K(t)}\le1$. For the lower bound,
$K(t)\le\log t+2\log\log t+c'$ for $t\ge2$ gives
$\Sigma_T(n)\ge2^{-c'/T}(\log n)^{-2/T}\sum_{t=2}^{n}t^{-1/T}$, and
$\sum_{t=2}^{n}t^{-1/T}\ge\frac T{T-1}((n+1)^{1-1/T}-2^{1-1/T})$, which
diverges.
\end{proof}

\paragraph{The served law under sampling.}
\label{rem:deployed-belief}
For $V=2$, $q(00)=q(01)=\frac14$, $q(10)=\frac12$ and $q(11)=0$, one has
$Q_T(1)=\frac12$, whereas the tempered joint law $q(x)^{1/T}/Z$ gives the
first symbol $1$ probability $1/(1+2^{1-1/T})$~\cite{karan2025sampling}.
Since $(cq)_T=q_T$, samples at any temperature do not see the scale of $q$,
and so not leakage (Lemma~\ref{lem:measurepart}).

\subsection{Floors, precision and truncation}
\label{app:deployed-floors}

\begin{proposition}[Floors and truncation]
\label{prop:floors}
Let $0<\phi\le1/V$ and $r_\phi:=\log\frac1{1-(V-1)\phi}$.
\begin{enumerate}
\item[\textup{(i)}] \textup{(Floor.)} If $q(a\mid x)\ge\phi$ for all $a,x$,
every single-step surprisal is at most $\log\frac1\phi$, and
$L_n(q;z)\ge n\,r_\phi$ for all $z\in\Ac^\infty$ and $n$; the coefficient
$r_\phi$ cannot be lowered.
\item[\textup{(ii)}] \textup{(Smoothing.)} For every predictor $\xi$,
$\xi^{\phi}(a\mid x):=(1-V\phi)\xi(a\mid x)+\phi$ has floor $\phi$ and
$L_n(\xi^{\phi};x)\le L_n(\xi;x)+n\,r_\phi$ for $x\in\Ac^{n}$, hence
$R_n(\xi^{\phi};\mu)\le R_n(\xi;\mu)+n\,r_\phi$. For a Bayes mixture $\xi_w$
over deterministic environments $z^{(j)}$, the loss of $\xi_w^{\phi}$ on
member $j$ lies in $[n\,r_\phi,\ n\,r_\phi+\log(1/w_j)]$.
\item[\textup{(iii)}] \textup{(Bounded logits.)} If all logits lie in
$[-B,B]$, the softmax has floor $\phi=1/(1+(V-1)e^{2B})$ when computed
without underflow, and loses at least $\log(1+(V-1)e^{-2B})$ bits per step on
every $z\in\Ac^\infty$.
\item[\textup{(iv)}] \textup{(Truncation.)} Top-$k_{\mathrm{top}}$ decoding
to fewer than $V$ symbols, or top-$p$ with threshold
$p_{\mathrm{top}}\in(0,1)$ when $V\ge1/(1-p_{\mathrm{top}})$, applied to
normalised conditionals, gives some symbol probability $0$ at every
context. The result, like any pipeline rounding probabilities to zero, has
$R_1(\cdot\,;\mu)=\infty$ for every $\mu$ with full prefix support and loss
$\infty$ at the first step of its own diagonal. With probabilities
quantised to multiples of $2^{-b}$, every step costs at most $b$ bits or
$\infty$.
\end{enumerate}
\end{proposition}

\begin{proof}
(i) $q(z_t\mid z_{<t})\le1-(V-1)\phi$; the predictor putting
$1-(V-1)\phi$ on $z_t$ attains the bound. (ii) Since $\xi\le1$,
$(1-V\phi)\xi+\phi\ge(1-(V-1)\phi)\xi$; the upper end uses
$\xi_w(z^{(j)}_{1:n})\ge w_j$ and the lower end is (i). (iii) The smallest
probability is at least $e^{-B}/(e^{-B}+(V-1)e^{B})$, and the largest at most
$e^{B}/(e^{B}+(V-1)e^{-B})$. (iv) Top-$k_{\mathrm{top}}$
is immediate. For top-$p$, the $V-1$ most likely symbols carry mass at least
$1-1/V\ge p_{\mathrm{top}}$, so the smallest set reaching $p_{\mathrm{top}}$
omits a least likely symbol. Nonzero quantised probabilities are at least
$2^{-b}$.
\end{proof}

\subsection{Tokenisation}
\label{app:deployed-tokenisation}

Let $\Sigma$ be a finite character alphabet, $\Ac\subseteq\Sigma^{+}$ a
finite vocabulary containing every character, and
$\dcat:\Ac^{*}\cup\Ac^{\infty}\to\Sigma^{*}\cup\Sigma^{\infty}$
concatenation; $\sigma^{-}$ drops the last token of $\sigma\ne\epsilon$.
For $y\ne\epsilon$, $C(y)$ is the finite set of $\sigma\ne\epsilon$ with
$\dcat(\sigma)\sqsupseteq y$ and $\dcat(\sigma^{-})\sqsubset y$ strictly;
$C(\epsilon):=\{\epsilon\}$. For a token semimeasure $p$,
$(\dcat_\#p)(y):=\sum_{\sigma\in C(y)}p(\sigma)$.

\begin{proposition}[Tokenisation]
\label{prop:tokenisation}
\begin{enumerate}
\item[\textup{(i)}] \textup{(Marginalisation.)} $\dcat_\#p$ is a semimeasure,
and a measure, or lower semicomputable, if $p$ is. $\dcat_\#$ is
monotone ($p\ge c\nu$ on token cylinders implies
$\dcat_\#p\ge c\,\dcat_\#\nu$), so dominance constants, and with them
Fact~\ref{thm:attainable}, transfer to characters.
\item[\textup{(ii)}] \textup{(Data processing.)} For measures,
$R_\infty(\dcat_\#p;\dcat_\#\nu)\le R_\infty(p;\nu)$, and likewise for
regret at length $L$.
\item[\textup{(iii)}] \textup{(Tokeniser invariance.)} For universal
semimeasures $M_{\Ac}$ on tokens and $M_\Sigma$ on characters,
$2^{-K(\Ac)-c}M_\Sigma\le\dcat_\#M_{\Ac}\le2^{K(\Ac)+c}M_\Sigma$.
\item[\textup{(iv)}] \textup{(Canonical scoring.)} Let $p$ be tight on
$\mathrm{EOS}$-closed documents, with document law
$p_{\mathrm{doc}}(\sigma):=p(\sigma\,\mathrm{EOS})$, character law
$P(y):=\sum_{\dcat(\sigma)=y}p_{\mathrm{doc}}(\sigma)$, canonical score
$s(y):=p_{\mathrm{doc}}(\tau^{*}(y))$ and non-canonical mass $\delta$. Then
$s\le P$, $\sum_ys(y)=1-\delta$,
$\log\frac{P(y)}{s(y)}=-\log p_{\mathrm{doc}}(\tau^{*}(y)\mid\dcat=y)$, whose
$P$-expectation is at least $\log\frac1{1-\delta}$, and dominance does not
transfer through $s$.
\end{enumerate}
\end{proposition}

\begin{proof}
(i) Fix $y$, a character $a$ and $\tau\in C(ya)$. If
$\dcat(\tau^{-})\sqsubset y$ strictly, then $\tau\in C(y)$ and
$\dcat(\tau)\sqsupset y$ strictly; if $\dcat(\tau^{-})=y$, then
$\tau^{-}\in C(y)$ (also if $\tau^{-}=\epsilon=y$). So each $\sigma\in C(y)$
with $\dcat(\sigma)\sqsupset y$ strictly lies in exactly one $C(ya)$, and
each with $\dcat(\sigma)=y$ contributes its one-token extensions, of total
mass at most $p(\sigma)$, equal for measures. Monotonicity is termwise, and
$C(y)$ is finite and computable.
(ii) Data processing for relative entropy under $\dcat$, with
Lemma~\ref{lem:measurepart}; at length $L$ the character cylinders are
determined by the first $L$ tokens.
(iii) Upper bound: Fact~\ref{fact:si-dominance} for the lower
semicomputable $\dcat_\#M_{\Ac}$. Lower bound: the token semimeasure
$\iota_\#M_\Sigma(\sigma):=M_\Sigma(\dcat(\sigma))$ on single-character
tokens ($0$ otherwise) satisfies $\iota_\#M_\Sigma\le2^{K(\Ac)+c}M_{\Ac}$
by Fact~\ref{fact:si-dominance}; apply the monotone $\dcat_\#$, with
$\dcat_\#\iota_\#M_\Sigma=M_\Sigma$.
(iv) The identity is Bayes' rule on the fibre $\dcat^{-1}(y)$, and the
expected gap is $\dl{P}{s}\ge\log\frac1{\|s\|_1}$. If $\nu$ puts its mass on
a non-canonical tokenisation, $s(y)=0<\dcat_\#\nu(y)$ is possible.
\end{proof}

\subsection{Truncation, tokenisation and reasoning traces}
\label{app:deployed-more}

\paragraph{Truncation is dogmatism by design.}
\label{rem:truncation}
Top-$k_{\mathrm{top}}$ truncation to fewer than $V$ symbols assigns
probability zero to some symbol at every context. So does top-$p$
truncation with threshold $p_{\mathrm{top}}$ whenever
$V\ge 1/(1-p_{\mathrm{top}})$ (Proposition~\ref{prop:floors}), and so do
min-$p$, grammar masks and rounding to zero whenever they are active.
Greedy decoding, the limit $T\to0$, is top-$1$ truncation. Read as a
predictor, a truncated sampler has infinite regret at the first step on
every source of full support and infinite loss on its own diagonal: the
extreme form of the dogmatism excluded by Definition~\ref{def:nondogmatic}.
Nucleus sampling was introduced because sampling from the unreliable tail
degrades text~\cite{holtzman2020degeneration}, Hewitt et al.\ read
truncation as undoing the model's smoothing~\cite{hewitt2022truncation},
and the softmax bottleneck forces spurious nonzero tail mass, which
truncation removes~\cite{finlayson2024degeneration}. These readings agree
with this one: truncation trades any guarantee on improbable continuations
for sample quality on probable ones; under a softmax bottleneck the model
cannot represent true near-zeros, so non-dogmatism and tail accuracy
conflict, and truncation chooses the latter~\cite{finlayson2024degeneration}.
Non-dogmatism itself is cheap: smoothing any predictor to a floor $\phi$ costs at
most $n\log\frac1{1-(V-1)\phi}\approx n(V-1)\phi\log e$ bits on any stream, and no
floor can cost less on a deterministic one (Proposition~\ref{prop:floors}). For logits soft-capped at $30$ with
$V=2^{18}$, the smallest probability representable is
$\phi=1/(1+(V-1)e^{60})\approx3\times10^{-32}$, and a floor of that size costs
about $10^{-26}$ bits per step, and the capped softmax itself pays at least $3\times10^{-21}$ bits
per step. The floor caps every single-step surprisal at about $105$ bits,
so a measured surprisal above that ceiling indicates truncation or
numerics, not beliefs. A majority vote
over $N$ sampled answers puts probability zero on every answer but the most
frequent, whatever mass the model gives the rest. Training
on truncated samples makes these zeros part of the next model's
allocation; Appendix~\ref{app:training} treats refitting on generated data.

\paragraph{Tokenisation.}
\label{rem:tokenisation}
A language model is a measure over token sequences absent truncation, and a
distribution over finite token strings only under tight end-of-sequence handling
(Section~\ref{sec:si-prelim}). A claim about induction on text concerns the
law over characters, which is the marginal over all tokenisations of a
string~\cite{cao2021marginal,pimentel2024word,vieira2024tokens}.
By Proposition~\ref{prop:tokenisation}, the marginal is a semimeasure (a
measure if the token law is one) and dominance constants pass through
marginalisation, so Solomonoff's guarantee is
tokeniser-invariant up to $O(K(\Ac))$ bits. Scoring only the canonical
tokenisation instead gives a semimeasure that loses the model's mass on
non-canonical token strings, and dominance does not pass through it.

That mass is reported to be small for trained models~\cite{cao2021marginal,chirkova2023marginalize}. Still, models generate
non-canonical token sequences~\cite{chatzi2025canonical}, aggregating over
non-canonical tokenisations carries signal~\cite{geh2024signal}, and
conditioning on a
canonically tokenised prompt can give probability zero to continuations of
positive marginal probability~\cite{phan2024tokenization,phan2025exact}:
a failure of non-dogmatism created by the interface, not by the beliefs.
How digits are tokenised alone changes arithmetic accuracy sharply~\cite{singh2024tokenization}. For the programme of
Section~\ref{sec:readprior}, a character-level coherence residual can come
from the tokeniser alone, and a slope of surprisal against description
length includes the tokeniser's own allocation. Rendering each symbol as a
single token (for example digits, under a tokeniser that splits digits
individually) removes this confound, but not a second one: a read restricted to
the alphabet and renormalised discards the mass the model puts elsewhere, and what
is discarded depends on the rendering; the rendering and the discarded mass must be
stated for each model.

\paragraph{Chain of thought, test-time compute and effective context.}
When a model answers after a reasoning trace $\zeta$, its answer law is
$\sum_\zeta p(\zeta\mid x)\,p(a\mid x,\zeta)$, a marginal over traces. This
has three consequences.

\emph{Reading.} No single log-probability gives this law; reading it
requires sampling traces, self-consistency voting~\cite{wang2023selfconsistency} estimates its mode, and
Assumption~\ref{ass:elicit} applies to the marginal.

\emph{Compute.} Deployed models have finite description length and bounded
computation per emitted symbol. Chain of thought and search leave the
end-to-end map a closed computable predictor with a larger per-step budget,
so the diagonals of Section~\ref{sec:scale} apply unchanged; they use
nothing beyond computability. What changes is the class the pipeline can
cover end to end. The per-step computation of a fixed-depth, log-precision
transformer lies within uniform $\mathrm{TC}^0$~\cite{merrill2023}; with
polynomially many reasoning steps, transformers reach exactly polynomial
time~\cite{merrill2024}, chain of thought lets constant-size transformers
solve arithmetic and dynamic-programming tasks they cannot solve directly~\cite{feng2023cot}, and test-time compute can substitute
for parameters~\cite{snell2024testtime}. Architectures also differ in the
classes on which they generalise: on formal-language tasks, structured
memory generalises higher in the Chomsky hierarchy, while transformers fail
on some regular tasks~\cite{deletang2023chomsky}, an empirical counterpart
of a class.
Appendix~\ref{sec:rb} states both diagonals in the resource-bounded regime
on one reference machine, with the following consequences. Let
$\CVST(S,\tau)$ be the class of sequences generated by programs of at most
$S$ bits within time $\tau(n)$ at step $n$.
\begin{itemize}
\item A clocked Bayes mixture has regret at most $S$ bits on every member
of $\CVST(S,\tau)$, at per-symbol time about $2^{S+1}n\,\tau(n)$,
exponential in the description length $S$.
\item A predictor program $p$ with per-step time $t_p$ loses at least
$n\lfloor\log V\rfloor$ in $n$ steps on a diagonal environment in
$\CVST(|p|+c_0,\tau)$ whenever $\tau$ exceeds $t_p$ by a quasi-linear
simulation overhead. The uniform version has a rate exponentially small in
the model's first-step time, so at deployed sizes it is qualitative only.
\item Whether the exponential factor is necessary is open.
\end{itemize}
Per-step and total compute thus decide which class $\CVST(S,\tau)$ a system
covers, not whether it has a diagonal: the environment on which it fails
moves with the budget. Test-time compute is how a fixed model buys
coverage, at up to the enumeration price.

\emph{Context.} Nominal and effective context differ. Models use
information in the middle of long contexts poorly~\cite{liu2024lost}, and
long-context benchmarks find effective lengths well below the nominal ones~\cite{hsieh2024ruler}. Weak use of distant information is not a hard
window: Proposition~\ref{prop:window-tax} applies with $W$ the nominal
window, since only beyond it is information provably discarded. Reading $W$
as the effective context is an analogy, not a consequence of the
proposition: by it the tax, and the horizons over which the floors of
Section~\ref{sec:properties} and Paper~B's tests can be read, would be
bounded by effective context, which is not proved here.

\section{Properties and served conditionals}
\label{app:readout}

\subsection{Proofs}
\label{app:readout-proofs}

\begin{proof}[Proof of Proposition~\ref{prop:coherence}]\label{app:proof-prop:coherence}\label{app:proof-prop:arbitrage}
(i) For such $P$, total probability over the next symbol gives
$P(A \mid h) = \sum_\sigma P(\sigma \mid h)\, P(A \mid h\sigma)$ ($A$ being
determined by the continuation of $h$), so served values give $\varrho(h;A) = 0$.
An autoregressive next-token law on one tokenisation is such a measure, so a
positive residual contradicts Assumption~\ref{ass:elicit}. For a semimeasure it holds with $\le$, the gap being
the $A$-weighted mass lost at $h$.
(ii) Buying $Q(A \mid h\sigma)$ units of the $\sigma$-claim at $h$ and, if
$\sigma$ occurs, rolling into one $A$-claim at $h\sigma$ pays
$\mathbf{1}_A \mathbf{1}\{\text{next symbol} = \sigma\}$, self-financed; summed
over $\sigma$ these replicate the direct claim at price
$\sum_\sigma Q(\sigma \mid h)\, Q(A \mid h\sigma)$; buying the cheaper and selling
the dearer (meeting obligations by selling $A$-claims at $h\sigma$) leaves
riskless profit $\varrho(h;A)$.
(iii) If served prices are conditionals of one $P$, the value process
of any finite self-financing strategy is a $P$-martingale from zero, so its
terminal payoff, of $P$-mean zero, cannot be nonnegative everywhere and
positive with positive probability~\cite{definetti1937}. A zero residual at one
event and history constrains no other: necessary, not sufficient.
\end{proof}

\begin{proof}[Proof of Proposition~\ref{prop:order}]\label{app:proof-prop:order}\label{app:proof-prop:suff}
(i) Conditional independence gives
$\pi(h \mid d_{1:t}) \propto \pi(h)\prod_{i=1}^{t} P_h(d_i)$, and each
$P_h(\cdot \mid d_{1:t})$ depends only on the demonstrations' multiset, so the
predictive is permutation-invariant.
(ii) By (a) the posteriors $\pi(h \mid c) \propto w_h P_h(c)$ and
$\pi(h \mid c')$ are proportional and normalised, hence equal; by (b) the
predictives agree. Without (b), the deterministic $000\cdots$ and $0011\cdots$
with contexts $0$, $00$ satisfy (a) but agree after $0$ and not after $00$;
without (a), $\mathrm{Bernoulli}(0.9)$ and $\mathrm{Bernoulli}(0.1)$ under a
uniform prior with contexts $1$, $0$ satisfy (b) yet give $1$ probability $0.82$
and $0.18$.
\end{proof}

\begin{proof}[Proof of Proposition~\ref{prop:slope}]\label{app:proof-prop:slope}
For a deterministic class $Q(h^\dagger_{1:n}) = \sum_{h \in S_n} w_h$ over the set
$S_n \ni h^\dagger$ of hypotheses consistent with $h^\dagger_{1:n}$, so
$L_n(Q;h^\dagger_{1:n}) \le \log(1/w_{h^\dagger})$, with equality iff $S_n$
carries no other weight. The weight hypothesis gives
$\log(1/w_{h^\dagger}) \le \ell(h^\dagger) + c$, with equality when
$w_h = 2^{-\ell(h)-c}$; a uniform prior gives $\log|\mathcal{H}|$ for every truth.
The ceilings $c_k = \log(1/w_{h^\dagger_k})$ of $m$ distinct truths satisfy
$\sum_k 2^{-c_k} \le 1$, so at most $2^{b}$ lie below $b$
(Corollary~\ref{cor:kraftfloor}). The bounds constrain points, not regressions: a
truth of length $\ell_1$ read before identification can have loss near $0$ and one
of length $\ell_2 > \ell_1$ read after it loss $\ell_2 + c$, a slope
$(\ell_2 + c)/(\ell_2 - \ell_1) > 1$.
\end{proof}

\begin{proof}[Proof of Proposition~\ref{prop:support}]\label{app:proof-prop:support}
(i) Rows of $G$ for off-class probes vanish, so $(Gw)_c = 0$ for $c \in O$ and all
$w$. (ii) Each deterministic $h$ is consistent with exactly one $x \in \Ac^{L}$.
Weight $Q(x)$ on one hypothesis consistent with each in-class $x$ is feasible and,
by additivity, gives $(Gw)_c = Q(c) - u(c)$ for $c \in C$, the upper bound.
Off-class $c$ have only off-class extensions, so $u \equiv 0$ iff all off-class
probes have $Q$-mass zero, when both bounds vanish; a zero residual gives
$(Gw)_x = Q(x)$, the total weight of hypotheses consistent with $x$. (iii) A
hypothesis consistent with $c$ fits exactly one extension $ca$, an in-class one, so $(Gw)_c = \sum_{a \in I(c)} (Gw)_{ca}$, while
$Q(c) - \sum_{a \in I(c)} Q(ca) = Q(c)\, u_1(c)$. Hence $r := Gw - q$ has
$\sum_{a \in I(c)} r_{ca} - r_c = Q(c)\, u_1(c)$, at most (Cauchy--Schwarz)
$\sqrt{|I(c)|+1}$ times the norm of $r$ on $\{c\} \cup \{ca : a \in I(c)\}$;
squared restricted norms add over contexts with disjoint index sets.
\end{proof}

\subsection{The implied prior, distillation and merging}
\label{app:readout-prior}

With $\mathcal{H}$, $C$, $G$ and $q$ as in Section~\ref{sec:selfeval}:

\begin{proposition}[Implied prior]
\label{prop:implied}
Let $g_1, \dots, g_m$ be the distinct columns of $G$ and, for a weight vector
$w$, let $\bar w_k$ be the total weight of the hypotheses with column $g_k$.
\textup{(i) Membership.} The minimum in
$\min\{\|Gw - q\|_2 : w \ge 0,\ \sum w \le 1\}$ is attained, and it is $0$ iff
some Bayes mixture over $\mathcal{H}$ agrees with $Q$ on $C$.
\textup{(ii) Identification.} Fitted values $Gw^{Q}$ are common to all
minimisers; weights enter only through the totals $\bar w$, which are identified
only up to the polytope
$W^{Q} := \{\bar w \ge 0,\ \sum_k \bar w_k \le 1,\ \sum_k \bar w_k g_k = Gw^{Q}\}$,
a single point when the distinct columns are linearly independent.
\textup{(iii) Implied coding.} $\ell^{Q}(h) := -\log w^{Q}_h$ is
Kraft-valid, so allocation statistics against $\xi_{w^{Q}}$ are machine-free given
$\mathcal{H}$. \textup{(iv) Machine distance.} The spread of $\ell^{Q} - \ell_F$
for a declared coding $\ell_F$ is bounded iff the two priors are equivalent up to a
constant~\cite{wood2013nonequivalence}.
\end{proposition}

\begin{proof}[Proof of Proposition~\ref{prop:implied}]\label{app:proof-prop:implied}
(i) $Q$ agrees on $C$ with $\sum_h w_h P_h$ iff $q = Gw$. Since
$Gw = \sum_k \bar w_k g_k$ and every $\bar w \ge 0$ with $\sum_k \bar w_k \le 1$
arises from some feasible $w$, the set $\{Gw\}$ is a polytope, hence compact, and
the minimum is attained. (ii) $Gw^{Q}$ is the Euclidean projection of $q$ onto
this compact convex set, hence unique; the totals $\bar w$ are determined by
$Gw^{Q}$ when the distinct columns are linearly independent, and otherwise each
ranges over an interval given by two linear programs (with fixed components,
identifiability of mixtures~\cite{teicher1963} is linear algebra).
(iii) is Kraft's inequality; (iv) is definitional.
Without finitely many distinct columns, $\{Gw\}$ need not be closed and (i)
fails: for $P_h = \mathrm{Bernoulli}(\theta_h)$ with
$\theta_h \to \theta^* \notin \{\theta_h\}$, $C = \{\epsilon, 1, 11\}$ and
$Q = \mathrm{Bernoulli}(\theta^*)$, the residual has infimum $0$, but a mixture
gives $(\sum_h w_h, \E\theta, \E\theta^2)$ with $\E\theta^2 > (\E\theta)^2$
unless the weights sit on one $\theta$, so none agrees with $Q$.
\end{proof}

\begin{proposition}[KL-implied prior]
\label{prop:distill}
For $\mathcal{L}(w) := \sum_{c \in C} \dl{Q(\cdot \mid c)}{\xi_w(\cdot \mid c)}$,
with $\mathcal{L}(w) := \infty$ if $\xi_w(c) = 0$ for some $c \in C$:
\textup{(i)} $\mathcal{L}(w) = 0$ iff $\xi_w$ has $Q$'s conditionals on $C$, so a
vanishing loss certifies membership; the infimum of $\mathcal{L}$ need not be
attained, even for finite $\mathcal{H}$; \textup{(ii)} conditionals identify weights only within
co-consistency components, so $C$ must be prefix-closed; \textup{(iii)} the
student needs a full-support floor column, whose weight reads $Q$'s off-class
allocation.
\end{proposition}

\begin{proof}[Proof of Proposition~\ref{prop:distill}]\label{app:proof-prop:distill}
(i) Relative entropy vanishes only at equality. For $h_1 = 0^\infty$,
$h_2 = 1^\infty$, $C = \{\epsilon, 0, 1\}$ and $Q$ with
$Q(0 \mid \epsilon) = Q(0 \mid 0) = Q(1 \mid 1) = 1$, one has
$\mathcal{L}(w) = \log(1 + w_{2}/w_{1})$ for $w_1, w_2 > 0$ and
$\mathcal{L}(w) = \infty$ when either weight is $0$: infimum $0$, never attained. (ii) $\xi_w(\cdot \mid c)$ depends
on $w$ only through the normalised weights of behavioural classes consistent
with $c$, so mass moved between classes never sharing a consistent probe
leaves the conditionals unchanged. For prefix-closed $C$ containing the empty
context, the conditionals fix ratios of consecutive prefix probabilities, hence the prefix probabilities and all that Proposition~\ref{prop:implied}
identifies. (iii) If $Q$ gives positive probability
to a continuation no hypothesis allows, $\mathcal{L} = \infty$ for all $w$; a
full-support column makes it finite and absorbs the off-class mass.
\end{proof}

\paragraph{Merging is free.}\label{rem:merging}
Merging of opinions~\cite{blackwelldubins1962}, like any observed per-step
excess tending to $0$, is invariant under finite modification
(Fact~\ref{lem:si-tail}) and prior reweighting (Proposition~\ref{prop:si-prior}): evidence of coverage, never of good coverage.

\subsection{The window excess over KT}
\label{app:readout-window}

Let $\Theta$ have Jeffreys' prior on the $V$-symbol simplex, with
$X_1, X_2, \dots$ i.i.d.\ given $\Theta$. A window-$W$ predictor, time-aware or
not, sees $W$ symbols i.i.d.\ given $\Theta$, so its least expected per-step
excess over the source is the Bayes window predictive's,
$r_W = \frac{V-1}{2W\ln 2}(1 + o(1))$ bits, where
$r_j := I(\Theta; X_{j+1} \mid X_{1:j})$ for $j \ge 0$,
attained asymptotically by KT on the window~\cite{clarke1990,krichevsky1981}.
Whole-stream KT has expected cumulative excess $\frac{V-1}{2}\log n + O(1)$, so
every window-$W$ predictor trails it in expectation by at least
$(n-W)\, r_W - \tfrac{V-1}{2}\log n - O(1)$, a rate some i.i.d.\ source forces. In
Proposition~\ref{prop:window-tax}, the Jeffreys mixture's memory term is
$\sum_{t=W+2}^{n}(r_W - r_{t-1})$; the approximation requires $W \gg V$.

\subsection{A corpus as one hypothesis}
\label{app:readout-corpus}

A corpus supplies the exact likelihood of one hypothesis, text drawn from its
law $\mu$, with prefix probabilities from a certified model ($n$-gram, PPM or context-tree). With the column $\mu(c)$
added to finitely many alternative columns, inverting as in Proposition~\ref{prop:implied}
recovers $w^{Q}_\mu$, the model's weight on corpus text before any context ($1$ for
the perfect corpus predictor, Theorem~\ref{thm:closure}), but neither a prior over
English nor structure within $\mu$; statistics on it belong to Paper~B.
Paper~B checks the membership residual, the corpus weight and distillation on
known reference systems.

\section{Related work and provenance}
\label{sec:provenance}

The table maps every displayed result and the main appendix results to
antecedents. Roles follow Section~\ref{sec:intro}: \emph{Fact}, known, source in
the label; \emph{derived}, a short standard consequence, proved for completeness; \emph{sharpened}, known in weaker, asymptotic or
differently scoped form; \emph{new, elementary}, not found stated, short once
posed; \emph{discussion}, prose applying cited work. Paper~B's work is not
credited here. \S\ marks paragraph-labelled items.

\begin{center}
\scriptsize
\setlength{\tabcolsep}{3pt}
\makeatletter
\let\A@startpbox\@@startpbox
\def\@@startpbox#1{\A@startpbox{#1}\raggedright\let\\\tabularnewline}
\makeatother
\begin{tabular}{p{4.0cm}p{7.2cm}p{3.8cm}}
\hline
\textbf{Result} & \textbf{Antecedents} & \textbf{Role} \\
\hline

\multicolumn{3}{l}{\emph{Section~\ref{sec:si-prelim}: background and the Kraft lens}} \\
Dominance, Solomonoff bound, program bound (\ref{fact:si-dominance}) & \cite{solomonoff1964,solomonoff1978,zvonkinlevin1970,hutter2005uai,li2019kolmogorov} & Fact \\
Kraft inequality, conservation (\ref{lem:antichain}) & Kraft; codes as priors~\cite{barron1998mdl,merhav1998universal,grunwald2007mdl}; conservation~\cite{schaffer1994}; no free lunch~\cite{wolpert1996,lattimore2013nfl} & Fact; budget reading ours \\
Profile, floors, exchange, counting (\ref{thm:profile}, \ref{cor:kraftfloor}, \ref{cor:meanfloor}, \ref{cor:exchange}, \ref{cor:typical}) & Kraft; incompressibility~\cite{li2019kolmogorov} & derived \\
Measure part (\ref{lem:measurepart}) & \cite{levinvyugin1977,bienvenu2017semimeasures} & derived \\
Frontier theorem (\ref{thm:frontier}) & \cite{merhav1995strong,cover2006,hutter2003optimality}; the completeness direction is posed as an open question in \cite[\S5.2]{hutter2003optimality}, and an unproved log-loss aside appears in \cite[\S4.1]{koolen2013pareto} & sharpened; equality case answers a special case of Hutter's question (log-loss, $R_\infty$, pairwise singular families); ledger/converse (\ref{prop:frontier-ledger}) new, elementary \\
Specialisation (\ref{cor:specialise}) & Theorem~\ref{thm:frontier}; \cite{cover2006} & new, elementary \\
Overlapping families (\ref{prop:frontier-overlap}) & see ``Frontier'' & sharpened; capacity identity new, elementary \\
\hline
\multicolumn{3}{l}{\emph{Section~\ref{sec:evaluation}: evaluation}} \\
Underdetermination (\ref{thm:si-underdet}, \ref{cor:si-uncert}) & coupling; \cite{gold1967,wolpert1996} & sharpened: quantitative \\
Adaptive underdetermination (\ref{thm:si-adaptive}) & adversary method (query complexity) & new, elementary \\
Accuracy vs.\ fidelity; emergence, contamination (\S\ref{rem:si-dissociation}) & Corollary~\ref{cor:si-uncert}; \cite{wei2022emergent,schaeffer2023mirage,oren2024contamination} & discussion \\
\hline
\multicolumn{3}{l}{\emph{Section~\ref{sec:scale}: scale}} \\
Diagonals, no computable universal predictor (\ref{thm:si-diag}, \ref{cor:si-nocomputable}, \ref{thm:si-classdiag}) & \cite{putnam1963,legg2006,sterkenburg2019}; dovetailing & sharpened: log-loss form, schedules \\
Quantifier exchange (\S\ref{rem:si-pointwise}) & pointwise vs.\ uniform convergence & discussion \\
Transformers, cheap diagonal (\ref{cor:si-transformers}, \ref{rem:si-cheap}) & \cite{yun2020,merrill2023} & derived \\
Randomised decoding (\S\ref{rem:randomised}); breakers (\ref{thm:breakers}) & diagonal; Solomonoff bound & new, elementary; sharpened \\
Resource-bounded diagonal (\ref{thm:rb}) & see ``Resource bounds'' & sharpened; reference formulation \\
\hline
\multicolumn{3}{l}{\emph{Sections~\ref{sec:closure}--\ref{sec:attainable}: corpora and inheritance}} \\
Closure (\ref{thm:closure}, \ref{cor:closure}) & \cite{solomonoff1978,tsybakov2009} & sharpened: domination form \\
Size premium (\ref{rem:price-of-universality}) & \cite{deletang2024compression} & new, elementary: accounting \\
Post-training (\ref{rem:posttrain}), App.~\ref{app:training} & see ``Training'' & derived (affine law); Jensen-gap law new, elementary \\
\hline
\end{tabular}

\medskip
\begin{tabular}{p{4.0cm}p{7.2cm}p{3.8cm}}
\hline
\textbf{Result} & \textbf{Antecedents} & \textbf{Role} \\
\hline
Attainability (\ref{thm:attainable}) & \cite{hutter2003optimality,hutter2005uai,ryabko2010predictors,cesabianchi2006} & Fact \\
No complexity form, non-collapse (\ref{prop:nokform}, \ref{thm:nocollapse}) & diagonal; adding an expert (folklore) & sharpened \\
Neurosymbolic mixtures (\ref{cor:neurosymbolic}) & \cite{jacobs1991,vovk1990}; retrieval (Section~\ref{sec:attainable}) & derived \\
Pretraining certification (\ref{prop:pretrain-certify}) & \cite{topsoe1979,muller2022pfn,xie2022bayesian,zhang2023icl,wies2023learnability} & sharpened: log-loss constants \\
Steering (\ref{prop:manipulable}) & \cite{wolf2024alignment,anil2024,arora2024bayesian,blackwelldubins1962} & derived \\
\hline
\multicolumn{3}{l}{\emph{Section~\ref{sec:si-tail}: asymptotic criteria}} \\
Tail invariance (\ref{lem:si-tail}) & \cite{doob1953,blackwelldubins1962,kalailehrer1994} & Fact \\
$M_N$ witness, invalidity (\ref{thm:si-tail}, \ref{cor:si-invalid}) & the tail lemma & new, elementary \\
Prior-blindness (\ref{prop:si-prior}, \ref{cor:blind}) & \cite{wood2013nonequivalence,leike2015bad} & sharpened: exact-regret class \\
\hline
\multicolumn{3}{l}{\emph{Section~\ref{sec:deployed}: the deployed regime}} \\
Window tax (\ref{prop:window-tax}) & see ``Deployment'' & new, elementary \\
Temperature (\ref{cor:temperature}) & see ``Deployment'' & new, elementary; miscalibration known \\
Truncation, tokenisation (\S\ref{sec:deployed}) & see ``Deployment'' & discussion \\
\hline
\multicolumn{3}{l}{\emph{Sections~\ref{sec:properties}--\ref{sec:tests}: properties and served conditionals}} \\
The two floors (\ref{def:nondogmatic}, \ref{def:dominance}) & \cite{herbster1998,krichevsky1981,clarke1990} & definitions \\
Calibration; ranking, undesirable properties (\S\ref{rem:calibration}, \S\ref{sec:undesirable}) & \cite{dawid1982,guo2017calibration,herbster1998,wallace2024hierarchy} & discussion \\
Elicitation (\ref{ass:elicit}) & \cite{hu2023prompting,sclar2024sensitivity} & explicit assumption \\
Coherence and arbitrage (\ref{prop:coherence}) & \cite{definetti1937}; \cite{zhu2024incoherent,betz2023,costello2014,paleka2025consistency,fluri2024consistency,falck2024martingale} & sharpened: tower form, arbitrage size \\
Invariance nulls (\ref{prop:order}) & \cite{lu2022order,zhao2021calibrate,muller2022pfn} & order derived; sufficiency new \\
Truth-loss levels, floor reading, merging (\ref{prop:slope}, \S\ref{rem:floorreading}, \S\ref{rem:merging}) & \cite{hutter2005uai,blackwelldubins1962}; Kraft floor & derived \\
Implied prior, support (\ref{prop:implied}, \ref{prop:support}) & \cite{teicher1963,wood2013nonequivalence}; elicitation by iterated in-context learning~\cite{zhugriffiths2025priors} & new, elementary \\
KL-implied prior (\ref{prop:distill}) & \cite{hinton2015distill,wang2023latent,arora2024bayesian,carlini2024stealing} & new, elementary \\
\hline
\end{tabular}
\end{center}

\paragraph{Frontier.}\label{app:frontier-prior}
Proposition~\ref{prop:frontier-overlap}(i), with Pareto-optimal implying Bayes,
is Wald's complete-class theorem~\cite{wald1950,blackwellgirshick1954}, uniqueness
coming from strict propriety of the log-loss; Pareto optimality of Bayes mixtures in
sequence prediction is Hutter's~\cite{hutter2003optimality}. The compensation
identity is Tops{\o}e's~\cite{topsoe1979}, and the diagonal of (i) is the
redundancy--capacity theorem~\cite{gallager1979,ryabko1979,davisson1980}, with
minimax theorems at $n=\infty$~\cite{sion1958,haussler1997}. Not found stated:
$C_n=2^{\max_\lambda H_\lambda(\kappa\mid X_{1:n})}$, whose maximisers follow from
Csisz\'ar--Tusn\'ady optimal pairs~\cite{csiszar1984}; overlap bounds are of
Kailath and Hellman--Raviv type~\cite{kailath1967,hellmanraviv1970}.

\paragraph{Resource bounds.}\label{sec:rb-prior}
Theorem~\ref{thm:rb}(i) is the Putnam--Legg diagonal with time-hierarchy
bookkeeping and simulation overhead~\cite{hartmanis1965,hennie1966}. Failure on
sequences of slightly larger time bound is implicit in time-bounded randomness,
resource-bounded measure and dimension~\cite{schnorr1971,ko1986,lutz1992,ambosspies1997,lutz2003dimension,hitchcock2003},
where diagonalising against all $n^{k}$-time predictors, as in (ii), appears with
exponential slowdown. For (iii), mixture regret over deterministic experts is
standard~\cite{vovk1990,cesabianchi2006}, clocked enumeration is AIXI$tl$~\cite[Ch.~7]{hutter2005uai} and universal search~\cite{levin1973}, and speed-prior
prediction~\cite{schmidhuber2002speed} is exponential- but not polynomial-time on
polynomial-time sequences~\cite{filan2016}. Added: a per-symbol two-sided
statement on one class with explicit constants; the exponential gap is open.

\paragraph{Training.}\label{rem:training-prior}
The loss identity behind Proposition~\ref{thm:training-inherit} is the
posterior-predictive divergence of prior-data fitted networks~\cite{muller2022pfn}
with Tops{\o}e's identity~\cite{topsoe1979}. Zhang et al.\ bound in-context regret
by a coverage constant (our $1/\lambda$) times pretraining error~\cite{zhang2023icl}, and Wies et al.\ need a floor on component weights~\cite{wies2023learnability}; added are log-loss constants with equality case and
sharpness. For post-training, the Gibbs identity underlies DPO~\cite{rafailov2023dpo} and the Bayesian reading is Korbak et al.'s~\cite{korbak2022rl}; the frontier criterion and the Jensen-gap ledger are new.
Minimax data mixtures are the redundancy--capacity theorem~\cite{gallager1979,davisson1980,merhav1995strong}, worst-case weighting is DoReMi~\cite{xie2023doremi}, and refitting as an absorbing chain with Wright--Fisher
fixation is in~\cite{shumailov2024collapse,ewens2004} (kept real data change the
dynamics~\cite{bertrand2024stability}); only the frontier reading is added.

\paragraph{Deployment.}\label{app:deployed-prior}
The window theorem is the relative-entropy chain rule plus the
redundancy--capacity decomposition: summed stationary memory terms are the excess
entropy~\cite{crutchfield2003regularities}, the cap is Sharan et al.'s bound~\cite{sharan2018memory} in stationary form for mixtures, and the copy witness is
folklore (bounded state: \cite{jelassi2024repeat}); not found stated are the memory
term as the exact window tax and the dichotomy. For temperature, miscalibration is~\cite{guo2017calibration}, generalised-Bayes tempering acts on posterior weights~\cite{grunwald2004}, and the $2^{-K(t)}$ rate follows from dominance (cf.~\cite{hutter2005uai}). Smoothing floors are standard~\cite{herbster1998};
truncation and its desmoothing reading are~\cite{holtzman2020degeneration,hewitt2022truncation}; tokenisation marginals and
bias are studied in~\cite{cao2021marginal,chirkova2023marginalize,vieira2024tokens,pimentel2024word,phan2024tokenization,phan2025exact}.
The companion work~\cite{cdm2026} is in preparation; nothing here
depends on it.

\providecommand{\enc}[1]{\langle #1\rangle}
\providecommand{\cd}[1]{\ulcorner #1\urcorner}
\providecommand{\bin}{\operatorname{bin}}
\providecommand{\CVST}{\mathcal{C}_V}

\section{Resource-bounded inheritance}
\label{sec:rb}

This appendix restates Theorems~\ref{thm:si-diag}
and~\ref{thm:si-classdiag} and Fact~\ref{thm:attainable} for programs of
bounded length and per-step time on a fixed machine, with explicit
constants (summary in Section~\ref{sec:deployed}).
\label{sec:rb-model}

\begin{definition}[Reference machine]
\label{def:rb-machine}
Machines are deterministic multitape Turing machines with a read-only
\emph{program tape}, a read-only \emph{input tape}, a one-way write-only
\emph{output tape} and work tapes, over alphabets containing $0$, $1$, $\#$
and the blank. Programs are binary: every string placed on the program tape
(in particular every program, generator, $\pi_f$ and $\cd{N}$ below) lies
in $\{0,1\}^{*}$, with blanks to its right. A string $w$ on the program tape is \emph{read exactly} on
input $y$ if the machine halts on $(w,y)$ having scanned every cell of $w$
and none to its right (the initial cell is always scanned). $N(w,y)$ is the
output of machine $N$ and $\theta_N(w,y)\in\Nb\cup\{\infty\}$ its step
count, output cells included. Lengths and times are measured on a fixed
\emph{reference} machine $U$ (distinct from the monotone machine of
Definition~\ref{def:si-M}).
\end{definition}

Hence: \textup{(C1)} for fixed $y$, the strings read exactly on $y$ are
prefix-free (on $wu$ the run is unchanged), so, being binary, at most
$2^{S}$ of them have length at most $S$; \textup{(C2)} a string read exactly within $\theta$ steps has
length at most $\theta+1$; \textup{(C3)} a machine can run $U$ as a
subroutine at one step per step, with the program on a work tape (plus end
marker) or on a segment of its own program tape starting at cell $c>1$ (a
mirror head emulates the left boundary at $O(1)$ cost), clocked by a binary
counter at amortised $O(1)$ cost; \textup{(C4)} the same bookkeeping
detects whether the subroutine scanned the last cell of its segment or
beyond it.

\begin{assumption}[Efficient universality with overhead $\sigma$]
\label{ass:rb-universal}
There is a nondecreasing $\sigma:\Nb\to\Nb$ with
$\sigma(m)\ge\max(m,1)$ such that every machine $N$ has a string
$\cd{N}$ and a constant $c_N$ with: $U$ reads $\cd{N} w$ exactly on $y$ iff
$N$ reads $w$ exactly on $y$, and then $U(\cd{N} w,y)=N(w,y)$ and
$\theta_U(\cd{N} w,y)\le c_N\, \sigma(\theta_N(w,y))$.
\end{assumption}

Hennie--Stearns~\cite{hennie1966} gives
$\sigma(m)=(m+2)\lceil \log(m+2)\rceil$~\cite[Thm.~1.9]{arora2009}, with
the emulation of (C3) securing exact reading; constant-factor models~\cite{cookreckhow1973,jones1993}
give $\sigma(m)=m+1$.

\paragraph{Encodings.} Identify $\Ac$ with $\{0,\dots,V-1\}$. $\bin(m)$ is
the binary numeral of $m$, $\enc{a}:=\bin(a)$,
$\enc{x}:=\enc{x_1}\#\cdots\enc{x_{|x|}}\#$ (so
$|\enc{x}|\le|x|\lceil\log 2V\rceil$), and
$\gamma(m):=0^{\lfloor \log m\rfloor}\bin(m)$ is the Elias gamma code,
prefix-free of length $2\lfloor\log m\rfloor+1$~\cite{elias1975}. Put
$h_V:=\lfloor \log V\rfloor$, $\chi_V:=\lceil \log 2V \rceil^{2}$ and
$k(s):=s+2\lfloor \log(s+1)\rfloor+1$.

\begin{definition}[Budgeted predictors and environments]
\label{def:rb-budget}
Let $t,\tau:\Nb_{\ge1}\to\Nb$.
(a) A program $p$ \emph{computes a predictor} if for every $x\in\Ac^{*}$,
$U$ reads $p$ exactly on $\enc{x}$ and outputs
\[
\bin(u_0)\#\bin(v_0)\#\cdots\bin(u_{V-1})\#\bin(v_{V-1})\#
\]
with $v_a\ge1$ and $\sum_a u_a/v_a\le1$; then
$q_p(a\mid x):=u_a/v_a$, $t_p(n):=\max_{x\in\Ac^{n-1}}
\theta_U(p,\enc{x})$, and $p$ is a \emph{$t$-predictor} if $t_p\le t$
(so $|p|\le t(1)+1$ by (C2)).
(b) $g$ is a \emph{$\tau$-generator} of $z\in\Ac^{\infty}$ if for every
$n$, $U$ reads $g$ exactly on $\enc{z_{<n}}$ within $\tau(n)$ steps and
outputs $\enc{z_n}$. The \emph{budget class} $\CVST(S,\tau)$ is the set of
$z$ with a $\tau$-generator of length at most $S$; it has at most $2^{S}$
elements (distinct $z$ have distinct generators, all read exactly on
$\enc{\epsilon}$, so (C1) applies), is monotone in $S$ and $\tau$, and $z$ is identified with its
point measure.
(c) $f$ is \emph{constructible via $\pi_f$} if $U$ reads $\pi_f$ exactly on
$\bin(m)$ within $f(m)$ steps and outputs $\bin(f(m))$, for every $m\ge1$.
\end{definition}

A deployed model (bounded context, fixed decoding, run from scratch per
context, exactly normalised output) computes such a predictor of about
$bP$ bits ($P$ parameters at $b$ bits).

\label{sec:rb-thm}
\begin{theorem}[Resource-bounded diagonal and clocked mixture]
\label{thm:rb}
Under Assumption~\ref{ass:rb-universal} there are constants
$c_0,c_1\ge1$ depending only on $U$ such that:
\begin{enumerate}
\item[\textup{(i)}] \textbf{Individual diagonal.} Let $p$ compute a
predictor, $|p|=s$, and let
$z^{p}_{n}:=\min\{ a:2^{h_V} u_a\le v_a \}$ with
$(u_a,v_a)_a:=U(p,\enc{z^{p}_{<n}})$. Then
$R_n(q_p;z^{p})\ge n \lfloor\log V\rfloor$ for all $n$, and
$z^{p}\in\CVST(s+c_0,\tau)$ whenever
$\tau(n)\ge c_1 \sigma(c_1 t_p(n))$ for all $n$. With the rule
$\min\{a:Vu_a\le v_a\}$ the loss is at least $n\log V$, under
$\tau(n)\ge c_1\sigma(c_1\chi_V t_p(n))$.
\item[\textup{(ii)}] \textbf{Uniform diagonal.} Let $t$ be constructible
via $\pi_t$ with $t(n)\ge n\lceil\log 2V\rceil$. One
$z^{t}\in\CVST(c_0+|\gamma(V)|+|\pi_t|,\tau)$, for every $\tau\ge
c_1\sigma(c_1 t)$, satisfies for every $t$-predictor $p$ and all $n$
\[
R_n(q_p; z^{t}) \ge \lfloor n\, 2^{-k(|p|)} \rfloor \lfloor\log V\rfloor
\ge \lfloor n\, 2^{-k(t(1)+1)} \rfloor \lfloor\log V\rfloor .
\]
\item[\textup{(iii)}] \textbf{Clocked mixture.} Let $\tau$ be nondecreasing
and constructible via $\pi_\tau$, and $S\in\Nb$. For $x\ne\epsilon$ let
$G(x)$ be the set of $g\in\{0,1\}^{\le S}$ that, for every
$j\le|x|$, are read exactly on $\enc{x_{<j}}$ within $\tau(j)$ steps with
output $\enc{x_j}$; put $\Omega(x):=\sum_{g\in G(x)} 2^{-|g|}$,
$\Omega(\epsilon):=1$, and $\xi_{S,\tau}(a\mid x):=\Omega(xa)/\Omega(x)$
if $\Omega(x)>0$, else $1/V$. Then $\xi_{S,\tau}$ is computed by a
program of length at most $c_0+|\gamma(S+1)|+|\gamma(V)|+|\pi_\tau|$
in per-step time $c_1\sigma(c_1 \Lambda_{S,\tau}(n))$, where
\[
\Lambda_{S,\tau}(n) := 2^{S+1}\bigl[ n\,\tau(n) + n^{2}\lceil\log 2V\rceil
+ (n+V)(S + \lceil \log 2V\rceil) \bigr],
\]
and $R_n(\xi_{S,\tau};z)\le|g|$ for every $\tau$-generator $g$ of $z$
with $|g|\le S$; in particular $R_n(\xi_{S,\tau};z)\le S$ on
$\CVST(S,\tau)$.
\end{enumerate}
\end{theorem}

\begin{corollary}[Budget sandwich]
\label{cor:rb-sandwich}
Let $\tau$ be nondecreasing and $S\in\Nb$.
(i) If $|p|\le S-c_0$ and $c_1\sigma(c_1 t_p)\le\tau$, some
$z\in\CVST(S,\tau)$ has $R_n(q_p;z)\ge n\lfloor\log V\rfloor$ for all
$n$ (and $n\log V$ if $c_1\sigma(c_1\chi_V t_p)\le\tau$); as every
predictor has worst-case loss at least $n\log V$ on $\Ac^{n}$, these at most
$2^{S}$ environments are as hard as all $V^{n}$ strings.
(ii) If $t$ is constructible via $\pi_t$, $t(n)\ge n\lceil\log 2V\rceil$,
$\tau\ge c_1\sigma(c_1 t)$ and $S\ge c_0+|\gamma(V)|+|\pi_t|$, one
$z\in\CVST(S,\tau)$ gives every $t$-predictor
$R_n(q_p;z)\ge\lfloor n 2^{-k(t(1)+1)}\rfloor\lfloor\log V\rfloor$.
(iii) If $\tau$ is constructible via $\pi_\tau$,
$\sup_{z\in\CVST(S,\tau)} R_n(\xi_{S,\tau};z)\le S$, with description
length $O(\log S+\log V)+|\pi_\tau|$ and per-step time
$c_1 \sigma(c_1 \Lambda_{S,\tau}(n))$.
\end{corollary}

Both lower bounds are pointwise in the budget: sublinear worst-case regret
on $\CVST(S,\tau)$ violates the budget inequality at some $n$, not
necessarily infinitely often.

\label{sec:rb-proofs}
\begin{proof}[Proof of Theorem~\ref{thm:rb}(i)]
\emph{Loss.} $\sum_a u_a/v_a\le1$ gives
$\min_a u_a/v_a\le1/V\le2^{-h_V}$, so $z^{p}_n$ exists and each step
costs at least $h_V$; the rule $Vu_a\le v_a$ gives $q_p\le1/V$.
\emph{Generator.} Machine $N_1$ on $(w,y)$: (1) runs $U$ on $(w,y)$,
output to a work tape; (2) counts pairs to get $\bin(V)$ and writes
$h_V=|\bin(V)|-1$ in unary; (3) for $a=0,1,\dots$ tests
$2^{h_V}u_a\le v_a$ (true if $u_a=0$; else compare $|\bin(v_a)|$ with
$|\bin(u_a)|+h_V$, and on equality $\bin(u_a)$ with the leading digits of
$\bin(v_a)$), stopping at the first success; (4) writes $\bin(a)$. Only (1)
touches the program tape, so $N_1$ reads $w$ exactly iff $U$ does. On
$(p,\enc{z^{p}_{<n}})$, (1) takes at most $t_p(n)$ steps and leaves a list
of length $\ell\in[4V,t_p(n)]$; (2)--(4) cost $O(\ell)$, so
$\theta_{N_1}\le c' t_p(n)$. By Assumption~\ref{ass:rb-universal},
$\cd{N_1}p$ is a $\tau$-generator of $z^{p}$ of length $s+|\cd{N_1}|$
once $c_1\ge\max(c_{N_1},c')$. For the rule $Vu_a\le v_a$, schoolbook
multiplication by $V$ makes the total $O(\chi_V \ell)$.
\end{proof}

\begin{proof}[Proof of Theorem~\ref{thm:rb}(ii)]
\emph{Schedule.} Let $\mathrm{cw}(w):=\gamma(|w|+1)\,w$, prefix-free
with $|\mathrm{cw}(w)|=k(|w|)$, and $\beta(j)$ the binary digits of $j$,
least significant first, then zeros. For $j\ge1$ exactly one $w=:w(j)$
has $\mathrm{cw}(w)$ a prefix of $\beta(j)$: uniqueness is prefix-freeness;
for existence, if $\beta(j)$ has $e$ leading zeros, $b_e\cdots b_{2e}$
spells some $\bin(m)$ with $\lfloor\log m\rfloor=e$ and
$w(j):=b_{2e+1}\cdots b_{2e+m-1}$. As $\mathrm{cw}(w)$ is a prefix of
$\beta(j)$ iff $j \equiv r_w \pmod{2^{k(|w|)}}$ for some $r_w\ge1$,
\begin{equation}
\label{eq:rb-count}
|\{ 2 \le n' \le n : w(n'-1) = w \}| \ge \lfloor n\, 2^{-k(|w|)} \rfloor .
\end{equation}
\emph{Construction.} Machine $N_2$ on program $\gamma(V)\pi$ and input
$\enc{x}$, $x\in\Ac^{n-1}$: decodes $V$; counts separators to get $n$;
runs $\pi$ on $\bin(n)$ as a program segment (C3) to get $t(n)$; outputs
$\enc{0}$ if $n=1$; otherwise runs $U$ on $w(n-1)$ (work tape, end
marker) with input $\enc{x}$, clocked at $t(n)$. The run \emph{succeeds} if
it halts in time, has scanned all of $w$ and not the marker (C4), and
outputs $V$ pairs with $v_a\ge1$; then $N_2$ outputs the least $a$ with
$2^{h_V}u_a\le v_a$, else $\enc{0}$. Set
$z^{t}_n:=N_2(\gamma(V)\pi_t,\enc{z^{t}_{<n}})$.
\emph{Resources.} $N_2$ reads $\gamma(V)\pi_t$ exactly (its program tape
serves only to decode $V$ and run $\pi_t$), in time
$O(n\lceil\log 2V\rceil+t(n))\le c' t(n)$; as in (i),
$\cd{N_2}\gamma(V)\pi_t$ is a $\tau$-generator of the stated length.
\emph{Loss.} If $w(n-1)=p$ for a $t$-predictor $p$, the run succeeds and
the emitted symbol has $q_p$-probability at most $2^{-h_V}$; by
\eqref{eq:rb-count} this happens at least $\lfloor n2^{-k(|p|)}\rfloor$
times, and $k(|p|)\le k(t(1)+1)$.
\end{proof}

\begin{proof}[Proof of Theorem~\ref{thm:rb}(iii)]
\emph{Predictor.} For $x\ne\epsilon$ the sets $G(xa)$ are disjoint
subsets of $G(x)$ (members of $G(xa)$ output $\enc{a}$ on $\enc{x}$), so
$\sum_a \Omega(xa)\le\Omega(x)$; for $x=\epsilon$ they are disjoint sets
of strings read exactly on $\enc{\epsilon}$, so $\sum_a \Omega(a)\le1$ by
(C1). \emph{Regret.} A $\tau$-generator $g$ of $z$ with $|g|\le S$, being binary, lies in
every $G(z_{1:j})$, so $\xi_{S,\tau}(z_{1:n})=\Omega(z_{1:n})\ge2^{-|g|}$.
\emph{Program.} Machine $N_3$ on $\gamma(S+1)\gamma(V)\pi$ and $\enc{x}$,
$n:=|x|+1$: for each $g\in\{0,1\}^{\le S}$ and $j\le n$ it computes
$\tau(j)$ via $\pi$, runs $U$ on $g$ (work tape, end marker) with input
$\enc{x_{<j}}$ clocked at $\tau(j)$, and records success and output. Then
$g\in G(x)$ iff the runs with $j<n$ output $x_j$, and $g\in G(xb)$ if
also the run with $j=n$ outputs $b<V$. Integer accumulators
$2^{S}\Omega(xa)$ and $2^{S}\Omega(x)$ (at most $2^{S}$ by (C1)) receive
$2^{S-|g|}$; the output is these pairs, or $(1,V)$ if $\Omega(x)=0$.
$N_3$ reads its program exactly. \emph{Time.} A pair $(g,j)$ costs
$O(\tau(n)+n\lceil\log 2V\rceil+S)$; over at most $2^{S+1}$ strings and
$n$ values of $j$, plus $O(V(S+\lceil\log 2V\rceil))$ per $g$ for the
accumulators and $O(VS+\log S+\log V)$ for decoding and output,
$\theta_{N_3}\le c'\Lambda_{S,\tau}(n)$.
\emph{Constants.} $c_0:=\max_i |\cd{N_i}|$ and $c_1$ the largest
$c_{N_i}$ and $c'$; $N_1,N_2,N_3$ do not depend on $V$, $S$, $\tau$, $t$,
$p$ or $n$.
\end{proof}

\begin{proof}[Proof of Corollary~\ref{cor:rb-sandwich}]
(i) Theorem~\ref{thm:rb}(i) and monotonicity in $S$; the worst-case bound is
$\sum_{x\in\Ac^{n}} q(x)\le1$ (Fact~\ref{thm:conservation}). (ii)
Theorem~\ref{thm:rb}(ii) and monotonicity. (iii) Theorem~\ref{thm:rb}(iii).
\end{proof}

\label{sec:rb-remarks}
\begin{remark}[Levin complexity of the diagonal environment]
\label{rem:rb-kt}
With $Kt(y):=\min\{|g|+\log \theta_U(g,\enc{\epsilon}):U(g,\enc{\epsilon})=y,\ g \text{ read exactly}\}$~\cite{levin1973,levin1984},
$\sigma(m)=(m+2)\lceil\log(m+2)\rceil$ and nondecreasing
$t_p(n)\ge n\lceil\log 2V\rceil$,
$Kt(z^{p}_{1:n})\le s+\log(n\, t_p(n))+\log n+O(\log\log(n\,t_p(n)))$. \emph{Proof.} A fixed machine with program
$\delta(n)\cd{N_1}p$ ($\delta$ the Elias delta code~\cite{elias1975}) runs
the generator of Theorem~\ref{thm:rb}(i) $n$ times, feeding outputs back, in
$O(n\,\sigma(c_1t_p(n))+n^{2}\lceil\log 2V\rceil)=O(n\,\sigma(c_1t_p(n)))$
steps; $\delta(n)$ costs the $\log n$ term.
\end{remark}

\begin{remark}[Tightness, open questions and floating-point outputs]
\label{rem:rb-tight}\label{rem:rb-gap}\label{rem:rb-real}
The rate of (i) is optimal, since the uniform predictor loses exactly
$n\log V$; the rate of (ii) is uniform over $t$-predictors only through
(C2) and exponentially small in $t(1)$, hence qualitative for models whose
first step reads about $10^{11}$ bits; a slope independent of $t(1)$ is
open. The overhead $\tau\ge c_1\sigma(c_1t)$ is inherent to diagonal
arguments up to constants~\cite{hartmanis1965,furer1982,jones1993}, and the
factor $n$ in $\Lambda_{S,\tau}$ comes from statelessness. Whether the
factor $2^{S}$ between the lower threshold (about $\tau/\log\tau$) and the
mixture ($2^{S+1}n\tau$) is necessary is open. Finally, if a program
outputs rational approximations to real conditionals within $1/(4V)$,
selecting the least $a$ with approximation at most $1/V+1/(4V)$ gives
per-step loss at least $\log(2V/3)$ (as in
Corollary~\ref{cor:si-nocomputable}); if rounded outputs have sum
$s:=\sum_a u_a/v_a\in(0,2]$, the least $a$ with $2^{h_V-1}u_a\le v_a$
exists (as $\min_a u_a/v_a\le s/V\le2^{1-h_V}$) and forces per-step loss
at least $\lfloor\log V\rfloor-1$ on the raw outputs and at least
$\lfloor\log V\rfloor-1-\log(1/s)$ on the outputs renormalised by $s$,
so at least $\lfloor\log V\rfloor-1$ on both when $1\le s\le2$.
\end{remark}

\par\endgroup

\section*{Data and code availability}
The code and seeds that produce this paper's own numbers, namely the worked
numerical illustrations and the frontier figure, are released with it (the code under the MIT License, the outputs under CC~BY~4.0) at
\texttt{[URL/DOI to be inserted]}; one command regenerates them. The manuscript itself is \copyright{} the
authors; all rights reserved. The
served-model readings relayed in Section~\ref{sec:discussion} are produced by
Paper~B's pipeline and are reproducible from that paper's released caches, not
from this one's; that includes the long-horizon dominance sweep, which Paper~B
collected afresh and whose responses it also releases.

\section*{Acknowledgements}
[To be completed.]

\section*{Funding and competing interests}
[To be completed.]

\section*{Use of AI tools disclosure}

Generative AI tools were used to assist with language editing, \LaTeX{} formatting, literature search, checking
derivations, proposing objections to arguments, organising the material, and writing and reviewing the accompanying code.
Every definition, theorem, proof and
interpretation in this paper originated with the authors and was verified by them, and the authors take full
responsibility for the manuscript.

\end{document}